\documentclass{article}

     \PassOptionsToPackage{numbers, square, compress}{natbib}

\usepackage[preprint]{neurips_2026}

\usepackage[utf8]{inputenc} 
\usepackage[T1]{fontenc}    
\usepackage{hyperref}       
\usepackage{url}            
\usepackage{booktabs}       
\usepackage{amsfonts}       
\usepackage{nicefrac}       
\usepackage{microtype}      
\usepackage{xcolor}         
\usepackage{wrapfig}

\usepackage{geometry}
\usepackage{microtype}
\usepackage{graphicx}
\usepackage{subcaption}
\usepackage{booktabs} 
\usepackage{float}
\usepackage{placeins}

\usepackage{amsmath}
\usepackage{amssymb}
\usepackage{mathtools}
\usepackage{amsthm}

\usepackage{bm}

\usepackage{tcolorbox}
\tcbuselibrary{theorems}
\usepackage{mdframed}
\usepackage{enumitem}

\usepackage[square,numbers,compress]{natbib}

\newtcolorbox{theorembox}{
  colback=gray!5,          
  colframe=gray!50,       
  boxrule=0.4pt,           
  arc=3pt,                 
  left=4pt,        
  right=4pt,       
  top=3pt,
  bottom=3pt,
  outer arc=3pt,
  before skip=8pt,
  after skip=8pt
}

\usepackage[capitalize,noabbrev]{cleveref}
\usepackage{authblk}

\newcommand{\prox}{\operatorname{prox}}
\newcommand{\R}{\mathbb{R}}
\newcommand{\ip}[2]{\langle #1,#2\rangle}
\newcommand{\aff}{\operatorname{aff}}
\newcommand{\parl}{\operatorname{par}}
\newcommand{\ri}{\operatorname{ri}}

\newcommand{\Id}{\mathrm{Id}}
\newcommand{\supp}{\operatorname{spt}}
\newcommand{\M}{\mathcal{M}}
\newcommand{\pistar}{\pi^\star}
\newcommand{\barpi}{\bar\pi}

\theoremstyle{plain}
\newtheorem{theorem}{Theorem}[section]
\newtheorem{proposition}[theorem]{Proposition}
\newtheorem{lemma}[theorem]{Lemma}
\newtheorem{corollary}[theorem]{Corollary}
\theoremstyle{definition}
\newtheorem{definition}[theorem]{Definition}
\newtheorem{assumption}[theorem]{Assumption}
\theoremstyle{remark}
\newtheorem{remark}[theorem]{Remark}

\DeclareMathOperator{\dom}{dom}

\title{Flow Matching from Viewpoint of Proximal Operators}
\author{%
  \textbf{Kenji Fukumizu}\textsuperscript{1} \hspace{.34cm} 
  \textbf{Wei Huang}\textsuperscript{2,1} \hspace{.34cm} 
  \textbf{Han Bao}\textsuperscript{1,3,2} \hspace{.34cm} 
  \textbf{Shuntuo Xu}\textsuperscript{4} \hspace{.34cm} 
  \textbf{Nisha Chandramoorthy}\textsuperscript{5} 
  \\
  \textsuperscript{1}The Institute of Statistical Mathematics. 
  10-3 Midoricho, Tachikawa 190-8562 Tokyo, Japan \\
  \textsuperscript{2}RIKEN AIP.  1-4-1 Nihonbashi, Chuo-ku,  Tokyo 103-0027, Japan \\
  \textsuperscript{3}Tohoku University. 6-3-09 Aramaki Aza Aoba, Aoba-ku, Sendai, Miyagi 980-8579 Japan \\
  \textsuperscript{4}East China Normal University. 3663 North Zhongshan Road, Shanghai 200062, China \\
  \textsuperscript{5}The University of Chicago. 5747 S. Ellis Avenue
Chicago, IL 60637 USA

}

\begin{document}

\maketitle

\begin{abstract}
We discuss the dynamics of Optimal Transport Conditional Flow Matching (OT-CFM) under the manifold hypothesis, where the support of the target distribution is a low dimensional manifold.  We first derive an exact proximal expression of the vector field, using an extended Brenier potential for the OT; the vector field can be expressed by a denoiser, which is presented as the proximal operator associated with the potential function.  Based on this explicit proximal form of the vector field, 
we analyze the manifold stability for perturbations to initial points, using the Lyapunov exponents of the dynamics.  By analyzing the local behavior with the tools in convex analysis, we show that the dynamics contract linearly in the perturbation of the normal directions to the manifold, while perturbations in the tangential directions are preserved; this implies the stability of the manifold structure to the perturbations. 
To justify the theoretical arguments on the population OT, we also establish the convergence of minibatch OT-CFM with neural network learning to population OT-CFM under regularity and stability assumptions, with rate exponents governed by the intrinsic dimension of the data manifold.
Experiments on synthetic data, MNIST, and CIFAR-10 support the predicted geometric behavior.
\end{abstract}

\section{Introduction}
\label{sec:intro}

Dynamical generative models such as diffusion models \citep{SohlDickstein2015,Ho2020DDPM,Song2021SDE} and flow matching \citep{Lipman2023,Liu2023-ha,Albergo2023} have recently achieved remarkable success in broad application fields
\citep{Yang2024-vq}, showing strong sampling ability for high-dimensional domains. 
Prominent examples include text-to-image generation~\citep{Rombach2022LDM,Saharia2022Imagen}, posterior sampling for inverse problems~\citep{Kawar2022DDRM}, and diffusion-based policies for robots~\citep{Chi2023DiffusionPolicy}.
In scientific domains, diffusion models have also led to strong results in molecular design~\citep{Corso2023DiffDock,Watson2023RFdiffusion}.
In parallel, flow-based models have emerged as competitive alternatives and have been applied to image generation~\citep{Esser2024-pf} and scientific modeling such as weather forecasting \cite{Couairon_ScienceAdvances} and molecule editing \citep{Ikeda2025-uv}.

This paper focuses on flow matching (FM), which learns a vector field $v_t(x)$ whose induced ODE transports a source distribution
$P_0$ (typically standard normal for generation tasks) to a target distribution $P_1$.
The method of Conditional FM (CFM, \cite{Lipman2023}) generates random conditional paths
\begin{equation}\label{eq:xt}
    x_t = \alpha_t x_0 + \beta_t x_1 \qquad (t\in[0,1])
\end{equation}
given $(x_0,x_1)\sim \pi$ and schedulers $\alpha_t,\beta_t\in[0,1]$, where $\pi$ is a joint distribution with marginals $P_0$ and $P_1$.
A random velocity field along the conditional path is constructed by
\begin{equation}\label{eq:vt}
    v_t(x_t; x_0,x_1) = \dot{\alpha}_t x_0 + \dot{\beta}_t x_1,
\end{equation}
where $\dot{\alpha}_t$ and $\dot{\beta}_t$ denote the time derivatives.  Such random velocities are used as teaching data at $x_t$ to train a neural network to estimate the marginal vector field
\begin{equation*}
    v_t(x_t) \coloneqq \mathbb{E}_{x_0,x_1}\left[ v_t(x_t;x_0,x_1)\mid x_t = \alpha_t x_0 + \beta_t x_1 \right].
\end{equation*}

With the above CFM construction, many different vector fields can transport $P_0$ to  $P_1$.   
The degree of freedom comes from the choice of coupling $\pi$ between $P_0$ and $P_1$, which is implicit in the interpolation path \eqref{eq:xt}. 
The popular {\em OT-CFM} constructs the coupling by solving (minibatch) optimal transport (OT) problems in the stochastic gradient descent learning of neural networks \citep{Tong2024,Pooladian2023-kl}.  This improves training by reducing the variance of the random teaching data \citep{Pooladian2023-kl} and solves ODE more easily by straighter vector fields, resulting in better empirical performance.

It is often observed that the strong empirical performance of generative models is more evident when the data are distributed on a low-dimensional manifold, a phenomenon called the \emph{manifold hypothesis}; high-resolution images are typical examples.
Some recent theoretical studies analyze diffusion models under the manifold hypothesis, including manifold learning \citep{Pidstrigach2022-nt}, minimax estimation rates governed by intrinsic dimension \citep{Oko2023,fu2026intrinsic,Yakovlev2025-ef}, mixtures of low-rank Gaussians \citep{wang2025diffusion}, 
a clean subspace/tangent--normal decomposition \citep{Chen2023-lr,Chen2024-pj},
Lyapunov stability near manifolds \citep{Chandramoorthy2025-sd},
and geometric effects/singular behavior of the score around manifolds \citep{Li2025-ut,Liu2025-hg}. On the other hand, similar analysis of FM, especially of OT-CFM under the manifold hypothesis, is scarce, while there are some recent works on FM with independent coupling \citep{Kumar2026-kq,pi2026learning}. Since the learning of OT-CFM is based on a distinct principle from the score, different approaches are required to understand the dynamical behavior around a data manifold. 

This paper presents a novel framework for analyzing the vector field and dynamics of OT-CFM, in the general case where the target distribution concentrates on a low-dimensional manifold.  
We derive an explicit formula for the vector field using the extended Brenier potential and its proximal operator, which are popular tools in convex analysis \citep{Brenier1991-lg,RockafellarVariatioal,ParikhBoyd_proximal}.  The formula expresses the vector field as an affine form of the proximal operator of the potential at the scaled location $x_t/t$, and provides a denoiser interpretation to OT-CFM in a manner analogous to that in diffusion models.

Based on the proximal expression of the vector field, under the manifold hypothesis, we analyze the stability of the data manifold against the perturbation of initial points.
Applying existing results on proximal sensitivity under partly smoothness \citep{HareLewis2004Identification,DaniilidisHareMalick2006}, we derive an explicit tangent--normal decomposition
along population OT trajectories. For affine interpolation and the
terminal time variable $\tau=-\log(1-t)$, infinitesimal
perturbations in the normal direction decay with exponent $-1$, while tangent perturbations approach finite, nonzero limits and have exponent $0$. This implies that small changes in the initial point causes a move in the data manifold. 

Our analysis primarily uses a theoretical assumption that the conditional random paths employ the population OT
coupling. While this is strong compared to the minibatch OT used in practice, it is not restrictive; we show that the
vector field of the minibatch OT with finite data converges to the one given by the population OT for the large sample
size limit under some assumptions. We also show that the convergence rate depends only on the manifold dimensionality,
which is of independent interest.

Our main contributions are summarized as follows:

\begin{itemize}[topsep=3pt, partopsep=0pt, itemsep=3pt, parsep=0pt, leftmargin=12pt]
\item 
 We provide a framework of OT-CFM using an extended Brenier potential, which expresses the OT-CFM vector field with the subdifferential of the potential, covering the manifold hypothesis cases.
\item 
  We show that the OT-CFM vector field 
  admits an explicit expression by the proximal operator, which has an interpretation as a denoiser.  
\item  Using convex analysis tools, we prove that perturbations to the initial point shrink in the normal directions to the data manifold and are preserved in the tangential directions, meaning that the manifold structure is stable against perturbation.
\item We theoretically and empirically demonstrate that the analysis based on the population OT provides an approximation of the practical minibatch OT-CFM. 
\end{itemize}

\section{Formulation of OT-based Flow Matching via Proximal Operators}
\label{sec:prox_OTCFM}

\begin{figure}[tb]
  \centering
  \begin{minipage}[t]{0.48\linewidth}
  \vspace{18pt} 
   \centering
    \includegraphics[width=\linewidth]{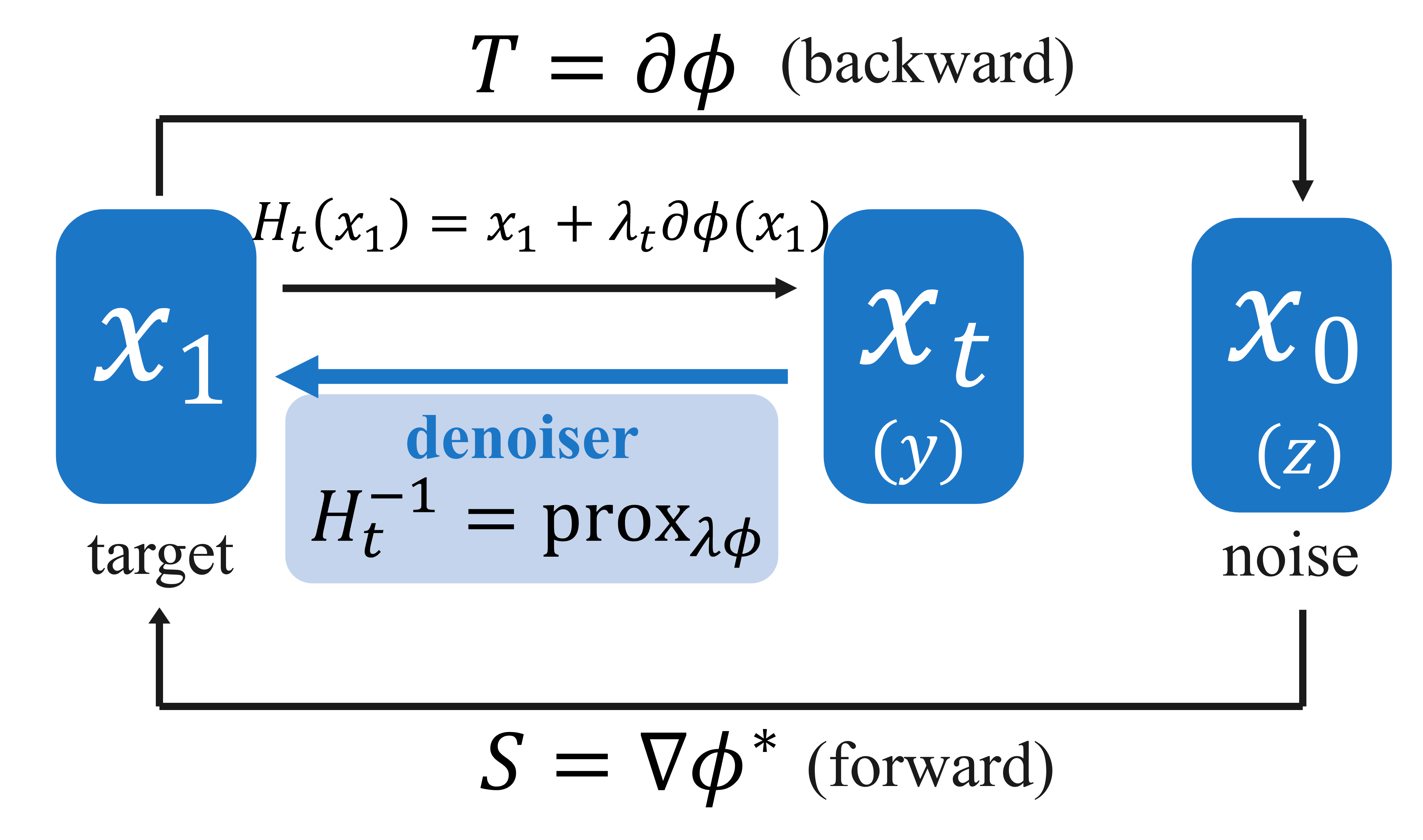}
    \vspace{12pt}
    \caption{
      Relation among a target sample $x_1$, noise sample $x_0$, and an interpolation $x_t$.
      Here $\phi$ denotes the extended Brenier potential (Sec.~\ref{sec:AB-potential}).
      The target $x_1$ and noise $x_0$ are connected via $x_0 \in T(x_1)$, where $T = \partial \phi$ is the (extended) Brenier map.
      In parallel, $x_t$ is mapped back to $x_1$ via the denoiser 
      (Sec.~\ref{sec:VF_prox}).}
    \label{fig:prox}
  \end{minipage}\hfill
  \begin{minipage}[t]{0.48\linewidth}
  \vspace{0pt}
    \centering
    \includegraphics[width=6cm]{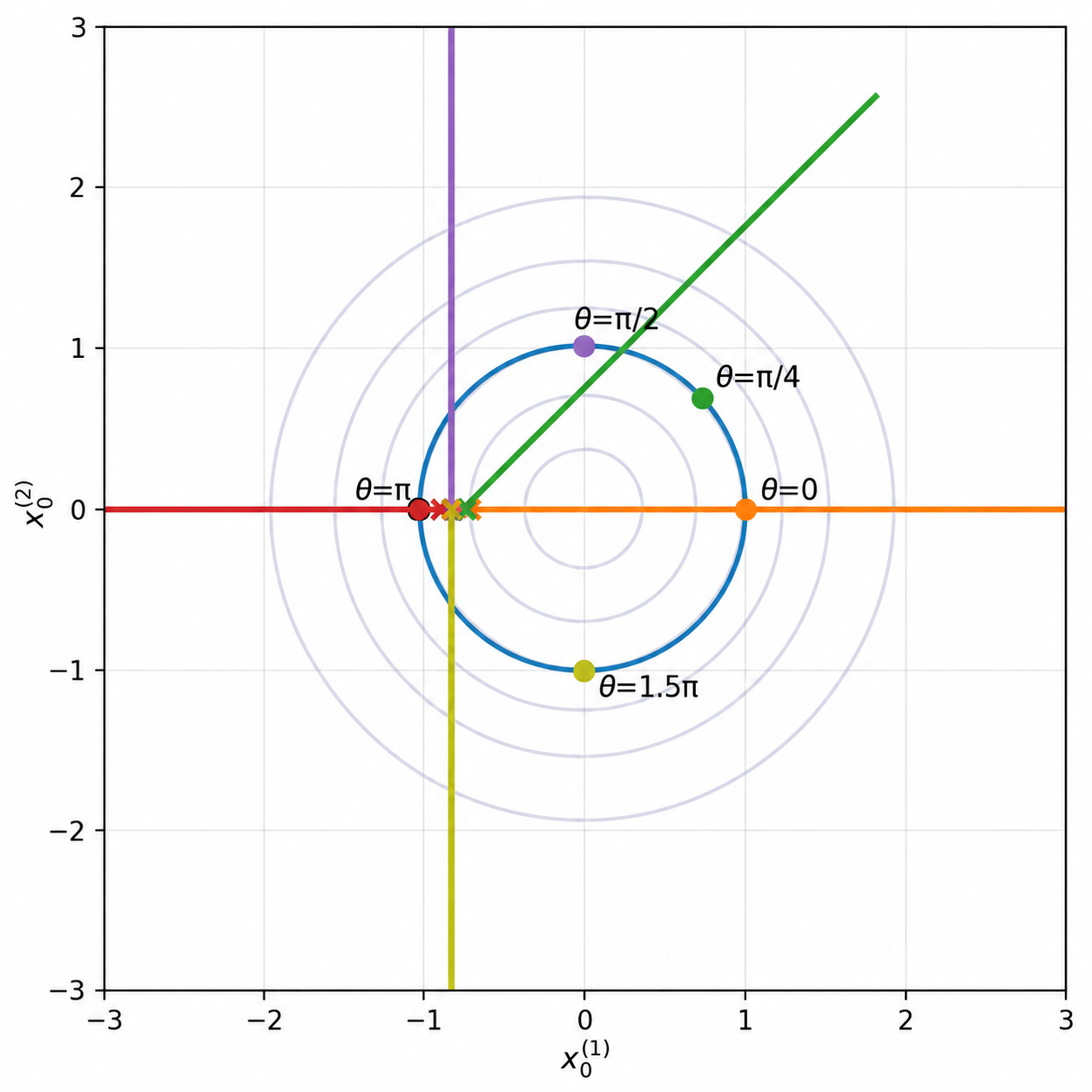}
        \caption{Examples of coupling. $P_0=N(0,I_2)$ and $P_1$: von Mises distribution $\rho_{1,0}$. The point on the circle and the half line in the same color are coupled by OT. Note that the direction of $\partial\phi(x_1)$ is normal to the manifold at $x_1$. }
    \label{fig:vonMises_fiber}
  \end{minipage}
\end{figure}

\subsection{OT-CFM}
We use the notation of the conditional path \eqref{eq:xt} and vector field  \eqref{eq:vt}. Throughout this paper, we assume that the source distribution $P_0$ is absolutely continuous with respect to the Lebesgue measure on $\mathbb{R}^d$, and $P_0,P_1 \in\mathcal{P}_2(\mathbb{R}^d)$, which have a finite second moment. The source $P_0$ is typically the standard normal $N(0,I_d)$ in practice, but except Sec.\ref{sec:denoiser} we do not use special properties of Gaussian. 
Also, we do \emph{not} assume that $P_1$ admits a density function; rather, we are interested in the case where the support of $P_1$ is a low-dimensional manifold.
The extended real numbers are denoted by $\overline{\R}\coloneqq\R\cup\{+\infty\}$. 

For simplicity of theoretical expositions, this paper only considers the affine paths $x_t = (1-t)x_0 +tx_1$. This is not restrictive; extending to general schedules $(\alpha_t,\beta_t)$ is easy under appropriate changes. By eliminating $x_0$ in \eqref{eq:vt} by $x_0 = (x_t-t x_1)/(1-t)$, we obtain
\begin{equation}
     v_t(x_t) = -\frac{x_t - \mathbb{E}[x_1|x_t]}{1-t}.
     \label{eq:vf_denoise}
 \end{equation}
This is essentially the same as the drift term of diffusion models, in which the term $\mathbb E [x_1|x_t]$ is often regarded a {\em denoiser} to estimate clean $x_1$ from noisy $x_t$. The form of \eqref{eq:vf_denoise} is general to any CFM methods using affine paths.  The only non-trivial part is $\mathbb{E}[x_1|x_t]$, which is also the only part dependent on the coupling $(x_0,x_1)$.

In our theoretical analysis, we primarily consider the OT-CFM \emph{in population}; instead of OT coupling in minibatches, the population OT between $P_1$ and $P_0$ is considered to generate random paths \eqref{eq:vt} and $\mathbb E[x_1|x_t]$.  
In particular, the OT with the quadratic cost, 
$\inf_{\pi\in\Pi(P_0,P_1)} \int \frac12\|x_0-x_1\|^2\,d\pi(x_0,x_1)$,
is considered, where $\Pi(P_0,P_1)$ denotes the joint distributions on $\R^d\times \R^d$ with the marginals $P_0,P_1\in\mathcal{P}_2(\mathbb{R}^d)$.
We will discuss the convergence of the minibatch OT with neural net learning theoretically in Sec.~\ref{sec:consistency} and empirically in Sec.~\ref{sec:experiment}, demonstrating that the theoretical argument based on population OT provides a reasonable approximation of the practical minibatch OT-CFM.

\subsection{Extended Brenier potential and proximal form of OT-CFM vector field}

\label{sec:AB-potential}\label{sec:VF_prox}

To derive $\mathbb E[x_1|x_t]$ in \eqref{eq:vf_denoise}, 
we consider the transport map from $P_1$ to $P_0$ to eliminate $x_0$\footnote{
We can also consider the transport map, either from $P_0$ to $P_1$ or $P_1$ to $P_0$.  We use the latter in this paper to see the local behavior around the manifold. }
When the target $P_1$ does not have a density, such as in the manifold hypothesis, an OT \emph{map} or \emph{Brenier map} from $P_1$
to $P_0$ may not exist \citep{Brenier1991-lg}.
Nevertheless, for the quadratic cost, the OT plan always admits an expression  by the subdifferential of some convex potential.
More concretely, it is known \citetext{\citealp[Theorem~12.2]{Villani2003TopicsOT}; \citealp{Ruschendorf1990-vz}} that, for an OT plan $\pistar$, 
there exists a proper lower semicontinuous (lsc) convex function
$\phi:\mathbb{R}^d\to\overline{\mathbb R}$ such that
\begin{equation}\label{eq:aleksandrov_graph}
(x_0,x_1)\in\mathrm{supp}(\pistar)
\quad\Longrightarrow\quad
x_0\in \partial \phi(x_1),
\end{equation}
where $\partial\phi$ is the subdifferential of  $\phi$ \citep{Bauschke2017}.
Extending the usual terminology to distributions $P_1$ that need not admit a density,  
we call such $\phi$ the \emph{extended Brenier potential} for the transport from $P_1$ to $P_0$ in this paper. 
The extended Brenier potential may not be differentiable; 
if $P_1$ has a p.d.f., then $\phi$ is differentiable
$P_1$-a.e., and \eqref{eq:aleksandrov_graph} reduces to the Brenier map
$x_0=\nabla\phi(x_1)$ \citep{Brenier1991-lg}.
Intuitively, if the support of $P_1$ is an $m$-dimensional manifold, a point $x_1$ couples with a $(d-m)$-dimensional subset in $\partial \phi (x_1)$, induced by $\pistar$ 
(see Sec.~\ref{sec:example} for an example). 

We express the vector field of the OT-CFM using the proximal operator (see, e.g., \cite{ParikhBoyd_proximal,RockafellarVariatioal}), which is commonly used in convex analysis.
Let $(x_0,x_1)$ be a pair sampled by the OT plan $\pistar$.  Then, using $x_0 \in \partial \phi(x_1)$, we have, for $t\in (0,1)$, 
\begin{equation}\label{eq:inclusion_xt}
x_t = (1-t)x_0 + t x_1 \,\in\, t\Bigl( I + \frac{1-t}{t} \partial \phi\Bigr)(x_1),
\end{equation}
where $I$ is the identity operator.
This gives an inclusion expression of $x_t$ in terms of $x_1$.  To obtain $\mathbb E[x_1|x_t]$, we wish to have an inverse relation, which can be realized by the proximal operator. 

Generally, for a convex function $F : \mathbb{R}^d \to \overline{\R}$
and $\lambda > 0$, the \emph{proximal operator} of $F$ with parameter $\lambda$ is defined by 
\begin{equation}
    \prox_{\lambda F}(y)
    \coloneqq \arg\min_{u\in\R^d} \Bigl\{
        F(u) + \frac{1}{2\lambda}\|u-y\|^2
    \Bigr\}.
    \label{eq:prox_def}
\end{equation}
This can be regarded as the minimizer of the $L^2$-regularized objective.  For differentiable $F$, $x = \prox_{\lambda F}(y)$ if and only if 
$y = x + \lambda \nabla F(x)$, which allows the interpretation as an implicit Euler step.

The basic relation for our purpose is the next lemma, which  
does \emph{not} require the differentiability of $\phi$.
\begin{lemma}\label{lem:prox_general}
For $(x_0,x_1)$ sampled from the population OT coupling $\pistar$, let $x_t = (1-t) x_0 + t x_1$ and $\phi$ be defined as above. For $t\in(0,1)$, 
\begin{equation}\label{eq:prox_identity}
x_1 =
\prox_{\lambda_t \phi}\Bigl( \frac{x_t}{t} \Bigr),
\qquad
\lambda_t = \frac{1-t}{t}.
\end{equation}
Consequently, the vector field given by the population OT coupling is represented by 
\begin{equation}\label{eq:vf_prox_general}
u^\star(x_t)
= - \frac{x_t -\prox_{\lambda_t\phi}( x_t/t)}{1-t}.
\end{equation}
\end{lemma}
\begin{proof}
From \eqref{eq:inclusion_xt}, we have 
$
0 \in \partial ( \lambda_t \phi(x_1) + \frac{1}{2}\| x_1 - \tfrac{x_t}{t}\|^2 )$, which means $x_1$ is the minimizer of $\lambda_t \phi(\cdot) + \frac{1}{2}\| \cdot - \tfrac{x_t}{t}\|^2$. 
Thus, $x_1 = \prox_{\lambda_t\phi}  (x_t/t)$ by definition.  
\end{proof}
Note that in the current setting of $P_0$ and $P_1$, $x_1$ is a deterministic, single-valued map of $x_t$. Eq.~\eqref{eq:vf_prox_general} provides a formula for the OT-CFM vector field as an explicit function of location $x_t$.  We use this explicit form to discuss the stability of the manifold structure in Sec.~\ref{sec:lyap-prox-epi}.

\label{sec:denoiser}
{\bf Denoiser view.} The $\prox$ term of \eqref{eq:vf_prox_general} can be interpreted as a precise \emph{denoiser} in the following
``noise-corrupted observation'' sense, when $P_0\sim N(0,I_d)$. It is similar to diffusion models, where $\mathbb{E}[x_1|x_t]$ is regarded as a denoiser. 
For simplicity, suppose that there exists a Brenier \emph{map}
$T=\nabla \phi:\mathbb{R}^d\to\mathbb{R}^d$ such that $
T_{\#}P_1 = P_0$, 
where $\phi$ is a differentiable Brenier potential and $T_{\#}P_1$ is the pushforward of $P_1$ by $T$.
Sample a clean data $x_1\sim P_1$ and define its \emph{OT code} by $z \coloneqq T(x_1)$, which follows Gaussian distribution as a marginal. 
Now, define the observed (corrupted) variable by
$y \coloneqq x_1 + \lambda z = x_1 + \lambda \nabla\phi(x_1)$. 
Then the proximal map recovers the clean data $x_1$ given corrupted $y$ by $\prox_{\lambda\phi}(y)=
x_1$ exactly (Eq.~\eqref{eq:prox_identity}).  
Hence $\prox_{\lambda\phi}$ is a deterministic denoiser for this 
OT-induced corruption: it inverts the ``noising'' mechanism
$x\mapsto x+\lambda T(x)$. This makes a clear contrast to diffusion models, where independent Gaussian noise is used; hence, its denoising is statistical.
See Fig.~\ref{fig:prox} for the relationship among $(x_1,y,z)$.

\subsubsection{An example of OT coupling for low-dimensional manifold}
\label{sec:example}

As an example of OT coupling under the manifold hypothesis, we consider the von Mises distribution on the unit circle $\mathbb S^1$ as $P_1$ and $N(0,I_2)$ as $P_0$. The details are shown in Appendix \ref{app:von-mises-circle}. The density of von Mises distribution with mean $\mu\in[0,2\pi)$ and concentration $\kappa>0$ is given by 
\[
    \rho_{\kappa,\mu}(\theta)
    =
    \frac{\exp\{\kappa\cos(\theta-\mu)\}}
         {2\pi I_0(\kappa)}, \qquad (\theta\in [0,2\pi)),
\]
where $I_0$ is the modified Bessel function of the first kind. 

Let $e_\theta=(\cos\theta,\sin\theta)\in \mathbb S^1$. Using the Kantorovich dual formulation, we can see that the function $a(\theta)$ that attains the minimum of 
\[
\min_a \int_{\R^2}\sup_{\theta\in[0,2\pi)}\{\langle x,e_\theta\rangle - a(\theta)\}dP_0(x) + \int_0^{2\pi} a(\theta)d\theta,
\]
gives the fiber $\partial\phi(e_\theta)$ in the form of 
\[
    \partial\phi_a(e_\theta)  =
    \left\{
        a'(\theta)e_\theta^\perp+s e_\theta :  s\ge r_a(\theta)
    \right\},
\]
where $r_a(\theta)  \coloneqq   \sup_{\eta\ne\theta}
    \frac{ a(\theta)-a(\eta)+a'(\theta)\sin(\eta-\theta) }{1-\cos(\eta-\theta)  }$
(see Appendix \ref{app:von-mises-circle} for details).  Figure \ref{fig:vonMises_fiber} shows the $e_\theta=\{0,\frac{\pi}{4},\frac{\pi}{2},\pi,\frac{3\pi}{2}\}$ and the coupled fibers in the same color.  We can see that the fiber coupled with $e_\theta$ is normal to the tangent space $T_{e_\theta}\mathbb S$. This is generic; see Sec.~\ref{sec:proximal-sensitivity}.

\section{Analysis of Terminal Lyapunov Exponents}
\label{sec:lyap-prox-epi}

This section analyzes the terminal behavior of flow matching (FM) under the low-dimensional manifold hypothesis, based on the proximal formulation developed in Sec.~\ref{sec:prox_OTCFM}.
In data generation, it is often observed that small perturbations to the initial condition lead to variations that remain close to the data manifold, suggesting anisotropic stability of the dynamics: contraction in normal directions and neutral behavior along tangential directions. To formalize this observation, we study the local behavior of trajectories via the Jacobian of the vector field, and characterize the behaviors using Lyapunov exponents.
The explicit proximal expression \eqref{eq:vf_prox_general} allows us to apply active-manifold
sensitivity results 
\citep{HareLewis2004Identification,DaniilidisHareMalick2006}.
For the affine interpolation $x_t=(1-t)x_0+tx_1$, this yields an explicit Jacobian along each qualifying OT trajectory. We use it to characterize the terminal growth rates of infinitesimal perturbations.

\subsection{Lyapunov exponents}
We review the definition of Lyapunov exponents. Consider, in general, the ODE with time $\tau\in[0,\infty)$
\begin{equation}\label{eq:ODE_tau}
    x'(\tau)=u_{\tau}(x(\tau)),
\end{equation}
where $x'(\tau) \coloneqq \frac{d x(\tau)}{d \tau}$ denotes the time derivative with respect to $\tau$ in the sequel.  
For $\tau_0>0$ and $v\in\R^d\setminus\{0\}$, the {\em Lyapunov exponent} in the direction $v$
is defined (if it exists) by
\begin{equation}\label{eq:lyapunov_def}
\lambda_{\tau_0}(v)
\coloneqq\lim_{\tau\to+\infty}\frac{1}{\tau-\tau_0}\log\bigl\|D_x\Phi_{\tau,\tau_0}(x(\tau_0))\,v\bigr\|,
\end{equation}
where $\Phi_{\tau,\tau_0}$ denotes the associated flow from $\tau_0$ to $\tau$; 
$x(\tau)=\Phi_{\tau,\tau_0}(x(\tau_0))$, and 
$D_x\Phi_{\tau,\tau_0}(x(\tau_0))$ is the space derivative of $\Phi_{\tau,\tau_0}$ at $x(\tau_0)$. See, e.g. \cite{barreira2002lyapunov,hirsch2006invariant}. The norm $\|\cdot\|$ is any norm on $\R^d$; as all norms are equivalent, $\lambda_{\tau_0}(v)$ is norm-independent. 
Negative exponents (resp., positive) correspond to exponential contraction (resp., expansion) of nearby trajectories in $\tau$-time. 
Let 
\[
\xi(\tau)
 \coloneqq 
D_x\Phi_{\tau,\tau_0}(x(\tau_0))v.
\]
From \eqref{eq:ODE_tau}, $\xi$ satisfies the ODE
\[
\xi'(\tau) = D_x u_\tau(x(\tau))\xi(\tau).
\]
We therefore analyze the growth of solutions of this variational equation. For the OT-CFM considered below, the proximal sensitivity formula provides a fixed tangent--normal decomposition along the reference trajectory, allowing the equation to be solved explicitly.

\subsection{Proximal sensitivity and tangent--normal splitting}
\label{sec:proximal-sensitivity}

We use existing results of active-manifold sensitivity \citep{HareLewis2004Identification,DaniilidisHareMalick2006} to describe the local behavior of the proximal denoiser. 
See Appendix \ref{sec:dom-second-subderivative} for details. 

Fix an OT pair $(x_1,x_0)$ with $x_0\in\partial\phi(x_1)$ and define
\begin{equation}
\label{eq:shiftedPotential}
 f \coloneqq \phi-\langle x_0,\cdot\rangle,
 \qquad \Xi_\lambda \coloneqq \prox_{\lambda f}\quad(\lambda>0).
\end{equation}
Then $f$ is proper, lsc, and convex, and $0\in\partial f(x_1)$, so that
$\Xi_\lambda(x_1)=x_1$. In the sequel, we make the following assumptions:
\begin{itemize}[topsep=0pt,partopsep=0pt,itemsep=1pt,parsep=0pt,leftmargin=36pt]
\item[{\bf (M)}] $\M\subset \R^d$ is a $C^2$ submanifold of dimension $m(<d)$, and the support of $P_1$ is $\M$.
\item[\textbf{(PS)}] $f$ is $C^2$-partly smooth at $x_1$ relative to $\M$. (See Sec.~\ref{sec:partly_smooth} and \cite{HareLewis2004Identification} for the definition.) 
\item[\textbf{(ND)}] $0\in\ri(\partial f(x_1))$, equivalently
$x_0\in\ri(\partial\phi(x_1)),$ where $\ri$ stands for relative interior.
\end{itemize}
For $x_1\in \M$, the tangent and normal spaces are denoted by  $T_{x_1}\M$ and $N_{x_1}\M$, respectively 
(normal with respect to the inner product of $\R^d$), so that
$\R^d \cong T_{x_1}\M \oplus N_{x_1}\M$. 
The assumption \textnormal{(PS)} means that $f|_\M$ is $C^2$ near
$x_1$ that $\partial f$ is nonempty at nearby points of $\M$ and continuous
at $x_1$ relative to $\M$, and that
$\parl(\partial f(x_1))=N_{x_1}\M$
\citep[Definition~2.3]{HareLewis2004Identification}.
Here $\parl(C) \coloneqq \{a-b: a, b \in \aff(C)\}$, and continuity is understood in the
set-valued sense of inner and outer semicontinuity. These assumptions do
not require the differentiability of $\phi$ in the ambient $\mathbb R^d$. 
 A sufficient condition for
\textnormal{(ND)} to hold almost surely is given in
Appendix~\ref{sec:ri-subgradient-disintegration}.

Write $T \coloneqq T_{x_1}\M$ and $N \coloneqq N_{x_1}\M$ for simplicity, and let $\Pi_T$ and $\Pi_N$ be
the orthogonal projections onto these spaces. Define
\begin{equation}
\label{eq:tangential-hessian}
 H \coloneqq \operatorname{Hess}_{\M}(f|_{\M})(x_1):T\to T,
\end{equation}
using the Riemannian metric induced by the ambient Euclidean space.
Since $x_1$ minimizes $f$, its Hessian $H$ at $x_1$ is self-adjoint and positive semidefinite.

\begin{proposition}
\label{prop:proximal-sensitivity}
Under (M), (PS), and (ND), for every fixed $\lambda>0$, $\Xi_\lambda$ 
is $C^1$ on a neighborhood $U_\lambda$ of $x_1$ and
$\Xi_\lambda(U_\lambda)\subset\M$. Its Jacobian at $x_1$ is
\begin{equation}
\label{eq:prox-tangent-jacobian}
 D\Xi_\lambda(x_1)=(I_T+\lambda H)^{-1}\Pi_T,
\end{equation}
where $I_T$ is the identity on $T$.
In particular, 
\begin{equation}
\label{eq:prox-tangent-normal}
 \|D\Xi_\lambda(x_1)-\Pi_T\|_{\mathrm{op}}
 \le\lambda\|H\|_{\mathrm{op}}.
\end{equation}
\end{proposition}
The statements of local identification $\Xi_\lambda(U_\lambda)\subset\M$ and smoothness of $\Xi_\lambda$  follow from
\citet[Theorem~5.3]{HareLewis2004Identification} and the sensitivity
argument of \citet[Theorems~28--29]{DaniilidisHareMalick2006}.
Although the latter reference treats general prox-regular functions and states these results for sufficiently small $\lambda>0$, in the present convex setting, the argument applies to every $\lambda>0$. Appendix~\ref{sec:proof-proximal-sensitivity} proves these claims, and derives
\eqref{eq:prox-tangent-jacobian} by differentiating the
manifold-restricted optimality equation.

\paragraph{Application to the OT denoiser.}
For $0<t<1$, use the affine interpolation and denoiser
\[
 x_t=(1-t)x_0+tx_1,\qquad
 \lambda_t=(1-t)/t, \qquad  Q_t(x):=\prox_{\lambda_t\phi}(x/t).
\]
It is easy to see the shift identity
$\prox_{\lambda\phi}(y)=\Xi_\lambda(y-\lambda x_0)$, which gives
$Q_t(x_t+h)= \Xi_{\lambda_t}(x_1+h/t)$. Hence, from \eqref{eq:prox-tangent-jacobian}, the relation 
\begin{equation}
\label{eq:DQDP}
 DQ_t(x_t)
 =\frac1t D\Xi_{\lambda_t}(x_1)
 =\bigl(tI_T+(1-t)H\bigr)^{-1}\Pi_T
\end{equation}
holds.  Thus $DQ_t(x_t)$ annihilates normal directions and converges to $\Pi_T$
as $t\uparrow1$. Note that the directions $T$ and $N$ are determined at the location $x_1\in \M$.

\subsection{Terminal Lyapunov exponents and manifold stability}
\label{sec:terminal-affine}

To follow the standard terminology of Lyapunov analysis, we use the time interval $[0,\infty)$ by setting $\tau \coloneqq -\log(1-t)$ or $t(\tau)=1-e^{-\tau}$.  Then, $\tau\in[0,\infty)$ and $dt/d\tau=1-t$. The rescaled OT-CFM equation is
\[
 x'(\tau)=u_\tau(x(\tau)),\qquad
 u_\tau(x) \coloneqq Q_{t(\tau)}(x)-x,
\]
where $x'(\tau)$ denotes the derivative with respect to time $\tau$.
Fix the OT trajectory
$x(\tau)=(1-t(\tau))x_0+t(\tau)x_1$ and use the notation of
Sec.~\ref{sec:proximal-sensitivity}. Put $B_t \coloneqq tI_T+(1-t)H$. Then, 
Eq.~\eqref{eq:DQDP} gives the block decomposition
\begin{equation}
\label{eq:Du_0}
 D_xu_\tau(x(\tau))
 =-\Pi_N+(B_t^{-1}-I_T)\Pi_T,
 \qquad t=t(\tau).
\end{equation}
In particular, the fixed subspaces $T_{x_1}\M$ and $N_{x_1}\M$ are
preserved over $t$ or $\tau$ by the variational equation. Note again that these spaces are defined at the
endpoint $x_1$, not at the intermediate point $x_t$.

\begin{theorem}
\label{thm:terminal-lyap}
Fix an OT pair $(x_1,x_0)$ and assume \textnormal{(M)}, \textnormal{(PS)}, and \textnormal{(ND)}.
For any $\tau_0>0$, let $t_0=t(\tau_0)$ and write $v=v_T+v_N$ according
to $\mathbb R^d=T_{x_1}\M\oplus N_{x_1}\M$. Then
\begin{equation}
\label{eq:affine-flow-derivative}
 D_x\Phi_{\tau,\tau_0}(x(\tau_0))v
 =B_{t(\tau)}B_{t_0}^{-1}v_T
   +e^{-(\tau-\tau_0)}v_N.
\end{equation}
Consequently, the Lyapunov exponents defined by \eqref{eq:lyapunov_def} satisfy
\[
 \lambda_{\tau_0}(v)=-1
 \quad\text{for }v\in N_{x_1}\M\setminus\{0\},
 \qquad
 \lambda_{\tau_0}(v)=0
 \quad\text{for }v\in T_{x_1}\M\setminus\{0\}.
\]
The same zero exponent holds whenever $v_T\ne0$.
\end{theorem}

\begin{proof}
For any finite time interval beginning at $\tau_0>0$, the local
smoothness in Prop.~\ref{prop:proximal-sensitivity} gives the  variational equation
\begin{equation}
\label{eq:dudtau_0}
 \xi'(\tau)=D_xu_\tau(x(\tau))\xi(\tau),\qquad \xi(\tau_0)=v.
\end{equation}
By \eqref{eq:Du_0}, the normal component solves $\xi_N'=-\xi_N$.
For the tangent component, observe that $B_t^{-1}-I_T=(1-t)(I_T-H)B_t^{-1}$ and $\frac{d}{d\tau}B_{t(\tau)}=(1-t)(I_T-H)$. It follows from \eqref{eq:Du_0} that 
\[
\xi_T' = B_{t(\tau)}' B_{t(\tau)}^{-1} \xi_T.
\]
Therefore, the solution is $\xi_T(\tau)=B_{t(\tau)}B_{t_0}^{-1}v_T$, proving 
\eqref{eq:affine-flow-derivative}.
Since $H\succeq0$, $B_{t_0}$ is invertible for $t_0>0$, and
$B_t\to I_T$ as $t\uparrow1$.
Thus a nonzero tangent component has the nonzero limit
$B_{t_0}^{-1}v_T$, while the normal component decays exactly as
$e^{-(\tau-\tau_0)}$. The claimed exponents follow.
\end{proof}

Theorem~\ref{thm:terminal-lyap} concerns infinitesimal perturbations along a reference OT
trajectory. Their normal components contract by the exact factor
$e^{-(\tau-\tau_0)} = (1-t)/(1-t_0)$, whereas nonzero tangential components approach finite,
nonzero limits.  Note that the normal Lyapunov exponent $-1$ depends on the time scheduling; it is not difficult to see that if we use \eqref{eq:xt} with $\alpha_t = (1-t)^\gamma$, the exponent is $-\gamma$.
The result of the theorem agrees with the stability analysis of the diffusion models of \cite{Chandramoorthy2025-sd}, which showed that if the Lyapunov exponents along the tangential directions are dominant, the manifold structure is stable in dynamics.
Some recent work \cite{debortoli2022convergence,Li2025-ut,Liu2025-hg} showed the Taylor expansion of the score function, demonstrating zero Lyapunov exponents along the manifold and negative along the normal directions.  Although the Lyapunov exponents suggest similar behaviors, the trajectories of OT-CFM may not approach to the manifold in the normal direction.

\section{Minibatch OT-CFM and Population OT-CFM}
\label{sec:consistency}

Theoretical arguments in the previous sections consider the population OT for constructing the endpoints of the paths in CFM.  However, in practical situations where we have only finite data, CFM methods often use the minibatch OT, where a batch of size $B$ is taken from the training data and also from $P_0$. The samples in the two batches are coupled by an OT method, such as the Hungarian algorithm, resulting in the matched samples $\{(X_{0,i},X_{1,i}^{{\rm match}})\}_{i=1}^B$.   A random path $X_t = (1-t)X_{0,i} + t X_{1,i}^{{\rm match}}$ is constructed by a random choice of a pair $(X_{0,i},X_{1,i}^{{\rm match}})$, and a neural network model uses the velocity $V(X_t,t) \coloneqq X_1^{{\rm match},i}-X_0^i$ as the teaching data at $(X_t,t)$ for stochastic gradient descent (SGD). 

This section makes a theoretical argument that the population OT-CFM provides a reasonable surrogate of the practical minibatch OT-CFM by showing the convergence of the latter to the former as the sample size ($N$) and the batch size ($B$) go to infinity.  This implies that the minibatch OT-CFM and population OT-CFM provide similar behaviors asymptotically.  We also derive a convergence rate, in which the exponents depend only on the manifold dimensionality.  We give the result informally in this section, and the details are deferred in Section \ref{sec:app:minibatch-OT}.  

One of the difficulties in discussing the finite-sample convergence is that the vector field given by the batches can be defined only on the finite paths, which may have measure zero. Thus, we need to discuss regression of the vector field, typically by neural network models.  
This paper considers an idealized regression procedure following \cite{BoiteDelonNadjahi2026}, but combines it with neural network learning as a new development.  Consider independent batches $X_{0,1},\ldots,X_{0,B}\overset{\mathrm{iid}}{\sim}P_0$ and $X_{1,1},\ldots,X_{1,B}\overset{\mathrm{iid}}{\sim}P_1$, 
and define empirical distributions $\widehat P_{0,B} \coloneqq \frac1B\sum_{i=1}^B\delta_{X_{0,i}}$ and $\widehat P_{1,B} \coloneqq \frac1B\sum_{j=1}^B\delta_{X_{1,j}}$.
Suppose an OT solver selects an optimal permutation $\widehat\pi_B$. Then, a neural network $u$ is learned via minimizing
\[
 \mathbb E_{(X_0,X_1)\sim \widehat \pi_B, t} \| u(X_t, t) - V(X_0,X_1)\|^2, 
\]
where $t$ is taken independently. Assume that we generate such random batches $L$ times independently. Then, it is well known that the least-squares learning ideally converges as $L\to\infty$ to the average of the vector field over the random choice of batches:
\begin{equation}\label{eq:expbatchOT}
\mathbb E_B \mathbb E_{(X_0,X_1)\sim \widehat \pi_B}[V(X_0,X_1)|X_t] = \arg\min_u \mathbb E_B \mathbb E_{(X_0,X_1)\sim \widehat \pi_B, t} \| u(X_t, t) - V(X_0,X_1)\|^2, 
\end{equation}
where $\mathbb E_B$ is the average over the random  batches.  We consider the \emph{expected batch OT plan} $\barpi \coloneqq \mathbb E_B[\widehat \pi_B]$ following \cite{BoiteDelonNadjahi2026}, the average of the minibatch OT plan $\widehat \pi_B$ over the random batches.
If we take a random pair with the law $\barpi$, the resulting expected vector field is equal to \eqref{eq:expbatchOT}, so we can analyze $\barpi$.

{\bf Neural network learning.}
For simplicity, we restrict our discussion in this section to a sparse ReLU network defined on a compact set.  
Let $K_i \coloneqq \operatorname{supp} P_i$ ($i=0,1$), and assume that they are compact sets and $K_1$ equals an $m$-dimensional manifold $M$.  Although compact $K_0$ excludes Gaussian distributions, this does not cause much practical harm since the source $P_0$ for FM can be arbitrary, such as truncated Gaussian.  
Let $I=[0,c]$ with $0<c\le 1$. Then, the set of possible paths $\mathcal X_I \coloneqq 
 \left\{\bigl(t,(1-t)x_0+tx_1\bigr):
 t\in I,\ x_0\in K_0,\ x_1\in K_1\right\}$ is a compact subset of $\R^{d+1}$; we assume $\mathcal X_I \subset [0,1]^{d+1}$, w.l.o.g.

Let $\rho(a) \coloneqq \max\{a,0\}$ be the ReLU activation.  A network
$\Phi:\mathbb R^{q_0}\to\mathbb R^{q_J}$ with depth $J$ and widths
$\mathbf q=(q_0,\ldots,q_J)$ is written as
\begin{equation}
\Phi
=
F_J\circ\rho\circ F_{J-1}\circ\cdots\circ\rho\circ F_1,
\qquad
F_j(z)=W_jz+b_j,
\label{eq:relu-network-definition-zero}
\end{equation}
where the ReLU is applied componentwise, and the last affine map is not followed by an activation.  

Fix constants $C_D,R_f,C_A,\kappa_A>0$. 
We consider a network model $\mathcal F_P$ with a complexity budget $P$ (sparsity) defined by \begin{equation*}
\begin{aligned}
&  J\le  C_D\log^2(P+1), 
\quad 
{\textstyle \max_j} q_j\le P+d+1,
\quad 
 {\textstyle \sum_j} \|W_j\|_0+\|b_j\|_0\le P,
\\
&  {\textstyle \max_j} \{\|W_j\|_\infty,\|b_j\|_\infty\}  \le C_A P^{\kappa_A},
\qquad 
{\textstyle \sup_{z\in\mathcal X_I}} \|\Phi(z)\|\le R_f.   
\end{aligned}
\end{equation*}

{\bf Convergence result.} 
Let $(X_0,X_1)\sim \pistar$ be a pair sampled from the population OT, and define
\begin{equation}\label{eq:populationOT_path}
    X_t^\star=(1-t)X_0+tX_1 
\end{equation}
and the corresponding ideal vector field $u^\star$ \eqref{eq:vf_prox_general}.  Suppose that $\widehat u_{B,L,P} \in \mathcal F_P$ is a trained vector field with $L$ independent repetitions of size-$B$ batches.  The ODE for generation is given by 
\begin{equation}
 \frac{d\widehat X_t}{dt}=\widehat u_{B,L,P}(\widehat X_t,t),
 \qquad
 \widehat X_0=X_0^\star=X_0.
 \label{eq:learned-flow-zero}
\end{equation}
We assume that the ideal vector field $u^\star$ is smooth of degree $s$:  $u^\star\in C^s([0,1]^{d+1};\mathbb R^d)$. 

Extending the analysis of the vector field and its flow to the terminal
time $t=1$ requires additional care because of the potentially singular
behavior of \eqref{eq:vf_denoise} as $t\uparrow1$; see, e.g.,
\citet[Appendix~C.2]{Wan2025Denoisers}.
We therefore consider $c<1$ (ideally close to $1$) for $I=[0,c]$ in the following theorem.

\begin{theorem}[Informal]
\label{thm:convergence_rate}
Let $m>4$ and choose the integer parameter budget $P$ appropriately depending on $L$ and $m$.
Under some assumptions listed in Appendix \ref{app:minibatch-OT}, 
for all sufficiently large $L$ and every $\eta\in(0,1)$, with
probability at least $1-\eta$ over the training batches,
\begin{equation}
\left(\mathbb E_\star\sup_{t\in I}
\|\widehat X_t-X_t^\star\|^2\right)^{1/2}
 \lesssim
B^{-1/m}
+\left(\frac{\log^5(eL)}{L}\right)^{s/(2s+m+1)}
+\sqrt{\frac{\log(2/\eta)}{L}}
+\sqrt{\varepsilon_{\mathrm{opt}}},
\label{eq:generic-fast-rate-zero}
\end{equation}
where $\mathbb E_\star$ denotes the expectation with respect to the independent test data and $\varepsilon_{\mathrm{opt}}$ is the optimization error of neural networks. 
The Wasserstein error at time $c$
satisfies the
same upper bound. 
\end{theorem}
This theorem shows that the resulting distributions of the minibatch OT-CFM and the population OT-CFM are close for large batch size $B$ and the number of batches $L$.  The first term comes from the minibatch bias and the second term is the neural approximation-estimation contribution.  The polynomial rates of convergence on $B$ and $L$ depend only on the manifold dimensionality $m$.  Note that the dimensionality of the input to the learned vector field $\widehat{u}_t(x)$ is $d+1$, and thus we need special arguments for the approximation error of the model to derive the bound depending on $m$.  To achieve this, we basically extend the approximation on $M$ by using the affine structure of the population OT, extending the arguments developed for diffusion models in \cite{fu2026intrinsic} to the OT-CFM.   To the best of our knowledge, this is the first result of the convergence rate of minibatch OT-CFM depending on the intrinsic dimensionality under the manifold hypothesis.

\section{Related work}
\label{sec:related}

{\bf Proximal view to diffusion models.} 
Diffusion models have been discussed from the proximal viewpoint with the Wasserstein metric \citep{Zhang2026-ro,Li2023-ev,Caluya2022-dr}.  These works consider proximal operators in infinite dimensional probability spaces and thus differ from the current paper. 

{\bf Diffusion models and the manifold hypothesis.}
Some recent work discusses the behavior of diffusion models under the manifold hypothesis; \citet{Pidstrigach2022-nt} derived conditions under which score-based models sample the underlying low-dimensional data manifold. 
\citet{Oko2023} showed that the asymptotic convergence rate of learning the target distribution depends on the dimension of the manifold, not the ambient space. In addition to the work discussed in the end of Sec.~\ref{sec:lyap-prox-epi}, the dynamical behavior of diffusion models under low-dimensional subspaces has also been studied, such as the orthogonal decomposition of the space \cite{Chen2023-lr} and the provable recovery \cite{wang2025diffusion}. 
Compared with these results, analogous analysis for FM under the manifold hypothesis remains limited.

{\bf Flow matching and proximal viewpoints.}
FM models are often trained with improved couplings of minibatch OT
\citep{Pooladian2023-kl,Tong2024}.
Our work provides a theoretical basis for this line with a convex-analytic characterization, enabling stability analysis in manifold-supported cases.  Our framework does not apply directly to rectified flow \cite{Liu2023-ha}. 
Note that our usage of proximal operators is different from the standard view of finite proximal-steps, which include Schr\"odinger bridges/entropic OT as KL-proximal iterations \citep{Leonard2014Survey,Cuturi_NIPS2013}; or Wasserstein-proximal step \citep{balcerak2025energy}.
The most relevant to our work is 
Optimal Flow Matching \citep{kornilov2024optimal}, which 
uses convex potentials to carry straight trajectories ($P_0$ to $P_1$). Although their flow inversion is equivalent to a proximal operator, they do not explicitly adopt the proximal or denoiser perspective, nor do they

\section{Experiments}\label{sec:experiment}

\begin{figure}[tb]
  \centering
  \begin{minipage}[t]{0.48\linewidth}
    \centering
   \includegraphics[width=1.0\linewidth]{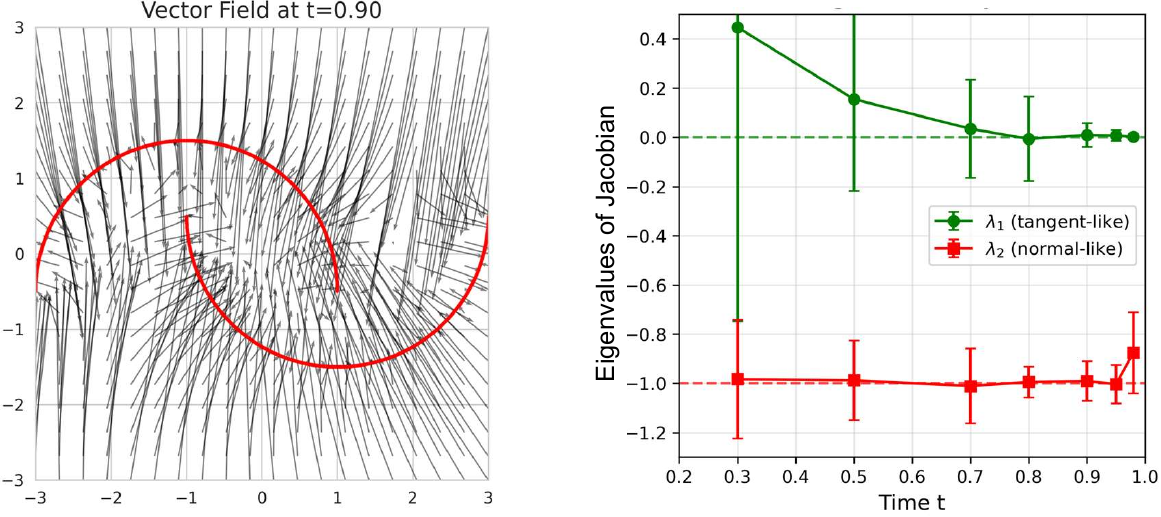}
  \caption{Left: OT-CFM vector field ($t=0.9$) obtained by a neural network. An attracting force on the manifold can be seen.  Right: Mean and STD of the eigenvalues of $(1-t)Dv_t$ over different trajectories. The second eigenvalue corresponds to a Lyapounov exponent close to -1.
  }\label{fig:two_moon}
  \end{minipage}\hfill
  \begin{minipage}[t]{0.48\linewidth}
    \centering
    \includegraphics[width=1.0\linewidth]{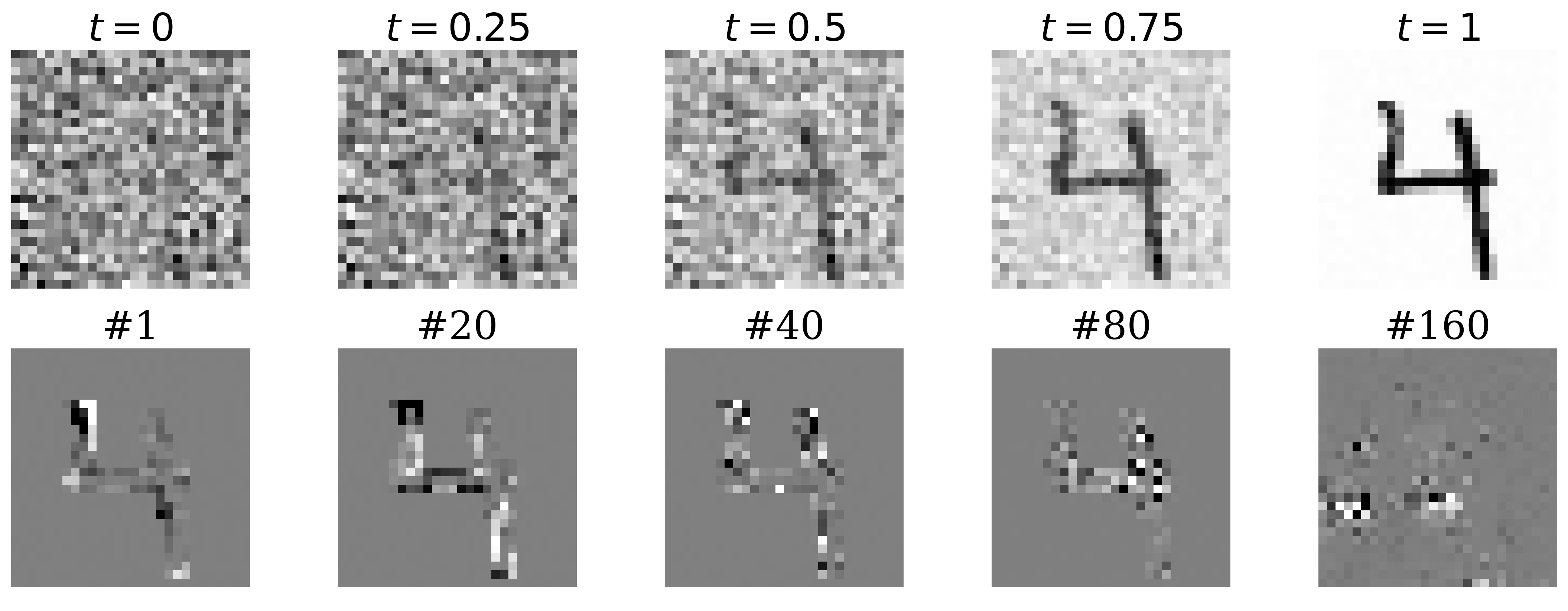}
    \caption{Top: Generation with OT-CFM. 
    Bottom: Eigenvectors of the Jacobian of the vector field $v_t(x)$ with respect to the 1st, 20th, 40th, 80th, and 160th eigenvalues in the descending order. 
    }
    \label{fig:MNIST_FM}
  \end{minipage}
\end{figure}

{\bf Toy examples. }
We applied the minibatch OT-CFM to the dataset ``Two Moons'' (Fig.~\ref{fig:two_moon}), where the target $P_1$ is uniform on the red curves. See Sec.~\ref{sec:experiments_add_toy} for an additional toy example.
The vector field is trained by an MLP with batch size 512. The Hungarian algorithm is applied for the minibatch OT. 
The left figure depicts the trained vector field, 
and the right figure shows the means and STDs of two eigenvalues of $D_x v_t(x_t)$ for different $t$ over 100 initial points.  The values are well separated with $\lambda_1 = 0$ and $\lambda_2 = -\gamma = -1.0$ for $t\to 1$, which agrees with Theorem~\ref{thm:terminal-lyap}.  Note that the example has two arcs with boundaries, at which the theoretical analysis in Sec.~\ref{sec:lyap-prox-epi} does not necessarily apply.

{\bf MNIST. }
We used MNIST images to train OT-CFM with minibatch size 256.  Fig.~\ref{fig:MNIST_FM}, bottom, presents the top eigenvectors (up to 160th) of $D_x v_t(x)$.  Up to the 80th eigenvalue, they have a strong correlation to the shape ``4''.  Fig.~\ref{fig:Eigenspectrum} shows the eigenvalues of $(1-t)D_x v(x_t)$ at $t=0.98$.  
We can see that there is a clear gap in the eigenspectrum; about 100 eigenvalues are significantly larger than the others.  The second right ($P_1$) and rightmost ($P_{100}$) images are perturbed from a generated image $x_1$ with eigenvectors of the 1st and 100th largest eigenvalue, respectively. The results suggest that the eigenspaces with the eigenvalues $\approx 0$ correspond to the tangential directions to the manifold that constitutes the MNIST images.  See also Sec.~\ref{sec:MNIST47}. 
In line with the assertion of Theorem \ref{thm:terminal-lyap}, We have also carried out experiments by perturbing the initial point and investigated its effects on the generated images. The perturbation causes only the change along the manifold of MNIST images, and we can also analyze the effective dimensionality with PCA. See Sec.~\ref{sec:robustness_MNIST} for the details.

{\bf CIFAR-10.} We conduct experiments on perturbations to initial points with CIFAR-10. Sinkhorn algorithm is applied to approximate OT with minibatch size 1024, and the UNet is used to learn a vector field. 
After training, perturbations are applied to initial points in two ways: (A) Uniformly random directions, (B) Singular directions: right singular vectors of the singular value decomposition of $D_x\widehat \Phi(x_0)$. See the results in Fig.~\ref{fig:cifar10}.  The dimensionality of the data manifold should be much smaller than the ambient dimension of 3072.  Most of the components in the random perturbation tend to be normal to the manifold, and they shrink at the terminal points. In contrast, the results show that the top singular vectors tend to be tangential to the manifold, resulting in small changes along the manifold.

\begin{figure}[tb]
  \centering
  \begin{minipage}[t]{0.48\linewidth}
  \vspace{18pt}
    \centering
      \includegraphics[width=0.87\linewidth]{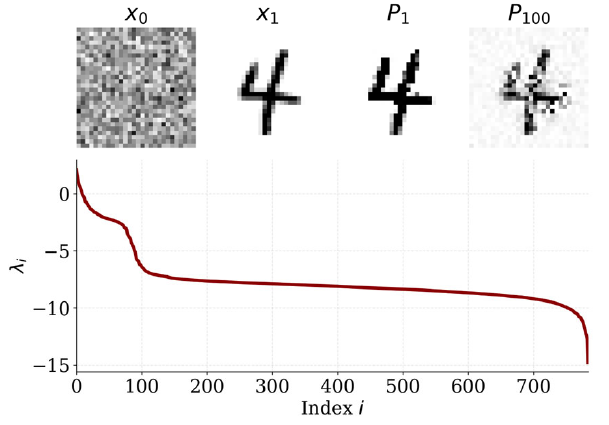}
      \vspace{12pt}
      \caption{
    {Top:} $x_1$ generated with $x_0\sim N(0,I_d)$, and perturbed images by adding to $x_1$ the eigenvectors $P_1$ and $P_{100}$. {Bottom:} sorted eigenvalues of $(1-t)D_x v(x_t)$ in the descending order.}
    \label{fig:Eigenspectrum}
  \end{minipage}\hfill
  \begin{minipage}[t]{0.48\linewidth}
  \vspace{0pt}
    \centering
      \includegraphics[width=1.0\linewidth]{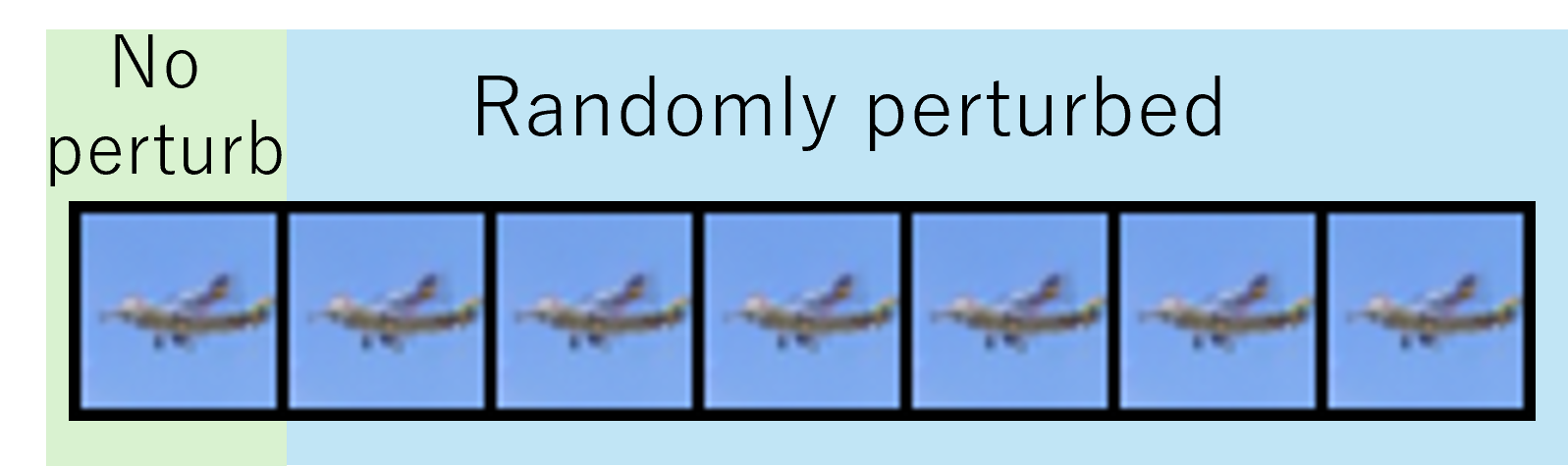}\\
      \includegraphics[width=1.0\linewidth]{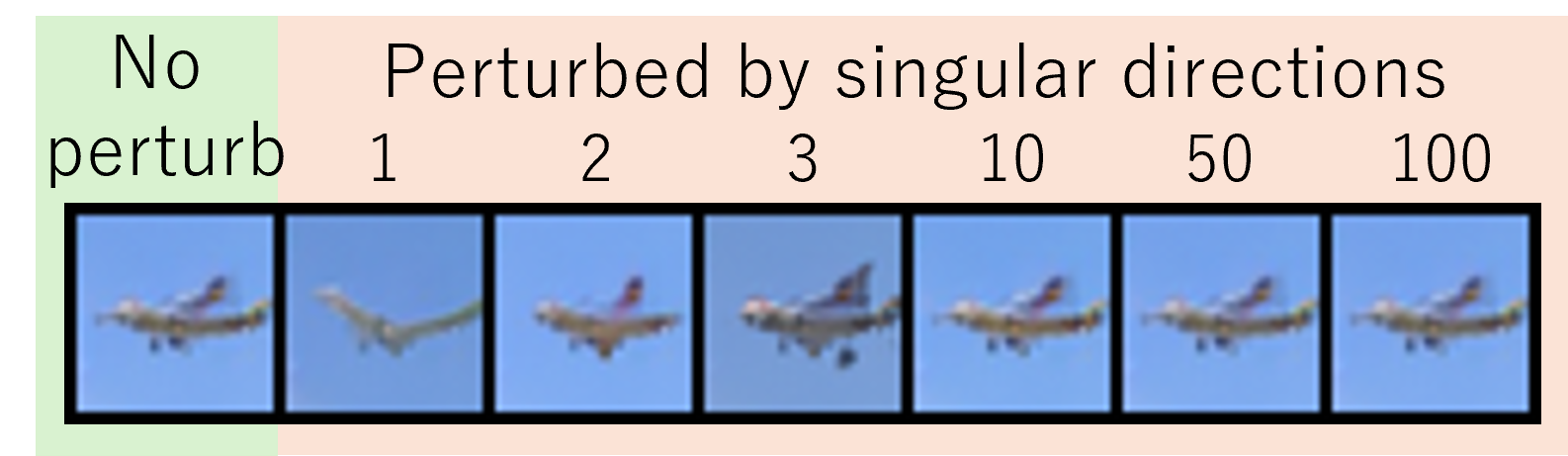}\\
      \caption{
    {Top:} Generated images with randomly perturbed initial points.  {Bottom:} Generated images $x^k_0$ with initial points by adding the rightsigular vector of $D_x\widehat\Phi(x_0)$. The number $k$ indicates the $k$-th largest singular values.  The random perturbations give almost the same images as the original, while the perturbations for the 1st, 2nd, and 3rd sigular vectors generate slightly modified images from the original. }
    \label{fig:cifar10}
  \end{minipage}
\end{figure}

{\bf Minibatch OT-CFM.}  Beyond the convergence of the vector field of minibatch OT in Sec.~\ref{sec:consistency}, we empirically compare the Lyapunov exponents between the minibatch and population OT-CFM. We describe the summary here and defer the details to Sec.~\ref{sec:exp_minibatch}.  The empirical eigenvalues of the Jacobian $(1-t)D_x v_t$ obtained by minibatch OT-CFM with uniform data on $\mathbb S^1$ in $\mathbb R^2$ do not differ significantly from the theoretical ones for population OT-CFM as $t\to 1$, while larger batch sizes converge more smoothly. The results provide empirical support for the convergence of the minibatch OT-CFM in the level of the derivative of vector fields.

\section{Conclusions}

We developed a proximal-operator framework for OT-based conditional FM. By representing the quadratic-cost OT via an extended Brenier potential, we obtained an exact proximal form of the FM vector field and its denoiser viewpoint, without requiring the target distribution to admit a density. 
We also established a terminal Lyapunov theory under the manifold hypothesis, which implies stability of the manifold structure against perturbation to the dynamics.

\section*{Acknowledgements}
KF is partially supported by JSPS KAKENHI 26K02990 and JSPS Grant-in-Aid for Transformative Research Areas (A) 22H05106. WH is supported by JSPS KAKENHI (24K20848) and JST BOOST (JPMJBY24G6). HB is supported by JST PRESTO (JPMJPR24K6). The authors also thank the anonymous reviewers for their constructive comments, particularly for drawing attention to relevant literature on active manifolds and suggesting refinements to the convergence proof.

\bibliographystyle{abbrvnat}
\bibliography{fukumizu}

\newpage
\appendix

\section{An example of OT under manifold hypothesis}
\label{app:von-mises-circle}

This appendix provides a two-dimensional example illustrating the
extended Brenier potential and its subgradient fibers when the target
distribution is supported on a one-dimensional manifold.  The example demonstrates that an OT fiber is parallel to the normal direction at the target manifold point. 

Let $M=\mathbb S^1  =  \{e_\theta=(\cos\theta,\sin\theta):\theta\in[0,2\pi)\}$ is the unit circle in $\R^2$,
and define $e_\theta^\perp := (-\sin\theta,\cos\theta)$.
We take $P_0=N(0,I_2)$ and $P_1$ is a von Mises distribution on \(\mathbb S^1\):
\[
    P_1(d\theta)=\rho_{\kappa,\mu}(\theta)\,d\theta,
    \qquad
    \rho_{\kappa,\mu}(\theta)
    =
    \frac{\exp\{\kappa\cos(\theta-\mu)\}}
         {2\pi I_0(\kappa)},
\]
where \(\kappa\ge 0\), \(\mu\in[0,2\pi)\), and \(I_0\) is the modified Bessel
function of the first kind.  The case \(\kappa=0\) corresponds to the uniform
distribution on the circle.  See Figure \ref{fig:vonMises_pdf} for an example of the density function $\rho_{\kappa,\mu}$. 

\begin{figure}[tb]
    \centering
    \includegraphics[width=0.6\linewidth]{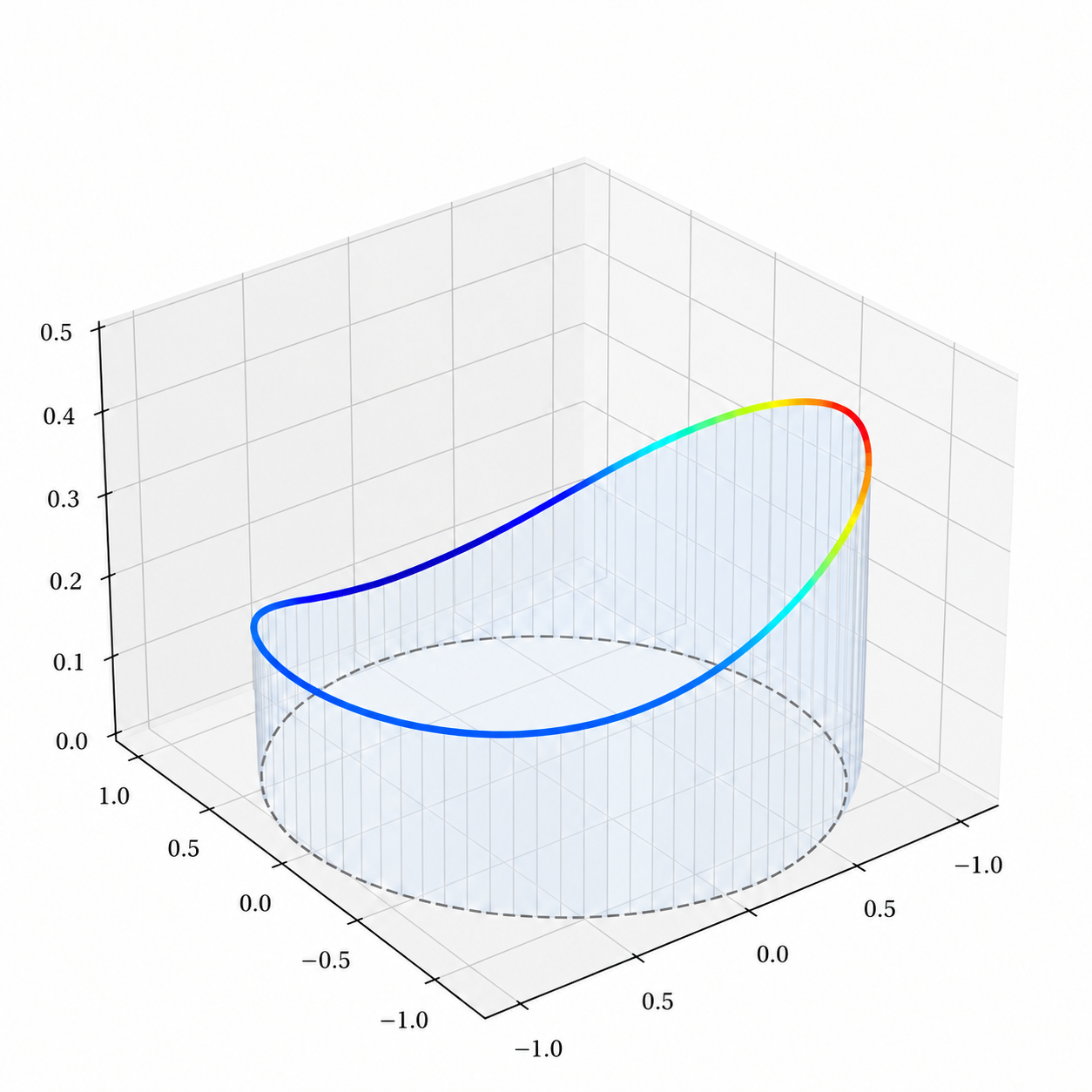}
    \caption{The probability density function of von Mises distribution ($\mu=0$, $\kappa=1$).  The peak is attained at $(1,0)$. }
    \label{fig:vonMises_pdf}
\end{figure}

\subsection{Dual representation of the OT map}

Consider the Kantorovich dual problem:
\[
\sup_{u,v}\left\{ \int_{\R^2} u(x)dP_0(x) + \int_{\mathbb S^1} v(e)dP_1(e): u(x) + v(e) \leq \frac12 \|x-e\|^2\right\},
\]
and let $u_\star$, $v_\star$ be the Kantorovich potential for $P_0$, $P_1$, respectively. If we introduce 
\[
a(\theta) := \frac12\|e_\theta\|^2 -v(e_\theta) = \frac12 -v(e_\theta)
\]
and 
\begin{equation}
      \Phi_a(x)  :=  \sup_{\theta\in[0,2\pi)}
    \bigl\{
        \langle x,e_\theta\rangle-a(\theta)
    \bigr\} \; = \sup_{\theta\in[0,2\pi)}
    \bigl\{
         v(e_\theta) - \frac12\|e_\theta\|^2 + \langle x,e_\theta\rangle
    \bigr\}, 
    \label{eq:vonMises_Phi}
\end{equation}
we have $u(x) = \frac12\|x\|^2 - \Phi_a(x)$, and the Kantorovich dual is equivalent to the convex minimization problem
\begin{equation}
     \min_a \mathcal J(a),
    \qquad
    \mathcal J(a)
    :=
    \int_{\mathbb R^2}
        \Phi_a(x)\,dP_0(x)
    +
    \int_0^{2\pi}
        a(\theta)\rho_{\kappa,\mu}(\theta)\,d\theta.   
    \label{eq:Kantorovich_equiv}
\end{equation}

It is also known that the Brenier map $S$ from $P_0$ to $P_1$ is given by 
\[
S(x) \in \arg\min_{e\in \mathbb S^1} \left\{ \frac12 \|x-e\|^2 - v_\star(e) \right\},
\]
and the above expression can be rewritten as
\[
S(x) \in \arg\max_{\theta \in [0,2\pi)} \langle x,e_\theta \rangle - a_\star(\theta),
\]
where $a_\star$ corresponds to $v_\star$ and 
\[
    S(x)=e_{\Theta_{a_\star}(x)},
    \qquad
    \Theta_{a_\star}(x)\in
    \arg\max_{\theta\in[0,2\pi)}
    \bigl\{
        \langle x,e_\theta\rangle-a_\star(\theta)
    \bigr\}.
\]

The function $\Phi_{a_\star}$ is convex, and $1$-Lipschitz because $|\Phi_a(x)-\Phi_a(x')|\leq \sup_\theta |\langle x-x',e_\theta \rangle |\leq \|x-x'\|$. From Rademacher's theorem, $\Phi_{a_\star}$ is differentiable almost everywhere in Lebesgue measure, and also differentiable almost everywhere in $P_0$.
Consequently, $\arg\max_{\theta\in[0,2\pi)}\{ \langle x,e_\theta\rangle-a(\theta)\}$ is unique $P_0$-a.e., and 
\[
    S(x)=\nabla\Phi_{a_\star}(x)=e_{\Theta_{_\star}(x)}, \qquad P_0\text{-a.e.}
\]

The corresponding extended-Brenier potential for the reverse transport
relation \(P_1\to P_0\) can be taken by the convex conjugate
\[
    \phi_a = \Phi_a^*, 
\]
or equivalently, from \eqref{eq:vonMises_Phi}, 
\[
    \phi_a  =  (a+\iota_{\mathbb S^1})^{**},
\]
where \(a+\iota_{\mathbb S^1}\) denotes the $(\R\cup\{+\infty\})$-valued function that takes the value \(a(\theta)\) at \(e_\theta\in\mathbb S^1\) and \(+\infty\)
outside \(\mathbb S^1\).  Thus \(\phi_a\) is the lower-semicontinuous convex
envelope of the boundary data \(a\) on the circle.

By Fenchel duality,
\[
    x_0\in \partial \phi_{a_\star}(e_\theta)
    \quad\Longleftrightarrow\quad
    e_\theta\in \partial \Phi_{a_\star}(x_0).
\]
Therefore, the fiber \(\partial\phi_{a_\star}(e_\theta)\) is exactly the set of
source points \(x_0\) for which \(e_\theta\) is the target point of OT.

\subsection{Subgradient fibers on the circle}

This subsection calculates the explicit form of the fiber $\partial \phi_a(e_\theta)$.  Assume that \(a\) is \(C^2\) and that \(e_\theta\) is an active boundary point of the convex envelope.  From the first-order condition of the maximization of \eqref{eq:vonMises_Phi}, we have 
$0= \langle x,e_\theta^\perp\rangle-a'(\theta)$, and hence, every point in the fiber has the form
\[
    x  =  a'(\theta)e_\theta^\perp+s e_\theta
\]
for some scalar \(s\in\mathbb R\).
The global maximality condition is then given by 
\[
    \langle x,e_\theta\rangle-a(\theta)
    \ge
    \langle x,e_\eta\rangle-a(\eta)
    \qquad
    \text{for all }\eta.
\]
Substituting \(x=a'(\theta)e_\theta^\perp+s e_\theta\), we obtain
\[
    s\{1-\cos(\eta-\theta)\}
    \ge
    a(\theta)-a(\eta)
    +
    a'(\theta)\sin(\eta-\theta).
\]
Therefore, in full generality, the fiber can be written as
\[
    \partial\phi_a(e_\theta)  =
    \left\{
        a'(\theta)e_\theta^\perp+s e_\theta :  s\ge r_a(\theta)
    \right\},
    \label{eq:fiber_general}
\]
where
\[
    r_a(\theta) :=  \sup_{\eta\ne\theta}
    \frac{ a(\theta)-a(\eta)+a'(\theta)\sin(\eta-\theta) }{1-\cos(\eta-\theta)
    }.
    \label{eq:r_a}
\]

This formula makes the tangent--normal structure transparent.  Since $T_{e_\theta}\mathbb S^1=\mathbb R e_\theta^\perp$ and $N_{e_\theta}\mathbb S^1=\mathbb R e_\theta$, the parallel subspace of the fiber is
\[
    \operatorname{par}\bigl(\partial\phi_a(e_\theta)\bigr)
    =
    \mathbb R e_\theta
    =
    N_{e_\theta}\mathbb S^1.
\]
Thus, the fiber direction is normal to the circle.  On the other hand, an individual OT path is
\[
    x_0-e_\theta
    =
    a'(\theta)e_\theta^\perp+(s-1)e_\theta,
\]
which contains a tangential component \(a'(\theta)e_\theta^\perp\) and not normal to the manifold in general.

\subsection{Numerical computation of fibers}
\label{app:numerical-von-mises-fibers}

We describe the numerical procedure used to compute the
subgradient of the extended Brenier potential for von Mises distribution. 
The numerical example below uses $\kappa=1$ and $\mu=0$.

\paragraph{Fourier approximation of the dual potential.}
Since \(\mu=0\), both \(P_0\) and \(P_1\) are symmetric across the first coordinate axis.  We therefore use the truncated Fourier approximation
\[
    a_M(\theta)  = \sum_{m=1}^{M} c_m\cos(m\theta).
\]
We can omit the constant term by fixing the additive
normalization of $a$.  

We discretize the circle by $\theta_j=\frac{2\pi j}{J}$ ($j=0,\ldots,J-1$),
and set $e_j:=e_{\theta_j}$. 
The quadrature weights for \(P_1\) are
\[
    b_j := \frac{\rho_{\kappa,\mu}(\theta_j) }{  \sum_{\ell=0}^{J-1}
        \rho_{\kappa,\mu}(\theta_\ell)  },
    \qquad
    \sum_{j=0}^{J-1} b_j=1.
\]
We independently generate $X_1,\ldots,X_N\sim N(0,I_2)$ 
and approximate \(\mathcal J\)  \eqref{eq:Kantorovich_equiv} by
\begin{equation}
    \widehat{\mathcal J}_{N,J}(c)  =
    \frac1N \sum_{i=1}^{N} \max_{0\le j<J}
    \left\{  \langle X_i,e_j\rangle  -  a_M(\theta_j) \right\}
    +  \sum_{j=0}^{J-1} b_j a_M(\theta_j).
\label{eq:Kantorovich_equiv_approx}    
\end{equation}
This is a finite-dimensional convex function of \(c=(c_1,\ldots,c_M)\), but it is non-smooth because of the maximum.

\paragraph{Soft-max approximation and continuation.}

To optimize \eqref{eq:Kantorovich_equiv_approx} numerically, we approximate it by replacing the maximum with the soft-max:
\begin{equation}
      \operatorname{smax}_{\tau}(z_0,\ldots,z_{J-1})
    := \tau \log \left( \sum_{j=0}^{J-1}e^{z_j/\tau} \right),
    \qquad \tau>0. 
    \label{eq:softmax}
\end{equation}
It satisfies
\[ 
    \max_j z_j  \le  \operatorname{smax}_{\tau}(z)
    \le
    \max_j z_j+\tau\log J.
\]
Thus the approximation error converges uniformly to zero as
\(\tau\downarrow0\).
The resulting smooth objective is
\begin{equation}
    \widehat{\mathcal J}_{N,J,\tau}(c)
    =
    \frac{\tau}{N} \sum_{i=1}^{N} \log
    \left[
        \sum_{j=0}^{J-1} \exp \left\{
            \frac{\langle X_i,e_j\rangle  - a_M(\theta_j)}{\tau}  \right\}
    \right]
    +  \sum_{j=0}^{J-1} b_j a_M(\theta_j).
     \label{eq:Kantorovich_approx_smax}
\end{equation}
We minimize \eqref{eq:Kantorovich_approx_smax} by a quasi-Newton method such as L-BFGS.  To reduce the smoothing bias, we use a continuation scheme
$ \tau_1>\tau_2>\cdots>\tau_L>0$, using the solution obtained at \(\tau_\ell\) as the initial value for the optimization at \(\tau_{\ell+1}\).




\paragraph{Global computation of the fiber threshold.}

For a fixed \(\theta\), every point in the affine line containing
the fiber has the form
\[
    x_{\theta,s}  :=
    a_M'(\theta)e_\theta^\perp+s e_\theta. 
\]
Let 
\[
    q_\theta(\eta)  := \frac{a_M(\theta)-a_M(\eta)  +
        a_M'(\theta)\sin(\eta-\theta) }{ 1-\cos(\eta-\theta) },
\]
where $q_\theta(\theta)$ is defined by $-a_M''(\theta)$ at $\eta=\theta$, and define 
\[
    r_{a_M}(\theta) :=    \max_{\eta\in[0,2\pi)}
         q_\theta(\eta).
\]
The resulting numerical fiber is given by 
\begin{equation}
   \partial\phi_a(e_\theta) \approx
    \left\{ a_M'(\theta)e_\theta^\perp+s e_\theta: s\ge r_{a_M}(\theta)
    \right\}. 
\end{equation}

Figure \ref{fig:vonMises_fiber} is made by the above method 
for
\[
    \theta \in  \left\{  0,\, \frac{\pi}{4},\, \frac{\pi}{2},\,  \pi,\,
        \frac{3\pi}{2}  \right\}.
\]

\section{Semiderivatives and Proximal Operators}

We summarize some basic definitions in convex analysis: semiderivatives and proximal operators used in this paper.  See \citet{ParikhBoyd_proximal} and \citet{RockafellarVariatioal} for more details. 

\label{sec:prox}

For a mapping $G:\R^d\to\R^d$, the \emph{semiderivative} of $G$ at $x$ in direction $h$ is defined by 
\begin{equation}
    DG(x;h)\coloneqq \lim_{\substack{\varepsilon\downarrow 0\\ h'\to h}}
\frac{G(x+\varepsilon h')-G(x)}{\varepsilon},
\end{equation}
whenever the limit exists.
This is a generalization of the directional derivative with a one-sided limit.

For a convex function $F : \mathbb{R}^d \to \overline{\R}$
and $\lambda > 0$, the \emph{proximal operator} of $F$ with parameter $\lambda$ is defined by 
\begin{equation}
    \prox_{\lambda F}(y)
    = \arg\min_x \left\{
        F(x) + \frac{1}{2\lambda}\|x - y\|^2
    \right\}.
\end{equation}

The proximal operator is often used for optimization steps. A key identity from this viewpoint is, if $F$ is proper, lsc, convex, and differentiable, 
\[
    x = y- \lambda \nabla F(x)
    \quad\Longleftrightarrow\quad
    x = \prox_{\lambda F}(y).
\]
Thus, $x_{t+1}=\prox_{\lambda F}(x_t)$ can be regarded as the  implicit Euler step for the gradient flow $\dot{x} = -\nabla F(x)$. 

The \emph{Moreau envelope} (or \emph{Moreau--Yosida regularization}) of $F$ with parameter $\lambda$ is defined by 
\begin{equation}
    M_{\lambda F}(y)
    = \min_x \left\{
        F(x) + \frac{1}{2\lambda}\|x - y\|^2
    \right\}.
\end{equation}
If $F$ is convex, $x\mapsto  F(x) + \frac{1}{2\lambda}\|x - y\|^2$ is strictly convex, so the proximal operator $\prox_{\lambda F}(y)$ is the unique minimizer to give $M_{\lambda F}(y)$. 

\medskip 
The proximal operator of a convex function satisfies the following firm nonexpansiveness. 
\begin{proposition}[Firm nonexpansiveness of the proximal operator]
\label{prop:prox-nonexpansiveness}
Let $f:\mathbb{R}^d\to\mathbb{R}\cup\{+\infty\}$ be proper,
lower semicontinuous, and convex. Then, for $\lambda>0$, 
$\operatorname{prox}_{\lambda f}$ is firmly nonexpansive: for all
$x,y\in\mathbb{R}^d$,
\[
\|\operatorname{prox}_{\lambda f}(x)-\operatorname{prox}_{\lambda f}(y)\|^2
\le
\langle \operatorname{prox}_{\lambda f}(x)-\operatorname{prox}_{\lambda f}(y),x-y\rangle.
\]
In particular, 
\[
\|\operatorname{prox}_{\lambda f}(x)-\operatorname{prox}_{\lambda f}(y)\|
\le \|x-y\|.
\]
\end{proposition}
\begin{proof}
Let $P_\lambda(x) := \operatorname{prox}_{\lambda f}(x)$ for simplicity. 
Fix $x,y\in\mathbb{R}^d$ and write
$p=P_\lambda(x)$ and $q=P_\lambda(y)$.
The proximal optimality conditions give
\[
\frac{x-p}{\lambda}\in\partial f(p),
\qquad
\frac{y-q}{\lambda}\in\partial f(q).
\]
By the monotonicity of the convex subdifferential,
\[
0
\le
\left\langle
p-q,\frac{x-p}{\lambda}-\frac{y-q}{\lambda}
\right\rangle
=
\frac{1}{\lambda}
\left(
\langle p-q,x-y\rangle-\|p-q\|^2
\right).
\]
This proves firm nonexpansiveness. By the Cauchy--Schwarz
inequality,
\[
\|p-q\|^2
\le \langle p-q,x-y\rangle
\le \|p-q\|\,\|x-y\|.
\]
If $p\ne q$, division by $\|p-q\|$ yields nonexpansiveness;
if $p=q$, the conclusion is immediate.
\end{proof}

\medskip 


We recall two standard regularity notions from variational analysis (see, e.g., \citet[Chs.~1 \& 13]{RockafellarVariatioal}).
Let $f:\mathbb{R}^d\to\overline{\mathbb{R}}$ be a proper lower-semicontinuous (lsc) function.

\begin{definition}[Prox-boundedness]\label{def:prox-bounded}
The function $f$ is \emph{prox-bounded} if there exists $\lambda>0$ such that
$M_{\lambda f}(x)>-\infty$ for at least one $x\in\mathbb{R}^d$
(equivalently, $M_{\lambda f}$ is not identically $-\infty$).
The \emph{prox-threshold} of $f$ is defined as
\[
\bar\lambda_f
\coloneqq \sup\Bigl\{\lambda>0:\exists x\ \text{s.t.}\ M_{\lambda f}(x)>-\infty\Bigr\}
\in (0,+\infty].
\]
Intuitively, prox-boundedness means that $f$ does not decrease faster than a quadratic,
so the proximal subproblem is well-posed for sufficiently small step sizes
$\lambda<\bar\lambda_f$.
\end{definition}

\begin{definition}[Prox-regularity]\label{def:prox-regular}
Let $\bar x\in\dom f$ and $\bar v\in\partial f(\bar x)$.
We say that $f$ is \emph{prox-regular at $\bar x$ for $\bar v$} if there exist
$\varepsilon>0$ and $\rho\ge 0$ such that
\begin{equation}\label{eq:prox-regular}
f(x') \ \ge\ f(x)+\langle v, x'-x\rangle - \frac{\rho}{2}\|x'-x\|^2
\end{equation}
holds for all $x,x'\in B(\bar x,\varepsilon)$ and all $v\in\partial f(x)$ satisfying
$\|v-\bar v\|<\varepsilon$ and $|f(x)-f(\bar x)|<\varepsilon$.
\end{definition}

Inequality \eqref{eq:prox-regular} is a ``convexity up to a quadratic'' condition:
it says that, locally and along nearby subgradients, the function admits a quadratic
supporting model. Prox-regularity is a standard assumption guaranteeing good local behavior
of the proximal mapping (e.g.\ local single-valuedness and Lipschitz properties), which is
used in second-order variational characterizations of $\prox_{\lambda f}$ and its directional
derivatives.

\begin{remark}[Convex case]\label{rem:convex-prox}
If $f$ is proper lsc\ and \emph{convex}, then \eqref{eq:prox-regular} holds globally with
$\rho=0$ by the subgradient inequality, hence $f$ is prox-regular everywhere.
Moreover, convexity implies the existence of an affine minorant (e.g.\ if $v\in\partial f(\bar x)$,
then $f(y)\ge f(\bar x)+\langle v,y-\bar x\rangle$ for all $y$), which in turn implies that
$M_{\lambda f}(x)>-\infty$ for every $\lambda>0$ and $x\in\mathbb{R}^d$ (i.e.\ $\bar\lambda_f=+\infty$).  Thus, for a proper, lsc, convex function $f$, the prox-boundedness and prox-regularity are straightforward. 
\end{remark}

\section{Details of the Lyapunov analysis in Section \ref{sec:lyap-prox-epi}}
\label{sec:dom-second-subderivative}

This appendix recalls partial smoothness, proves the proximal
sensitivity result used in Section~\ref{sec:proximal-sensitivity}, and establishes almost-sure sharpness and nondegeneracy under an absolutely
continuous source.

\subsection{Partly smoothness}
\label{sec:partly_smooth}
We show the definition of partly smoothness.  See \cite{HareLewis2004Identification} for details.  Assume that $M$ is a $C^p$ submanifold embedded in $\R^d$, and $f:\R^d\to \R\cup\{\pm\infty\}$ is a function.
\begin{definition}
Let $x$ be a point of $M$. A function $f$ is \emph{$C^p$-partly smooth} at $x$ relative to $M$ if the following conditions hold:
\begin{enumerate}[topsep=0pt,partopsep=0pt,itemsep=1pt,parsep=0pt,leftmargin=36pt]
\item[(i)] (smoothness) $f|_{M}$ is a $C^p$ function in a neighborhood of $x$;
\item[(ii)] (regularity) $f$ is regular at all points $x \in M$ in a neighborhood of $x$, with $\partial f(x) \neq \emptyset$;
\item[(iii)] (sharpness) the affine span of $\partial f(x)$ is a translate of the normal space $N_x M$;
\item[(iv)] (sub-continuity) $\partial f(x)$ restricted to $M$ is continuous at $x$.
\end{enumerate}
\end{definition}

In (ii), a function \(f\) is said to be \emph{regular} at \(\bar x\) where \(f(\bar x)\) is finite if its epigraph is locally closed at \((\bar x,f(\bar x))\) and the Fr\'echet and limiting normal cones to its epigraph coincide there. In particular, a proper lsc convex function is regular at every point of its domain.  Therefore, the function $\phi$ and its shift $f=\phi-\langle x_0,\cdot\rangle$ in Section \ref{sec:lyap-prox-epi} are regular.  Regarding (iii), a sufficient condition for sharpness is discussed in Section \ref{sec:sharpness-nondegeneracy}. 

\subsection{Proof of Proposition~\ref{prop:proximal-sensitivity}}
\label{sec:proof-proximal-sensitivity}

We prove Proposition~\ref{prop:proximal-sensitivity} by applying the
identification theorem of
\citet[Theorem~5.3]{HareLewis2004Identification} and differentiating the
manifold-restricted proximal problem, as in
\citet[Theorems~28--29]{DaniilidisHareMalick2006}.
The latter paper states its general prox-regular result for sufficiently
small proximal parameters. The argument below explains why every fixed
$\lambda>0$ is allowed in the present convex setting.

\begin{proof}[Proof of Proposition~\ref{prop:proximal-sensitivity}.] 

{\bf Identification for each fixed parameter.}
Fix $\lambda>0$ and recall $\Xi_\lambda =\operatorname{prox}_{\lambda f}$. 
From $\operatorname{prox}_{\lambda_t \phi}(x_t/t) = x_1$, it is straightforward to see $\Xi_\lambda(x_1) = x_1$.  We first show that for $x_1\in M$ there is an open neighborhood $U_\lambda$ of $x_1$ in $\R^d$ such that $\Xi_\lambda(U_\lambda)\subset M$.  

Suppose $y_k\to x_1$ and set $z_k=\Xi_\lambda(y_k)$.
Nonexpansiveness of $\Xi_\lambda$ (Prop.~\ref{prop:prox-nonexpansiveness}) and the fact $\Xi_\lambda(x_1)=x_1$ yield
$\|z_k-x_1\|\le\|y_k-x_1\|$, hence $z_k\to x_1$.
By definition of proximal operator, we have $f(z_k)+\frac{1}{2\lambda}\|y_k-z_k\|^2
 \le f(x_1)+\frac{1}{2\lambda}\|y_k-x_1\|^2$.  Also, $f(x_1)\le f(z_k)$ from $0\in\partial f(x_1)$.  Accordingly, 
\[
 f(x_1)\le f(z_k) \le  f(x_1)+\frac{1}{2\lambda}\|y_k-x_1\|^2,
\]
so that $f(z_k)\to f(x_1)$. Proximal optimality gives $0 \in \partial f(z_k) + \frac{1}{\lambda}(z_k-y_k)$, and thus 
$w_k:=\frac{y_k-z_k}{\lambda}\in\partial f(z_k)$ satisfies $w_k\to0$.  This implies 
\[
\textrm{dist}\bigl(0, \partial f(z_k)\bigr) \leq \textrm{dist}(0,w_k) \to 0 \quad (k\to \infty).
\]
A proper lsc convex function is prox-regular. Thus
\citet[Theorem~5.3]{HareLewis2004Identification}, with
\textnormal{(PS)--(ND)}, implies $z_k\in\M$ for all sufficiently large
$k$. Since this applies to every sequence $y_k\to x_1$, there is a
neighborhood $U_\lambda$ of $x_1$ such that $\Xi_\lambda(U_\lambda)\subset\M$.

\medskip
{\bf Smoothness and the Jacobian.}
Let $V \subset\mathbb R^m$ be an open neighborhood of $0$ and $\chi:V\to\M$ be a $C^2$ local chart with $\chi(0)=x_1$ and $D\chi(0)=R$, where the columns of $R$ form an
orthonormal basis of $T=T_{x_1}\M$, i.e., $R^\top R =I_m$ . Write $g=f\circ\chi$.
Since $0\in \partial f(x_1)$ and $f|_M$ is $C^2$, $g$ takes a minimum at $0$ and 
\[
 \nabla g(0)=0,\qquad \nabla^2g(0)=R^\top HR\succeq0
\]
hold.  For $y$ near $x_1$, the identified proximal minimizer is a critical
point of
\[
 s\longmapsto\lambda g(s)+\tfrac12\|\chi(s)-y\|^2.
\]
Its stationarity equation is
\[
 G(s,y):=\lambda\nabla g(s)
     +D\chi(s)^\top (\chi(s)-y)=0.
\]
At $(s,y)=(0,x_1)$, using the assumption $R^\top R=I_m$, we have 
\[
 D_sG(0,x_1)=I_m+\lambda R^\top HR\succ0,
 \qquad D_yG(0,x_1)=-R^\top.
\]
The implicit function theorem supplies a unique $C^1$ local critical
branch $s=s(y)$. 
By identification and continuity of \(\Xi_\lambda\), after shrinking the neighborhood of \(x_1\) if necessary, the coordinate \(\widehat s(y):=\chi^{-1}(\Xi_\lambda(y))\) is well-defined and satisfies \(\widehat s(y)\to0\) as \(y\to x_1\). Since \(\Xi_\lambda(y)\) minimizes the original proximal objective, \(\widehat s(y)\) also minimizes the manifold-restricted objective and hence satisfies \(G(\widehat s(y),y)=0\). Local uniqueness in the implicit function theorem then implies \(\widehat s(y)=s(y)\), and hence \(\Xi_\lambda(y)=\chi(s(y))\).
 Therefore, by taking the derivative, we have 
\[
 D\Xi_\lambda(x_1)
 =R(I_m+\lambda R^\top HR)^{-1}R^\top 
 =(I_T+\lambda H)^{-1}\Pi_T.
\]
Since $H\succeq0$, we have $\|(I_T+\lambda H)^{-1}\|_{\mathrm{op}}\le1$, and consequently,
\[
 \|D\Xi_\lambda(x_1)-\Pi_T\|_{\mathrm{op}}
 =\|\lambda H(I_T+\lambda H)^{-1}\|_{\mathrm{op}}
 \le\lambda\|H\|_{\mathrm{op}}.
\]
This proves the proposition. 

Finally, fix \(\lambda_*>0\). We show that identification holds on a common neighborhood of \((\lambda_*,x_1)\). Let \(\lambda_k\to\lambda_*\) and \(y_k\to x_1\), and set \(z_k=\Xi_{\lambda_k}(y_k)\). Since \(\lambda_k\) is bounded away from zero, the preceding estimates remain valid and give
\[
z_k\to x_1,\qquad
f(z_k)\to f(x_1),\qquad
\operatorname{dist}(0,\partial f(z_k))\to0.
\]
The identification theorem therefore implies \(z_k\in M\) for all sufficiently large \(k\). Consequently, there exist an open interval \(J\subset(0,\infty)\) containing \(\lambda_*\) and a neighborhood \(U\) of \(x_1\) such that
\[
\Xi_\lambda(y)\in M
\qquad\text{for every }(\lambda,y)\in J\times U.
\]

We now include \(\lambda\) explicitly as a parameter in the stationarity equation:
\[
G(s,\lambda,y)
:=
\lambda\nabla g(s)
+D\chi(s)^\top\bigl(\chi(s)-y\bigr).
\]
This map is \(C^1\), and
\[
G(0,\lambda_*,x_1)=0,\qquad
D_sG(0,\lambda_*,x_1)
=
I_m+\lambda_*R^*HR\succ0.
\]
The implicit function theorem thus provides a unique local \(C^1\) solution \(s=s(\lambda,y)\) to \(G(s,\lambda,y)=0\). After shrinking \(J\) and \(U\), the identification above and the bound $\|\Xi_\lambda(y)-x_1\|\le\|y-x_1\|$
ensures that the proximal point lies in the chart neighborhood where this solution is unique. Its chart coordinates satisfy \(G=0\) by the first-order optimality condition for the manifold-restricted problem. Hence
\[
\Xi_\lambda(y)=\chi\bigl(s(\lambda,y)\bigr)
\qquad ((\lambda,y)\in J\times U),
\]
which proves that \((\lambda,y)\mapsto\Xi_\lambda(y)\) is jointly \(C^1\) near \((\lambda_*,x_1)\).
\end{proof}
\begin{remark}
For finite-valued convex functions satisfying \textnormal{(ND)} and
admitting a fast track,
\citet[Theorem~5.2]{Mifflin2002-ug} establish local identification
by the proximal map. Together with the equivalence between partial
smoothness and the existence of a fast track in this setting
\citep[Corollary~2.10]{Miller2005-qi}, this yields
$\Xi_\lambda(U_\lambda)\subset\M$ under \textnormal{(PS)--(ND)}
for every fixed $\lambda>0$.
The proof above applies the identification theorem for prox-regular
functions \citep[Theorem~5.3]{HareLewis2004Identification}
to the possibly extended-real-valued convex function $f$
considered here.
\end{remark}

\subsection{Almost-sure sharpness and nondegeneracy}
\label{sec:sharpness-nondegeneracy}
\label{sec:ri-subgradient-disintegration}

We derive a sufficient condition for the sharpness component of partly smoothness \textup{(PS)} and for nondegeneracy \textup{(ND)} to hold almost surely.
When $\phi|_M$ is $C^2$, absolute continuity of the source distribution
implies absolute continuity of the conditional source laws on the
corresponding affine normal spaces. This also determines the dimension
of the subgradient fibers.

\begin{lemma}
\label{lem:ac-source-sharpness-nd}
Let $M\subset\mathbb{R}^d$ be a $C^2$ embedded submanifold without
boundary, of dimension $m<d$, and set $q=d-m$.
Let $\phi:\mathbb{R}^d\to(-\infty,+\infty]$ be proper, lower
semicontinuous, and convex, and assume that $h:=\phi|_M$ is
finite-valued and of class $C^2$.
Let $\pi$ be a probability measure on
$\mathbb{R}^d\times\mathbb{R}^d$ with first marginal $P_1$
and second marginal $P_0$, such that
\begin{equation}
\label{eq:ac-source-graph}
P_1(M)=1,\qquad
P_0\ll\mathcal{L}^d,\qquad
\pi\bigl(\{(x_1,x_0):x_0\in\partial\phi(x_1)\}\bigr)=1.
\end{equation}
Write
\[
\pi(dx_1,dx_0)=P_1(dx_1)\,\pi_{x_1}(dx_0),
\qquad
A_y:=\nabla_M h(y)+N_yM\quad(y\in M),
\]
where $\nabla_M$ denotes the gradient for the induced Riemannian
metric. Then, for $P_1$-almost every $x_1$,
\begin{equation}
\label{eq:ac-source-fibers}
\pi_{x_1}\ll\mathcal{H}^q\!\restriction_{A_{x_1}},
\qquad
\operatorname{aff}\partial\phi(x_1)=A_{x_1},
\end{equation}
where $\mathcal{H}^q$ is the Hausdorff measure and $\mathcal{H}^q\!\restriction_{A_{x_1}}$ is its restriction on $A_{x_1}$.
In particular,
$\operatorname{par}\partial\phi(x_1)=N_{x_1}M$
for $P_1$-almost every $x_1$, and
\begin{equation}
\label{eq:ac-source-nd}
x_0\in\operatorname{ri}\partial\phi(x_1)
\qquad\text{for $\pi$-almost every $(x_1,x_0)$.}
\end{equation}
\end{lemma}

\begin{proof}
Put $g(y):=\nabla_M h(y)$.
For any $x\in\partial\phi(y)$ and $v\in T_yM$, apply the
subgradient inequality along a $C^1$ curve $\gamma$ in $M$ through
$y$ with velocity $v$:
\[
\phi(\gamma(t)) - \phi(y) \geq \langle x, \gamma(t)-y\rangle. 
\]
Differentiation from both sides of $t\to 0$ gives
\[
\langle x,v\rangle=Dh(y)[v]=\langle g(y),v\rangle.
\]
Consequently,
\begin{equation}\label{eq:tangential_identity}
   P_{T_yM}x=g(y), 
\qquad
\partial\phi(y)\subset g(y)+N_yM=A_y.
\end{equation}

Choose a $C^2$ chart $Y:U\to O\subset M$, with
$U\subset\mathbb{R}^m$ open.
After shrinking the chart if necessary, choose a $C^1$
orthonormal normal frame
$Q(s):\mathbb{R}^q\to N_{Y(s)}M$.
Define
\begin{equation*}
    F(s,a):=g(Y(s))+Q(s)a,
\qquad (s,a)\in U\times\mathbb{R}^q.
\end{equation*}
The map $F$ is $C^1$, hence locally Lipschitz, between spaces
of the same dimension $m+q=d$.

Let $\mu$ be the pushforward of the unnormalized finite positive measure
$\pi\!\restriction_{O\times\mathbb{R}^d}$ under
\[
\Gamma(x_1,x_0):=
\left(
Y^{-1}(x_1),
Q(Y^{-1}(x_1))^\top(x_0-g(x_1))
\right).
\]
The preceding tangential identity \eqref{eq:tangential_identity} tells $x_0 -g(x_1)\in N_{x_1}M$.  From the assumption \eqref{eq:ac-source-graph}, it holds that 
$F(\Gamma(x_1,x_0))=g(x_1) + P_{N_{x_1}M}(x_0-g(x_1)) = x_0$ almost surely on this chart.
Thus, for every Borel set $B\subset\mathbb{R}^d$,
\[
(F_\#\mu)(B)=\pi(O\times B)\le P_0(B).
\]

We first show the absolute continuity $\mu\ll\mathcal{L}^{m+q}$.
Let $E\subset U\times\mathbb{R}^q$ be a Borel Lebesgue-null
set. Since $F$ is locally Lipschitz between $d$-dimensional
spaces, $F(E)$ has Lebesgue outer measure zero.
Choose a Borel Lebesgue-null set $B$ containing $F(E)$.
Then
\[
\mu(E)
\le \mu(F^{-1}(B))
=(F_\#\mu)(B)
\le P_0(B)
=0.
\]
This proves the claim $\mu\ll\mathcal{L}^{m+q}$.

Therefore, there is a density $\rho$ such that 
\[
\mu(ds,da)=\rho(s,a)\,ds\,da.
\]
With the marginal density 
\[
r(s):=\int_{\mathbb{R}^q}\rho(s,a)\,da.
\]
the $s$-marginal of $\mu$ is given by 
\[
\nu(ds):=r(s)\,ds=(Y^{-1})_\#(P_1\!\restriction_O).
\]
For $\nu$-almost every $s$, one has $0<r(s)<\infty$, and the
conditional distribution of $a$ given $s$ is given by 
\[
\kappa_s(da)=\frac{\rho(s,a)}{r(s)}\,da.
\]
By uniqueness of disintegration,
\[
\pi_{Y(s)}=(F_s)_\#\kappa_s,
\qquad
F_s(a):=F(s,a),
\quad\text{for $\nu$-almost every $s$.}
\]
For each fixed $s$, the map $F_s$ is an affine isometry from
$\mathbb{R}^q$ onto $A_{Y(s)}$. Hence
\[
(F_s)_\#\mathcal{L}^q
=\mathcal{H}^q\!\restriction_{A_{Y(s)}},
\]
with the usual normalization of Hausdorff measure.
It follows that
$\pi_y\ll\mathcal{H}^q\!\restriction_{A_y}$
for $P_1$-almost every $y\in O$.
A countable cover by such charts proves this assertion on $M$.

The graph-support assumption also gives
$\pi_y(\partial\phi(y))=1$ for $P_1$-almost every $y$.
For such a $y$, the closed convex set
$\partial\phi(y)\subset A_y$ cannot have affine dimension
less than $q$: otherwise it would have zero
$\mathcal{H}^q$-measure, contradicting the conditional absolute
continuity just proved. Therefore
\[
\operatorname{aff}\partial\phi(y)=A_y,
\qquad
\operatorname{par}\partial\phi(y)=N_yM.
\]
Finally, the relative boundary of a nonempty closed convex set
has zero Hausdorff measure of the dimension of its affine hull.
Thus
\[
\pi_y\bigl(\operatorname{ri}\partial\phi(y)\bigr)=1
\quad\text{for $P_1$-almost every $y$,}
\]
which yields \eqref{eq:ac-source-nd} by disintegration.
\end{proof}

More generally, the same relative-boundary argument yields
\textup{(ND)} whenever
$\pi_y\ll\mathcal{H}^{k(y)}\!\restriction_{L_y}$
for $P_1$-almost every $y$, where
$L_y=\operatorname{aff}\partial\phi(y)$ and $k(y)=\dim L_y$.
Lemma~\ref{lem:ac-source-sharpness-nd} derives this conditional
absolute continuity, together with $k(y)=d-m$, from the source
absolute continuity and the smoothness of $\phi|_M$.

For $\pi$-almost every pair $(x_1,x_0)$, the function
$f=\phi-\langle x_0,\cdot\rangle$ used in Section~3.2 satisfies
\[
\operatorname{par}\partial f(x_1)=N_{x_1}M,
\qquad 0\in\operatorname{ri}\partial f(x_1),
\]
because $\partial f(x_1)=\partial\phi(x_1)-x_0$.
Thus the sharpness component of \textup{(PS)} and
\textup{(ND)} follow almost surely under the stated assumptions.
The $C^2$ regularity of $\phi|_M$ is assumed here; it is not
inferred from the manifold hypothesis.
Continuity of the subdifferential relative to $M$ remains a
separate requirement for \textup{(PS)}.

\par
\begingroup
\edef\MBOTsavedEquation{\number\value{equation}}
\edef\MBOTsavedSecnumdepth{\number\value{secnumdepth}}
\setcounter{secnumdepth}{4}
\allowdisplaybreaks

\providecommand{\R}{}
\renewcommand{\R}{\mathbb{R}}
\providecommand{\N}{}
\renewcommand{\N}{\mathbb{N}}
\providecommand{\E}{}
\renewcommand{\E}{\mathbb{E}}
\providecommand{\Ptwo}{}
\renewcommand{\Ptwo}{\mathcal{P}_2}
\providecommand{\Pq}{}
\renewcommand{\Pq}{\mathcal{P}_q}
\providecommand{\dd}{}
\renewcommand{\dd}{\,\mathrm{d}}
\providecommand{\Id}{}
\renewcommand{\Id}{\operatorname{Id}}
\providecommand{\Law}{}
\renewcommand{\Law}{\operatorname{Law}}
\providecommand{\Lip}{}
\renewcommand{\Lip}{\operatorname{Lip}}
\providecommand{\osLip}{}
\renewcommand{\osLip}{\operatorname{osLip}}
\providecommand{\prox}{}
\renewcommand{\prox}{\operatorname{prox}}
\providecommand{\supp}{}
\renewcommand{\supp}{\operatorname{supp}}
\providecommand{\diam}{}
\renewcommand{\diam}{\operatorname{diam}}
\providecommand{\Pdim}{}
\renewcommand{\Pdim}{\operatorname{Pdim}}
\providecommand{\Rad}{}
\renewcommand{\Rad}{\mathfrak{R}}
\providecommand{\norm}{}
\renewcommand{\norm}[1]{\left\lVert #1\right\rVert}
\providecommand{\abs}{}
\renewcommand{\abs}[1]{\left\lvert #1\right\rvert}
\providecommand{\ip}{}
\renewcommand{\ip}[2]{\left\langle #1,#2\right\rangle}
\providecommand{\one}{}
\renewcommand{\one}{\mathbf{1}}
\providecommand{\barpi}{}
\renewcommand{\barpi}{\overline{\pi}_B}
\providecommand{\pistar}{}
\renewcommand{\pistar}{\pi^\star}
\providecommand{\mustar}{}
\renewcommand{\mustar}{\mu_t^\star}
\providecommand{\nuB}{}
\renewcommand{\nuB}{\nu_{B,t}}
\providecommand{\uStar}{}
\renewcommand{\uStar}{u^\star}
\providecommand{\vB}{}
\renewcommand{\vB}{v_B}
\providecommand{\DeltaB}{}
\renewcommand{\DeltaB}{\Delta_B}
\providecommand{\dB}{}
\renewcommand{\dB}{d_B}
\providecommand{\rhoB}{}
\renewcommand{\rhoB}{\rho_{B,I}}
\providecommand{\rhostar}{}
\renewcommand{\rhostar}{\rho_{\star,I}}

\section{Convergence of minibatch OT--CFM in Section \ref{sec:consistency}}
\label{sec:minibatch-OT}
\label{app:minibatch-OT}
\label{sec:app:minibatch-OT}
\label{app:consistencyNN}


This appendix gives the detailed convergence analysis for 
Section~\ref{sec:consistency}.  The corresponding detailed result to Theorem \ref{thm:convergence_rate} is
Corollary~\ref{mbot:cor:intrinsic-fast-zero}.


\subsection{Preliminaries}
\label{mbot:sec:purpose-scope}

For the theoretical analysis of the minibatch OT-CFM in this paper, we consider idealized neural network learning in combination with the \emph{expected batch OT plan} introduced in \citep{BoiteDelonNadjahi2026}, where independent (fresh) minibatches of size $B$ taken from the population are repeatedly provided $L$ times, and the network minimizes the squared error for the teaching vector field along the paths constructed by the minibatch OT.  

The analysis has the following structure:
\begin{equation*}
\boxed{
\begin{gathered}
\text{Fresh batches of size }B \;
\Longrightarrow \;
\text{Expected batch plan }\overline\pi_B \;
\Longrightarrow \;
\text{Expected batch OT vector field }v_B,
\\[1mm]
B\to\infty\;
\Longrightarrow\;
v_B\to u^\star \text{ population vector field},
\\[1mm]
L\to\infty\text{ fresh training batches} \;
\Longrightarrow \;
\widehat u_{B,L,P}\to v_B\text{ within the neural class},
\\[1mm]
\text{Spatial Lipschitz/One-sided Lipschitz control} \;
\Longrightarrow  \;
\text{Flow convergence.}
\end{gathered}}
\end{equation*}

Assuming that the distribution of data is supported by a low-dimensional manifold, 
we will prove the convergence of the vector field of the minibatch OT--CFM to that of the population OT--CFM with a convergence rate that depends on manifold dimensionality. 
The purpose is not to provide a tight statistical generalization bound for the OT--CFM, and we do not claim that the obtained bound is optimal.

\subsubsection{Optimal Transport Conditional Flow Matching}

Let $P_0$ and $P_1$ be in the class $\mathcal P_2(\R^d)$, the probability distributions on $\R^d$ with a finite second moment.  The \emph{optimal transport (OT) plan} $\pistar$ between $P_0$ and $P_1$ with quadratic cost is defined by 
\begin{equation}
    \pistar := \arg\min_{\pi\in \Pi(P_0,P_1)}\int \frac12 \|x_0-x_1\|^2 d\pi(x_0,x_1),
\end{equation}
where $\Pi(P_0,P_1):=\{\pi \in \mathcal P_2(\R^d\times \R^d)\mid \mathrm{Proj}_{0,\#}=P_0, \mathrm{Proj}_{1,\#}=P_1\}$.

In \emph{Conditional Flow Matching} \citep{Lipman2023,Albergo2023,Liu2023-ha}, the vector field is constructed by using random paths from $x_0\sim P_0$ to $x_1\sim P_1$:
\begin{equation}\label{mbot:eq:path}
    X_t := (1-t)X_0+t X_1.
\end{equation}
The joint distribution $\pi$ of $(X_0,X_1)$ is arbitrary for the general class of CFM, while this appendix considers the OT coupling: $(X_0,X_1)\sim \pistar$.  This appendix focuses on the affine paths \eqref{mbot:eq:path}, but an extension to an arbitrary parameterization $\alpha_t x_0 + \beta_t x_1$ is possible with the necessary changes. Along the path $X_t$, the velocity vector 
\[
V(X_0,X_1):= X_1 - X_0
\]
is placed at $X_t$.   This creates a random vector field, and a regression model, such as a neural network, learns a vector field using the least squares criterion: 
\[
\widehat{u}(x,t) = \arg\min_{u}\mathbb E_\pi[\|u(x,t)-V(X_0,X_1)\|^2 \mid X_t = x].
\]
The least square solution is then given by the mean of the random vector field: 
\begin{equation}
\widehat{u}(x,t) = \mathbb E_\pi[V(X_0,X_1)\mid X_t = x].
\end{equation}

\subsubsection{Minibatch OT-CFM and expected batch OT plan}
OT-based CFM methods have been popularly used to make the regression learning of the vector field easier \citep{Tong2024,Pooladian2023-kl}.  In practice, the population OT is not feasible to obtain.  Alternatively, a resampled batch of size $B$ is taken from the training data and the source distribution $P_0$, and some method of optimal transport, such as the Hungarian algorithm, is used to create a coupling $(X_{0,i}, X_{1,i}^{match})_{i=1}^B$. The endpoints of the path \eqref{mbot:eq:path} are then determined by this coupling.  The velocities along sampled points on the minibatch-OT paths are provided to neural networks as the minibatch training data for stochastic gradient descent (SGD). 

For ease of theoretical analysis, we consider an idealized procedure for this minibath-OT coupling and neural net learning, using the expected batch OT plan.  
Consider independent batches of size $B$
\[
 X_{0,1},\ldots,X_{0,B}\overset{\mathrm{iid}}{\sim}P_0,
 \qquad
 X_{1,1},\ldots,X_{1,B}\overset{\mathrm{iid}}{\sim}P_1,
\]
and define
\[
 \widehat P_{0,B}:=\frac1B\sum_{i=1}^B\delta_{X_{0,i}},
 \qquad
 \widehat P_{1,B}:=\frac1B\sum_{j=1}^B\delta_{X_{1,j}}.
\]
Fix a measurable solver that selects an empirical optimal coupling
\[
 \widehat\pi_B\in\operatorname*{argmin}_{\pi\in\Pi(\widehat P_{0,B},\widehat P_{1,B})}
 \int \frac12 \|x_0-x_1\|^2 \,\dd\pi.
\]
We assume to select an optimal permutation plan.  Then, the neural network $u(x,t)$ minimizes 
\[
 \mathbb E_{(X_0,X_1)\sim \widehat \pi_B, t} [\| u(X_t, t) - V(X_0,X_1)\|^2 | X_t], 
\]
where $t$ is taken independently. 

Assume that we generate the above random batches and couplings $L$ times independently from the populations $P_0$ and $P_1$, and the corresponding training data are provided for neural networks.  Then, the least squares learning ideally converges as $L\to\infty$ to the average of the vector field over the random batches:
\begin{equation}
 \mathbb E_B \mathbb E_{(X_0,X_1)\sim \widehat \pi_B, t} [ \| u(X_t, t) - V(X_0,X_1)\|^2 | X_t], 
\end{equation}
where $\mathbb E_B$ is the average over the random choice of the batches.  Following \citet{BoiteDelonNadjahi2026}, we introduce the average of the minibatch OT plan over the random batches.   
\begin{definition} 
The \emph{expected batch OT plan} $\barpi$ is the unique probability measure on
$\R^d\times\R^d$ satisfying
\begin{equation}
 \int g\,\dd\barpi
 =\E_B\left[\int g\,\dd\widehat\pi_B\right]
 \label{mbot:eq:expected-plan-def}
\end{equation}
for every bounded Borel function $g$, where the expectation is taken over the random choice of batches of size $B$ from the populations $P_0$ and $P_1$.
\end{definition}
Equivalently, sample two batches of size $B$, solve empirical OT, and then choose one
matched pair according to the empirical optimal plan.  The law of the selected
pair is equal to $\barpi$.  

The following properties of the expected batch OT plan
are established in \cite{BoiteDelonNadjahi2026}. 
\begin{proposition}
[\citet{BoiteDelonNadjahi2026}, Section~A.1.3 and Proposition~3.2]
\label{mbot:prop:marginals-cost}
For every $B\ge1$,
\[
 \barpi\in\Pi(P_0,P_1).
\]
Moreover,
\begin{equation}
 \int \frac12 \|x_0-x_1\|^2 \,\barpi(\dd x_0,\dd x_1)
 =\frac12\E_B\!\left[W_2^2(\widehat P_{0,B},\widehat P_{1,B})\right].
 \label{mbot:eq:cost-identity}
\end{equation}
\end{proposition}


Define the signed expected-batch cost bias $\Delta_B$ and the plan $W_2$-distance $d_B$ by
\begin{align}
 \DeltaB
 &:=\int \frac12 \|x_0-x_1\|^2 \,\dd\barpi-\int \frac12 \|x_0-x_1\|^2 \,\dd\pistar   \nonumber \\
 & \;=\frac12\left\{
 \E_B W_2^2(\widehat P_{0,B},\widehat P_{1,B})
 -W_2^2(P_0,P_1)\right\}  \;\ge0,
 \label{mbot:eq:DeltaB}\\
 \dB&:=W_2(\barpi,\pistar).
 \label{mbot:eq:dB}
\end{align}
The following proposition shows the convergence of the expected batch OT plan to the population OT as $B\to\infty$.

\begin{proposition}
\label{mbot:prop:qualitative-plan}
Assume that \(P_0,P_1\in\mathcal P_q(\mathbb R^d)\) for some
\(q>2\).
Then,  
\[
    \Delta_B \longrightarrow 0 \quad\text{and}\quad   
    d_B \longrightarrow 0
    \quad\text{as }B\to\infty.
\]
\end{proposition}

\begin{proof}
The convergence $\Delta_B\to0$ follows from
\citet[Proposition~3.3(1)]{BoiteDelonNadjahi2026}, while
$\barpi\rightharpoonup\pistar$ follows from
\citet[Proposition~3.4]{BoiteDelonNadjahi2026}.
Since $\barpi$ and $\pistar$ have the same fixed marginals,
their second moments on $\mathbb R^d\times\mathbb R^d$ coincide.
The standard characterization of Wasserstein convergence therefore
upgrades the weak convergence to $W_2(\pi_B,\pi^\star)\to0$ (see \citet[Theorem~6.9]{Villani2009OTOldNew}, for example).
\end{proof}


\subsubsection{Interpolation laws and Brenier maps}

For the affine schedule, define for $t\in [0,1]$
\begin{equation}
 F_t(x_0,x_1):=(1-t)x_0+t x_1,
 \label{mbot:eq:path-velocity}
\end{equation}
and let 
\begin{equation}
   \mu_t^\star:=(F_t)_\#\pistar ,\qquad 
 \nu_{B,t}:=(F_t)_\#\barpi.  
 \label{mbot:eq:path-laws}
\end{equation}
The vector field of the expected batch OT plan is 
\begin{equation}
 \vB(t,x):=
 \E_{(X_0,X_1)\sim\barpi}
 \left[X_1-X_0\mid F_t(X_0,X_1)=x\right],
 \label{mbot:eq:vB}
\end{equation}
defined $\nu_{B,t}$-almost everywhere.




Let $P_0,P_1\in\Ptwo(\R^d)$, and let 
\[
       K_0:=\supp P_0,
       \qquad
       K_1:=\supp P_1.
\]
We make the following four assumptions throughout this section. 
\medskip 

\begin{assumption}
The source distribution $P_0$ is absolutely
continuous with respect to Lebesgue measure, and $K_0$ and $K_1$ are compact. 
\label{mbot:eq:compact-endpoint-supports}
\end{assumption}

\begin{assumption}
\label{mbot:ass:graph-restricted-forward-regularity}
There exist an open set 
$\Omega_0$ that contains $K_0$ and a proper lower-semicontinuous convex function
$\Psi:\R^d\to(-\infty,+\infty]$ such that $S:=\nabla\Psi$ is continuous on $\Omega_0$ and 
\begin{equation}
       S_\#P_0=P_1,
       \qquad
       \pistar=(\Id,S)_\#P_0.
       \label{mbot:eq:population-plan-forward}
\end{equation}
\end{assumption}

\begin{assumption}
There is a constant $L_S>0$ such that the
following Bregman co-coercivity 
\begin{equation}
\label{mbot:eq:graph-restricted-bregman}
 D_\Psi(x,z)
: =\Psi(x)-\Psi(z)-\ip{S(z)}{x-z}
\ge \frac{1}{2L_S}\norm{S(x)-S(z)}^2
\end{equation}
holds for all $x,z\in K_0$.
 \label{mbot:ass:graph-restricted-bregman}
\end{assumption}

\medskip 
\begin{assumption}
Fix $0 < c < 1$ and 
\begin{equation}
       I=[0,c]. 
       \label{mbot:eq:main-time-interval}
\end{equation}
All vector field, regression, and flow bounds concern early stopping at time $t=c$. 
\end{assumption}

\medskip

\begin{remark}
The assumption \ref{mbot:eq:compact-endpoint-supports} does not cover the Gaussian distribution for the source. However,  unlike diffusion models, a Gaussian distribution for the source is not necessary for flow matching, and a truncated Gaussian is often sufficient in practice.  This simplifies theoretical arguments significantly.
\end{remark}

\begin{remark}
The standard theory guaranties the existence of Brenier map defined on $\R^d$, if $P_0$ is absolutely continuous with respect to the Lebesgue measure. In the assumption \ref{mbot:ass:graph-restricted-forward-regularity}, we assume the continuous Brenier map $S$ exists only on $\Omega_0$ that contains $K_0$, whereas the Brenier potential $\Psi$, which may take $\infty$, exists on $\mathbb R^d$. 
\end{remark}


\begin{remark}[A checkable sufficient condition of \ref{mbot:ass:graph-restricted-bregman}]
\label{mbot:rem:restricted-regularity-sufficient}
A sufficient condition for
\eqref{mbot:eq:graph-restricted-bregman} is that $\Psi$ is the continuously differentiable functions on a convex open neighborhood $\Omega$ of compact $K_0$ such that its derivative $S=\nabla \Psi$ is Lipschitz with 
\[
       \Lip(S;\Omega)\le L_\Omega.
\]
In fact, there is $\rho>0$ such that $K_0+B(0,\rho)\subset \Omega$, where $B(0,\rho)$ is the closed ball with radius $\rho$. Set $L_S:= L_\Omega \max\{1, \text{Diam}(K_0)/\rho\}$.  For any fixed $x,z\in K_0$, define 
\[
h_z(u):= \Psi(u) -\Psi(z) - \langle S(z),u-z\rangle 
\]
for $u\in \Omega$.  Then, $h_z$ is convex and differentiable on $\Omega$, $h_z(u)\ge 0$ by the convexity of $\Psi$, and $\nabla h_z(u) = S(u)-S(z)$.  If we take $y:=x - (S(x)-S(z))/L_S$, then $y\in \Omega$, because $\|y-x\|\leq \rho$ can be easily confirmed.  Since the line segment $[x,y]$ is contained in $\Omega$, we have 
\[
0\leq h_z(y) \leq h_z(x) + \langle \nabla h_z(x),y-x\rangle + \frac{L_\Omega}{2}\|y-x\|^2. 
\]
Simple calculation yields 
\[
0\leq \Psi(x)-\Psi(z) - \langle S(z),x-z\rangle - \frac{1}{2L_S}\|S(x)-S(z)\|^2, 
\]
which proves the claim. 
\end{remark}
\smallskip

\begin{remark}
We state the Bregman co-coercivity on the source support $K_0$ only because it is the property used below and is required only for pairs of source points.  
Note also that, if we impose the Lipschitz property of $\Psi$ on the whole $\R^d$, the range of the map $S=\nabla\Psi$ must be convex, which excludes general low dimensional manifold structures for $K_1$. We, therefore, consider a restricted compact domain $K_0$, and the co-coercivity or Lipschitz property on the restricted domain. 
\end{remark}
\smallskip

\begin{remark}[Scope of the assumption \ref{mbot:ass:graph-restricted-forward-regularity}]
\label{mbot:rem:singular-forward-map-excluded}
Assumption~\ref{mbot:ass:graph-restricted-forward-regularity} may not hold in general. 
For example, for the rotationally symmetric quadratic OT from the uniform distribution on $K_0:=\{ x\in \mathbb R^d \mid 1/2\le \|x\| \leq 1\}$ to the uniform distribution on the unit circle $\mathbb{S}^1$, the forward Brenier map is given by 
\[
       S(x)=\frac{x}{\|x\|}.
\]
The potential function $\Psi$ must take the form $\Psi(x) = \|x\|$ to satisfy $\nabla\Psi(x) = S(x)$.  So, if we take $\Omega_0 :=\{ x\in \R^d\mid 1/4 < \|x\| < 2\}$, Assumption~\ref{mbot:ass:graph-restricted-forward-regularity} is satisfied.  However, if we consider the uniform distribution on the unit disk for $P_0$, the Bernier map $S$ cannot be continuous at $0$. 
\end{remark}

\smallskip

\begin{proposition}
Under the assumption \ref{mbot:ass:graph-restricted-bregman}, 
\begin{equation}
       \norm{S(x)-S(z)}\le L_S\norm{x-z},
       \qquad \text{for any } x,z\in K_0.
\end{equation}  
\end{proposition}
\begin{proof}
    Adding \eqref{mbot:eq:graph-restricted-bregman} with $(x,z)$ and $(z,x)$ yields
\[
\ip{S(x)-S(z)}{x-z}\ge \frac{1}{L_S}\|S(x)-S(z)\|^2.
\]
Combining Cauchy-Schwartz $\ip{S(x)-S(z)}{x-z}\le \|S(x)-S(z)\| \|x-z\|$ shows the Lipschitz property. 
\end{proof}


\subsection{OT--CFM vector field and Expected-batch OT}
\label{mbot:sec:fields}

\subsubsection{Lipschitz regularity of the population OT vector field}

Consider the vector field $\uStar$ given by the population OT plan $\pistar$.  Introduce $\phi:=\Psi^*$, 
which serves as the extended Brenier potential such that $x_0\in \partial \phi(x_1)$ whenever $(x_0,x_1)\sim \pistar$.  
As shown in Sec.~\ref{sec:prox_OTCFM}, we have the expression of the population OT vector field with the proximal operator of $\phi$:
\begin{equation}
   \uStar(t,x):=\frac{Q_t(x)-x}{1-t},
       \label{mbot:eq:u-star-proximal} 
\end{equation}
where  $Q_t(x)=\prox_{\lambda_t\phi}\!\left(\frac{x}{t}\right)$ with $\lambda_t = (1-t)/t$. 

For $x_0\in K_0$, define the exact interpolation as a function of $x_0$
\begin{equation*}
       Z_t(x_0):=(1-t)x_0+tS(x_0).
\end{equation*}
Since $x_0\in\partial\phi(S(x_0))$, the proximal optimality condition gives
\begin{equation}
       Q_t(Z_t(x_0))=S(x_0),
       \qquad
       \uStar(t,Z_t(x_0))=S(x_0)-x_0.
       \label{mbot:eq:exact-path-velocity}
\end{equation}


\medskip 
We show Lipschitz regularity of the proximal operator and vector field used in the sequel. 
For a vector field $f(t,x)$, define the \emph{one-sided Lipschitz (OSL) coefficient} by 
\begin{equation}
 \rho_f(t):=\sup_{x\ne y}
 \frac{\ip{f(t,x)-f(t,y)}{x-y}}{\norm{x-y}^2}.
 \label{mbot:eq:OSL-definition}
\end{equation}



\begin{proposition} 
\label{mbot:prop:exact-field-zero-regularity}
For every $t\in(0,1)$,
\begin{align}
 \Lip(Q_t)&\le\frac1t,
 \label{mbot:eq:Qt-prox-Lip}\\
 \Lip_x\uStar(t,\cdot)
 &\le 
 \max\left\{\frac1t,\frac1{1-t}\right\} =:\ell_t,
 \label{mbot:eq:u-prox-Lip}\\
 \rho_{\uStar}(t)&\le\frac1t.
 \label{mbot:eq:exact-osl-generic}
\end{align}
\end{proposition}

\begin{proof}
Write: 
The firm nonexpansiveness of the proximal operator (\ref{prop:prox-nonexpansiveness}) yields
\begin{equation}
\label{eq:Qt_nonexpansive}
       \| Q_t(x)-Q_t(y) \|^2 \leq \frac{1}{t} \langle Q_t(x)-Q_t(y), x-y\rangle.
\end{equation}
Combined with Cauchy-Schwartz, this proves \eqref{mbot:eq:Qt-prox-Lip}.  It follows from \eqref{mbot:eq:u-star-proximal} that 
\[
\begin{aligned}
     (1-t)^2\norm{\uStar(t,x)-\uStar(t,y)}^2
 & =\norm{Q_t(x)-Q_t(y)-(x-y)}^2  \\
 & =\norm{x-y}^2+\norm{Q_t(x)-Q_t(y)}^2-2\langle Q_t(x)-Q_t(y), x-y\rangle \\
 &  \le \norm{x-y}^2+(1-2t)\norm{Q_t(x)-Q_t(y)}^2,
\end{aligned}
\]
where the last inequality uses \eqref{eq:Qt_nonexpansive}. 
If $t\ge1/2$, the RHS is at most $\norm{x-y}^2$.  If $t\le1/2$, use
$ \|Q_t(x)-Q_t(y)\| \le t^{-1}\norm{x-y}$ to obtain
the $\text{RHS} \le \frac{(1-t)^2}{t^2}\norm{x-y}^2$. 
This proves \eqref{mbot:eq:u-prox-Lip}.  Further, using \eqref{mbot:eq:Qt-prox-Lip}, 
\[
 \ip{\uStar(t,x)-\uStar(t,y)}{x-y}
 =\frac{\ip{Q_t(x)-Q_t(y)}{x-y}-\norm{x-y}^2}{1-t}
 \le\frac{t^{-1}-1}{1-t}\norm {x-y}^2
 =\frac1t\norm {x-y}^2,
\]
 which is \eqref{mbot:eq:exact-osl-generic}. 
\end{proof}

\subsubsection{Expected batch regression risk}

For a measurable field $f:I\times\R^d\to\R^d$, define the norms by 
\begin{align}
 \norm{f}_{B,I}^2
 &:=\int_I\int\norm{f(t,x)}^2\,\nu_{B,t}(\dd x)\,\dd t,
 \label{mbot:eq:normB-fields}\\
 \norm{f}_{\star,I}^2
 &:=\int_I\int\norm{f(t,x)}^2\,\mu_t^\star(\dd x)\,\dd t.
 \label{mbot:eq:normstar-fields}
\end{align}
The \emph{expected-batch regression risk} is defined by 
\begin{equation}
 \mathcal R_B(f):=
 \int_I\E_{\barpi}
 \left[\norm{f(t,F_t(X_0,X_1))-V(X_0,X_1)}^2\right]\dd t.
 \label{mbot:eq:population-risk-fields}
\end{equation}

\begin{lemma}
\label{mbot:lem:bayes-risk-fields}
Any measurable vector field $f$ with finite risk satisfies  
\begin{equation}
 \mathcal R_B(f)-\mathcal R_B(\vB)
 =\norm{f-\vB}_{B,I}^2.
 \label{mbot:eq:bayes-risk-identity-fields}
\end{equation}
\end{lemma}

\begin{proof}
For each $t$, write $X_t=F_t(X_0,X_1)$ and $Y=V(X_0,X_1)=X_1-X_0$.  Since  
$\vB(t,X)=\E_{\barpi}[Y\mid X_t = X]$, the orthogonal projection property of the conditional expectation gives
\[
 \E_{\barpi}\norm{f(t,X)-Y}^2
 =\E_{\barpi}\norm{f(t,X)-\vB(t,X)}^2
 +\E_{\barpi}\norm{\vB(t,X)-Y}^2.
\]
Integration over $t\in I$ proves the claim.
\end{proof}

\subsection{Bias of the expected-batch OT plan}
\label{mbot:sec:direct-bias}

The difference of expected batch OT plan from the population OT causes a bias of the expected batch OT vector field. We quantify this bias below. 

\subsubsection{Fenchel slack and cost gap}

Define the Fenchel slack associated with $\phi=\Psi^\ast$ by
\begin{equation}
 G_\phi(x_1,x_0)
 :=\phi(x_1)+\Psi(x_0)-\ip{x_0}{x_1}\ge0.
 \label{mbot:eq:fenchel-slack}
\end{equation}
The equality $G_\phi(x_1,x_0)=0$ holds whenever
$x_0\in\partial\phi(x_1)$.
The following fact is a generalization of the result in \citet[Appendix A.5.3]{BoiteDelonNadjahi2026}, where they express the excess cost by the expected dual slack for the semidiscrete OT.   
\begin{lemma}
\label{mbot:lem:slack-cost}
For every $\pi\in\Pi(P_0,P_1)$,
\begin{equation}
 \int G_\phi\,\dd\pi
 =\int\frac12 \|x-y\|^2 \,\dd\pi-\int\frac12 \|x-y\|^2 \,\dd\pistar.
 \label{mbot:eq:slack-cost-identity}
\end{equation}
In particular,
\begin{equation}
       \int G_\phi\,\dd\barpi=\DeltaB.
       \label{mbot:eq:slack-Delta}
\end{equation}
\end{lemma}
\begin{proof}
Because $\pi$ and $\pistar$ have the same marginals,
\[
 \int\{\phi(x_1)+\Psi(x_0)\}\,\dd\pi
 =\int\{\phi(x_1)+\Psi(x_0)\}\,\dd\pistar.
\]
The latter equals $\int\ip{x_0}{x_1}\,\dd\pistar$ by Fenchel equality on the population OT graph.  Hence
\[
 \int G_\phi\,\dd\pi = \int \{\phi(x_1)+\Psi(x_0)\}d\pistar - \int\ip{x_0}{x_1}\,\dd\pi 
 =\int\ip{x_0}{x_1}\,\dd\pistar
 -\int\ip{x_0}{x_1}\,\dd\pi.
\]
The marginal terms in the cost $\frac12\norm{x_0}^2+\frac12\norm{x_1}^2-\ip{x_0}{x_1}$ cancel between $\pi$ and $\pistar$.  We thus obtain \eqref{mbot:eq:slack-cost-identity}. \eqref{mbot:eq:slack-Delta} is given by the definition of $\Delta_B$. 
\end{proof}

\subsubsection{Bias of vector field}

The following argument compares an expected-batch interpolation with the
population OT interpolation issued from the same source point.

\begin{lemma}
\label{mbot:lem:prox-reconstruction-I}
Let $x_0\in K_0$ and $x_1\in K_1$ be arbitrary points, and $t\in I$.  Set
\[
       x_t^B:=F_t(x_0,x_1),
       \qquad
       x_t^\star:=F_t(x_0,S(x_0))=Z_t(x_0).
\]
Then
\begin{equation}
 \norm{V(x_0,x_1)-\uStar(t,x_t^B)}
 \le \bigl(1+t\ell_t\bigr)
       \norm{x_1-S(x_0)},
 \label{mbot:eq:teacher-field-common-source}
\end{equation}
where $\ell_t$ is defined in \eqref{mbot:eq:u-prox-Lip}.
\end{lemma}

\begin{proof}
Note first that $(x_0,x_1)$ is an arbitrary pair, while $(x_0,S(x_0))$ is an OT pair. 
By the exact-path identity \eqref{mbot:eq:exact-path-velocity}, $    \uStar(t,x_t^\star)=S(x_0)-x_0$ and $V(x_0,x_1)=x_1-x_0$.
Additionally, we have 
\[
       x_t^B-x_t^\star= (1-t)x_0 + t x_1 - \{(1-t)x_0 + t S(x_0)\} = t\{x_1-S(x_0)\}.
\]
Therefore, using Proposition~\ref{mbot:prop:exact-field-zero-regularity},
\begin{align*}
\norm{V(x_0,x_1)-\uStar(t,x_t^B)}
 & \le
 \norm{x_1-S(x_0)}
 +\norm{\uStar(t,x_t^\star)-\uStar(t,x_t^B)}
 \\
 &\le \norm{x_1-S(x_0)}
 +\ell_t \norm{x_t^\star - x_t^B}  \le 
 \bigl(1+t\ell_t\bigr)\norm{x_1-S(x_0)}.
\end{align*}
\end{proof}
\begin{remark}
The factor $t$ in the \eqref{mbot:eq:teacher-field-common-source} cancels the possible $t^{-1}$ growth of the Lipschitz bound $\ell_t$ near the initial time.
\end{remark}

Define
\begin{equation}
 \kappa_I^{\mathrm{cs}}
 :=\int_0^c\bigl(1+t\ell_t\bigr)^2\dd t
 =
 \begin{cases}
       4c,&0<c\le\frac12,\\[1mm]
       (1-c)^{-1},&\frac12<c<1.
 \end{cases}
 \label{mbot:eq:kappa-common-source}
\end{equation}
Also set
\begin{equation}
 \varepsilon_B^2
 :=\int\norm{x_1-S(x_0)}^2\,
       \barpi(\dd x_0,\dd x_1).
 \label{mbot:eq:epsilonB}
\end{equation}

The following theorem expresses an upper bound for the bias of the exected-batch OT vector field in terms of $L^2(\nu_{B,t}dt)$
\begin{theorem}
\label{mbot:prop:batch-bias-zero}
The expected-batch OT vector field satisfies
\begin{equation}
       \norm{v_B-u^\star}_{B,I}^2
       \le \kappa_I^{\mathrm{cs}}\varepsilon_B^2.
       \label{mbot:eq:vB-u-epsilon}
\end{equation}
\end{theorem}

\begin{proof}
Fix $t\in I$ and let $(X_0,X_1)\sim\overline\pi_B$, $X_t:=F_t(X_0,X_1)$, and $V:=X_1-X_0$.
By definition $v_B(t,X_t)=\mathbb E[V\mid X_t]$, 
the orthogonality of conditional expectation gives 
\begin{equation*}
  \mathbb E\|V-u^\star(t,X_t)\|^2
 =
 \mathbb E\|V-v_B(t,X_t)\|^2
+
 \mathbb E\|v_B(t,X_t)-u^\star(t,X_t)\|^2. 
\end{equation*}
Theorefore, 
\[
 \int
 \|v_B(t,x)-u^\star(t,x)\|^2
 \,\nu_{B,t}(dx) \le
 \int
 \|V(x_0,x_1)-u^\star(t,F_t(x_0,x_1))\|^2
 \,\overline\pi_B(dx_0,dx_1).
\]
Applying Lemma~\ref{mbot:lem:prox-reconstruction-I}, we obtain
\[
\begin{aligned}
\int
 \|v_B(t,x)-u^\star(t,x)\|^2
 \,\nu_{B,t}(dx)
 & \le
 (1+t\ell_t)^2 \int \| x_1 - S(x_0)\|^2 \,d\overline\pi_B
\end{aligned}
\]
Integrating over $t$ proves the claim.
\end{proof}

\subsubsection{Graph-restricted detachment and plan discrepancy}

\begin{proposition}
\label{mbot:prop:automatic-detachment}
Under Assumption~\ref{mbot:ass:graph-restricted-bregman}, for every
$x_0\in K_0$ and $x_1\in K_1$,
\begin{equation}
       G_\phi(x_1,x_0)
       \ge\frac1{2L_S}\norm{x_1-S(x_0)}^2.
       \label{mbot:eq:automatic-detachment}
\end{equation}
Consequently, we have the following bounds.
\begin{align}
 \varepsilon_B^2&\le2L_S\Delta_B,
 \label{mbot:eq:epsilon-from-Delta}\\
 d_B&\le\varepsilon_B\le\sqrt{2L_S\Delta_B},
 \label{mbot:eq:dB-from-Delta-regular}\\
 \norm{v_B-u^\star}_{B,I}^2
 &\le2L_S\kappa_I^{\mathrm{cs}}\Delta_B.
 \label{mbot:eq:vB-u-I}
\end{align}
\end{proposition}

\begin{proof}
Fix $x_0\in K_0$ and $x_1\in K_1$.  Since
$K_1=S(K_0)$, choose $z_0\in K_0$ such that $x_1=S(z_0)$.
Fenchel equality at $(z_0,S(z_0))$ gives
$\phi(x_1)=\ip{z_0}{x_1}-\Psi(z_0)$.
Therefore
\begin{align*}
 G_\phi(x_1,x_0)
 = \phi(x_1) + \Psi(x_0) -\ip{x_1}{x_0}  
=\Psi(x_0)-\Psi(z_0)-\ip{S(z_0)}{x_0-z_0}
=D_\Psi(x_0,z_0),
\end{align*}
which is the Bregman divergence from $z_0$ to $x_0$.  Thus, the Assumption \ref{mbot:ass:graph-restricted-bregman} yields
\[
       G_\phi(x_1,x_0)
       \ge\frac1{2L_S}\norm{S(x_0)-S(z_0)}^2
       =\frac1{2L_S}\norm{S(x_0)-x_1}^2,
\]
which proves \eqref{mbot:eq:automatic-detachment}.  Integrating with respect to
$\barpi$ and using Lemma~\ref{mbot:lem:slack-cost} gives
\eqref{mbot:eq:epsilon-from-Delta}.

While the proof of $d_B\leq \varepsilon_B$ is basically given in \cite[Appendix A.5.3]{BoiteDelonNadjahi2026}, we show the proof here for completeness. To compare $\barpi$ and $\pistar$, draw $(X_0,X_1)\sim\barpi$.  Since the first marginal
of $\barpi$ is $P_0$, we see that $(X_0,S(X_0))\sim(\Id,S)_\#P_0=\pistar$.
Hence, the joint law of
$((X_0,X_1),(X_0,S(X_0)))$ is an admissible coupling of $\barpi$ and
$\pistar$, with squared transportation cost
$\E\norm{X_1-S(X_0)}^2=\varepsilon_B^2$.  This proves
$d_B=W_2(\pistar,\barpi)\le\varepsilon_B$.  Finally, combine
Theorem~\ref{mbot:prop:batch-bias-zero} with
\eqref{mbot:eq:epsilon-from-Delta} to obtain \eqref{mbot:eq:vB-u-I}.
\end{proof}


\subsection{Neural regression from fresh batches}
\label{mbot:sec:neural-regression}
In practical methods of OT-CFM, a continuous vector field is learned by a regression model, such as neural networks, so that we can generate a new sample with an arbitrary initial point.  This section derives the neural network contribution from standard approximation and regression theory. 

For ease of exposing the theory of neural approximation and
finite-sample rate, we now restrict ourselves to regression on a compact set in this section.  Recall that $K_i = \operatorname{supp}P_i$ ($i=0,1$).  Let us introduce the space of possible paths: 
\begin{equation}
 \mathcal X_I:=
 \left\{\bigl(t,(1-t)x_0+tx_1\bigr):
 t\in I,\ x_0\in K_0,\ x_1\in K_1\right\}
 \subset I\times\R^d.
 \label{mbot:eq:compact-regression-domain}
\end{equation}

\begin{assumption}
\label{mbot:ass:compact-neural-regression}
\[
\mathcal X_I \subset [0,1]^{d+1}.
\]
\end{assumption}

Note that, under the assumption of compact supports, $\mathcal X_I \subset [0,1]^{d+1}$ is satisfied by applying an affine rescaling
of the space.  This rescaling is used only for easiness of stating the bounds below. 
An unbounded source can be treated by a separate truncation or unbounded-envelope regression analysis, but they are left out of the scope of this appendix. 

We also use the notation 
\begin{equation}
       R_V:=\sup_{x_0\in K_0,\,x_1\in K_1}\norm{x_1-x_0}<\infty.
       \label{mbot:eq:RV-compact}
\end{equation}

\subsubsection{Neural Network Learning}
\label{mbot:subsec:online-objective}

For each iteration $m=1,\ldots,L$, sample two fresh endpoint batches of size $B$ independently over $m$, solve
the empirical quadratic OT problem, and draw independent time variables
$T_{m,1},\ldots,T_{m,B}\sim\operatorname{Unif}(I)$.  After reindexing
according to the selected OT coupling, write
\[
  U_m:=\left\{
  \bigl(X_{0,m,i},X^{\mathrm{match}}_{1,m,i},T_{m,i}\bigr)
  \right\}_{i=1}^{B}.
\]
The batch objects $U_1,\ldots,U_L$ are i.i.d., while the matched pairs inside a batch are generally dependent.

For a measurable field $f:I\times\R^d\to\R^d$, define
\begin{align}
  \mathcal L_f(U_m)
  &:=\frac{c}{B}\sum_{i=1}^{B}
  \left\|
  f\!\left(T_{m,i},
  F_{T_{m,i}}(X_{0,m,i},X^{\mathrm{match}}_{1,m,i})\right)
  -V(X_{0,m,i},X^{\mathrm{match}}_{1,m,i})
  \right\|^2,
  \label{mbot:eq:batch-loss-zero}\\
  \widehat R_{B,L}(f)
  &:=\frac{1}{L}\sum_{m=1}^{L}\mathcal L_f(U_m),
  \label{mbot:eq:empirical-risk-zero}
\end{align}
where $c$ is the early stopping time.  Then $\E[\mathcal L_f(U_m)]=\mathcal R_B(f)$.  

Throughout the remainder of Sec.~\ref{mbot:sec:neural-regression}, the regression model is a sparse deep
ReLU network.  Let $\rho(a):=\max\{a,0\}$ be the ReLU activation function.  A network
$\Phi:\mathbb R^{q_0}\to\mathbb R^{q_J}$ with depth $J$ and widths
$\mathbf q=(q_0,\ldots,q_J)$ is written as
\begin{equation}
\Phi
=
F_J\circ\rho\circ F_{J-1}\circ\cdots\circ\rho\circ F_1,
\qquad
F_j(z)=W_jz+b_j,
\label{mbot:eq:relu-network-definition-zero}
\end{equation}
where the last affine map is not followed by an activation.  We use
\begin{gather}
\mathcal L(\Phi):=J,\quad 
\mathcal W(\Phi):=\max_{0\le j\le J}q_j,\quad 
\mathcal M(\Phi):=\sum_{j=1}^{J}
 \bigl(\|W_j\|_0+\|b_j\|_0\bigr), \nonumber \\
\mathcal B(\Phi):=\max_j\{\|W_j\|_{\max},\|b_j\|_\infty\}, 
\label{mbot:eq:relu-complexity-zero}
\end{gather}
for the depth, maximal width, number of nonzero parameters, and maximal
parameter magnitude, respectively.  In particular, $P$ denotes sparsity
$\mathcal M(\Phi)$, and is used as a complexity budget to decide the other parameters. 
Fix positive constants $C_D, C_A, \kappa_A, R_f<\infty$, and 
let $\mathcal F_P$ be the class of
vector-valued ReLU networks $\Phi:I\times\mathbb R^d\to\mathbb R^d$ satisfying
\begin{equation}
\mathcal L(\Phi)\le \left\lceil C_D\log^2(P+1)\right\rceil,
\;\;
\mathcal W(\Phi)\le P+d+1,
\;\;
\mathcal M(\Phi)\le P,
\;\;
\mathcal B(\Phi)\le A_P,
\;\;
\sup_{z\in\mathcal X_I}\|\Phi(z)\|\le R_f,
\label{mbot:eq:relu-sieve-definition-zero}
\end{equation}
where $A_P:= C_A P^{\kappa_A}$. 
The width condition is inessential: after deleting inactive nodes, a network with at most $P$ nonzero parameters can be represented with maximal width at most $P+d+1$.  The depth budget is broad enough for the structured construction in Section~\ref{mbot:sec:intrinsic-approximation}.
A fixed output envelope $R_f$ can, if desired, be
enforced by coordinatewise ReLU clipping, which has constant additional
depth and sparsity.

For training, we assume that an approximate empirical minimizer
$\widehat u_{B,L,P}\in\mathcal F_P$ is obtained:
\begin{equation}
 \widehat R_{B,L}(\widehat u_{B,L,P})
 \le \inf_{f\in\mathcal F_P}\widehat R_{B,L}(f)
      +\varepsilon_{\mathrm{opt}},
 \label{mbot:eq:approx-erm-zero}
\end{equation}
where $\varepsilon_{\mathrm{opt}}$ is the optimization error.

For $\epsilon>0$, the \emph{covering number}
$N(\epsilon,\mathcal F_P,\|\cdot\|_{\infty,\mathcal X_I})$
is the smallest integer $K\ge1$ for which there exist
$f_1,\ldots,f_K\in\mathcal F_P$ satisfying
\[
    \sup_{f\in\mathcal F_P}
    \min_{1\le j\le K}
    \|f-f_j\|_{\infty,\mathcal X_I}
    \le \epsilon.
\]

\begin{lemma}
\label{mbot:lem:relu-sparse-entropy-zero}
There is a constant $C_{\mathrm{ent}}$, depending only on $C_D$ and the
fixed input and output dimensions, such that, for $0<\epsilon\le1$,
\begin{equation}
\begin{aligned}
 \log N\!\left(\epsilon,\mathcal F_P,
   \|\cdot\|_{\infty,\mathcal X_I}\right)
 &\le C P\log^2(P+1)
 \log\!\left(\frac{C A_P P\log(P+1)}{\epsilon}\right).
\end{aligned}
\label{mbot:eq:relu-covering-zero}
\end{equation}
The last logarithm can be bounded by $C\log(CP/\epsilon)$.
\end{lemma}

\begin{proof}
For scalar output, Proposition~5.1 of \citet{OhnKim2019}, applied to the
$1$-Lipschitz ReLU activation, gives
\[
 \log N(\epsilon,\mathcal F,\|\cdot\|_\infty)
 \le
 2\mathcal L(\Phi)(\mathcal M(\Phi)+1)
 \log\!\left(
 \epsilon^{-1}\mathcal L(\Phi)
 (\mathcal W(\Phi)+1)(\mathcal B(\Phi)\vee1)
 \right).
\]
The vector-output extension follows by covering the $d$ coordinates and
changing only a constant depending on $d$.  Substitution of
\eqref{mbot:eq:relu-sieve-definition-zero}
gives the inequality in \eqref{mbot:eq:relu-covering-zero}
Finally, $A_P\le C_AP^{\kappa_A}$ implies the last assertion.
\end{proof}

\subsubsection{A finite-sample oracle inequality for random batches}
\label{mbot:subsec:batch-oracle}

Define
\begin{equation}
       K_I^2:=|I|(R_f+R_V)^2=c(R_f+R_V)^2,
       \label{mbot:eq:KI}
\end{equation}
where $R_f$ and $R_V$ are defined in \eqref{mbot:eq:relu-sieve-definition-zero} and \eqref{mbot:eq:RV-compact}, respectively.  

\begin{theorem}
\label{mbot:thm:batch-relu-oracle-zero}
Suppose Assumption~\ref{mbot:ass:compact-neural-regression} holds and let $L\ge2$.  There exist
constants $C,C_0>0$, independent of $B,L,P$, such that, for every
$\eta\in(0,1)$, with probability at least $1-\eta$ over the random batches, the following inequality holds:
\begin{equation}
 \|\widehat u_{B,L,P}-v_B\|_{B,I}^2
 \le C\inf_{f\in\mathcal F_P}\|f-v_B\|_{B,I}^2
 +CK_I^2\frac{P\log^2(P+1)\log(C_0A_PPL)+\log(2/\eta)}{L}
 +C\varepsilon_{\mathrm{opt}}.
 \label{mbot:eq:batch-relu-oracle-zero}
\end{equation}
Consequently, with probability at least $1-\eta$, 
\begin{align}
 \|\widehat u_{B,L,P}-u^\star\|_{B,I}
 \le{}&C\, a_P(B)+C\sqrt{\kappa_{S,I}\Delta_B}
 \notag\\
 &+CK_I\sqrt{\frac{P\log^2(P+1)\log(C_0A_PPL)+\log(2/\eta)}{L}}
 +C\sqrt{\varepsilon_{\mathrm{opt}}}
 \label{mbot:eq:batch-relu-field-zero}
\end{align}
holds, where
\begin{equation*}
       a_P(B):=\inf_{f\in\mathcal F_P}\|f-u^\star\|_{B,I}.
\end{equation*}
\end{theorem}

\begin{proof}
The following proof is a specialization of the standard bounded empirical-risk minimization under a Bernstein-type variance condition. We include it for completeness. 

First, the loss map satisfies
\begin{equation}
 \sup_Z|\mathcal L_f(Z)-\mathcal L_g(Z)|
 \le 2c(R_f+R_V)\|f-g\|_{\infty,\mathcal X_I},
 \label{mbot:eq:loss_lip_estimate}
\end{equation}
so that Lemma~\ref{mbot:lem:relu-sparse-entropy-zero} controls the entropy of the excess loss class $\{ h_f:=\mathcal L_f-\mathcal L_{v_B} \mid f\in \mathcal F_P\}$.   
We first show that this excess-loss class satisfies the Bernstein condition. Let $\mathbb P$ denote expectation with respect to one batch object $Z$, and let
\[
  \mathbb P_L h:=\frac1L\sum_{m=1}^L h(U_m).
\]
From 
\[
| \mathcal L_f(Z)-\mathcal L_{v_B}(Z) |\leq \frac{c}{B}\sum_{i=1}^B \| f(T_i,F_i)-v_B(T_i,F_i)\| \left( \|f(T_i,F_i)\| + \|v_B(T_i,F_i)\| + 2\|V_i\|\right),
\]
we have 
\[
|h_f(Z)| \leq 2c (R_f+R_V) \frac{1}{B}\sum_{i=1}^B \| f(T_i,F_i)-v_B(T_i,F_i)\|. 
\]
Jensen's inequality and Lemma \ref{mbot:lem:bayes-risk-fields} give
\begin{align}
\mathbb P h_f^2 & \le 4c^2 (R_f + R_V)^2  \mathbb P \left(
\frac1B \sum_{i=1}^B \Vert f(T_i,F_i)-v_B(T_i,F_i)\Vert \right)^2  \nonumber \\
& \le CK_I^2  \mathbb P  \left(
\frac1B \sum_{i=1}^B \Vert f(T_i,F_i)-v_B(T_i,F_i)\Vert^2  \right) \nonumber \\
& = CK_I^2\|f-v_B\|_{B,I}^2 \nonumber \\
& = CK_I^2 \mathbb P h_f. 
 \label{mbot:eq:Bernstein_condition}  
\end{align}
Thus, the excess-loss class satisfies the $(1,CK_I^2)$-Bernstein condition in the sense of \citet[][Definition 2.6]{BartlettMendelson2006}.

Consider a finite $K_I^2/L$-net of the excess-loss class
in the uniform norm.  The preceding loss-Lipschitz estimate \eqref{mbot:eq:loss_lip_estimate} 
and Lemma~\ref{mbot:lem:relu-sparse-entropy-zero} imply
that its logarithmic cardinality is bounded by
\[
  C P\log^2(P+1)\log(C_0A_PPL).
\]
Applying the scalar Bernstein inequality
\citep[Theorem~1.1,][]{BartlettMendelson2006} with \eqref{mbot:eq:Bernstein_condition} 
to each element of the finite net  and taking a union bound yields
\[
\begin{aligned}
 |(\mathbb P_L-\mathbb P)h_f|
 \le{}&
 C\sqrt{
   \frac{
     K_I^2 (\mathbb P h_f)
     \{P\log^2(P+1)\log(C_0A_PPL)+\log(2/\eta)\}
   }{L}
 }
 \\
 &+
 C K_I^2
 \frac{
   P\log^2(P+1)\log(C_0A_PPL)+\log(2/\eta)
 }{L}
\end{aligned}
\]
with probability larger than $1-\eta$.
Using $\sqrt{ab}\le a/8+4b$, we further obtain 
\[
 |(\mathbb P_L-\mathbb P)h_f|
 \le{}  \frac{1}{4} \mathbb P h_f +
 C K_I^2
 \frac{
   P\log^2(P+1)\log(C_0A_PPL)+\log(2/\eta)
 }{L}.
\]
Taking a union bound over the
net, and then extending the estimate from the net to the
whole class by the loss-Lipschitz bound, we obtain, with
probability at least $1-\eta$,
\begin{equation}
   \sup_{f\in\mathcal F_P}
 \left\{
  |(\mathbb P_L-\mathbb P)h_f|
  -\frac14 \mathbb P h_f
 \right\}
 \le
 C K_I^2
 \frac{
   P\log^2(P+1)\log(C_0A_PPL)+\log(2/\eta)
 }{L}.  
 \label{mbot:eq:union_bound}
\end{equation}
This is the standard Bernstein-type uniform deviation inequality for a
bounded excess-loss class satisfying a Bernstein condition;
see \citet[Theorem~3.3 and Sections~3.1.4--3.2]
{BartlettBousquetMendelson2005}
and \citet[Theorem~1]{Koltchinskii2006}.

Applying \eqref{mbot:eq:union_bound} to 
$\widehat u:=\widehat u_{B,L,P}$ gives 
\[
\mathbb P h_{\widehat u} \leq \mathbb P_L h_{\widehat u} + \frac14 \mathbb P h_{\widehat u} +  C K_I^2
 \frac{
   P\log^2(P+1)\log(C_0A_PPL)+\log(2/\eta)
 }{L}.  
\]
Then, the ERM property of $\widehat u$, i.e., $\mathbb P_L h_{\widehat u} \leq \mathbb P_L h_f + \varepsilon_{\mathrm{opt}}$ means  
\[
\frac{3}{4}\mathbb P h_{\widehat u} \leq \mathbb P_L h_f  +  C K_I^2
 \frac{
   P\log^2(P+1)\log(C_0A_PPL)+\log(2/\eta)
 }{L} + \varepsilon_{\mathrm{opt}}  
\]
for an arbitrary $f$.  Finally, \eqref{mbot:eq:union_bound} for $f$ gives
\[
\mathbb P_L h_f \leq \frac{5}{4} \mathbb P h_f + C K_I^2
 \frac{
   P\log^2(P+1)\log(C_0A_PPL)+\log(2/\eta)
 }{L}.  
\]
Combining the above expressions finally yields
\[
 \|\widehat u_{B,L,P}-v_B\|_{B,I}^2
 \le
 C\|f-v_B\|_{B,I}^2
 +
 C K_I^2
 \frac{
   P\log^2(P+1)\log(C_0A_PPL)+\log(2/\eta)
 }{L}
 +
 C\varepsilon_{\mathrm{opt}}.
\]
Taking the infimum over $f\in\mathcal F_P$ proves
\eqref{mbot:eq:batch-relu-oracle-zero}.

To prove \eqref{mbot:eq:batch-relu-field-zero}, by the triangle inequality,
\[
\|\widehat u_{B,L,P}-u^\star\|_{B,I}
\le
\|\widehat u_{B,L,P}-v_B\|_{B,I}
+
\|v_B-u^\star\|_{B,I}.
\]
Taking square roots in
\eqref{mbot:eq:batch-relu-oracle-zero} with 
$\sqrt{a+b+c}\le \sqrt a+\sqrt b+\sqrt c$ gives
\[
\begin{aligned}
\|\widehat u_{B,L,P}-v_B\|_{B,I}
\le{}&
C\inf_{f\in\mathcal F_P}\|f-v_B\|_{B,I}
\\
&+
CK_I
\sqrt{
\frac{
P\log^2(P+1)\log(C_0A_PPL)+\log(2/\eta)
}{L}
}
+
C\sqrt{\varepsilon_{\mathrm{opt}}}.
\end{aligned}
\]
For every $f\in\mathcal F_P$, another application of the
triangle inequality yields
\[
\|f-v_B\|_{B,I}
\le
\|f-u^\star\|_{B,I}
+
\|u^\star-v_B\|_{B,I}.
\]
Hence, by the definition of $a_P(B)$,
\[
\inf_{f\in\mathcal F_P}\|f-v_B\|_{B,I}
\le
a_P(B)+\|u^\star-v_B\|_{B,I}.
\]
Finally, Proposition~\ref{mbot:prop:batch-bias-zero} gives
\[
\|u^\star-v_B\|_{B,I}
\le
\sqrt{2\kappa_{S,I}\Delta_B}.
\]
Combining the preceding three displays completes the proof.
\end{proof}

\subsubsection{Intrinsic-dimensional approximation of the population OT denoiser}
\label{mbot:sec:intrinsic-approximation}

Standard smooth-function approximation treats the population OT field as
a function of $d+1$ variables
\citep{Yarotsky2017,SchmidtHieber2020,OhnKim2019}; the vector field $u(t,t)$ indeed takes a $d+1$-dimensional variable as its input. 
In this subsection, we exploit the proximal representation of the
population OT field. Under additional smoothness assumptions on the
target manifold and on the restriction of the reverse Brenier potential
to that manifold, we obtain an approximation exponent depending on
$m+1$, where $m=\dim M$.

Our construction first develops the OT manifold proximal map on the exact population OT interpolation graph and then extends it to a fixed ambient tube around the graph.  The tube extension is important: it yields a direct
approximation bound under the expected-batch interpolation law by
perturbing the proximal objective.
This construction adapts the approach of the finite-anchor projection network developed by \citet{fu2026intrinsic} from their Euclidean metric projection to the
OT-induced manifold proximal map.  More precisely, the
squared-distance objective for the nearest-point projection in their diffusion case is replaced
by the time-dependent proximal objective $\mathcal J_{(t,x)}$ generated
by the reverse Brenier potential, and the reach-based projection gap is
replaced by an OT-specific Fenchel--Bregman quadratic-growth inequality.
The proof architecture is therefore related to that of
\citet{fu2026intrinsic}, while the stability mechanism and the selected
map are specific to the OT proximal geometry.

\paragraph{ReLU arithmetic modules.}

We use the ReLU-network notation and complexity measures
$\mathcal L(\Phi)$, $\mathcal W(\Phi)$, $\mathcal M(\Phi)$, and
$\mathcal B(\Phi)$ introduced in
\eqref{mbot:eq:relu-network-definition-zero}--\eqref{mbot:eq:relu-complexity-zero}.
The following standard arithmetic modules will be used repeatedly.

\begin{lemma}[ReLU arithmetic networks]
\label{mbot:lem:relu-arithmetic}
Fix $B, q\in \N$.
There are constants $C,c>0$, depending only on $B$ and $q$, such that
for every $\varepsilon\in(0,1/2)$ the following hold;
\begin{enumerate}
\item There is a ReLU network
\[
\operatorname{Mult}_{B,\varepsilon}:[-B,B]^2\to\mathbb{R}
\]
such that
\[
\sup_{|a|,|b|\le B}
\left|
\operatorname{Mult}_{B,\varepsilon}(a,b)-ab
\right|
\le \varepsilon,
\]
and
\[
\operatorname{Mult}_{B,\varepsilon}(0,b)
=
\operatorname{Mult}_{B,\varepsilon}(a,0)
=
0.
\]
Moreover,
\[
\mathcal{L}(\operatorname{Mult}_{B,\varepsilon})
+
\mathcal{M}(\operatorname{Mult}_{B,\varepsilon})
\le
C\log(1/\varepsilon),
\qquad
\mathcal{B}(\operatorname{Mult}_{B,\varepsilon})
\le
\varepsilon^{-c}.
\]

\item Let
\[
p(u,z)
=
\sum_{|\beta|\le q}
\ell_\beta(u)z^\beta,
\]
where $u$ ranges over a fixed compact subset of a Euclidean space,
each $\ell_\beta$ is affine in $u$, and $z$ ranges over $[-B,B]^m$.
There is a ReLU network $\operatorname{Poly}_{p,\varepsilon}$ such
that
\[
\sup_{u,z}
\left|
\operatorname{Poly}_{p,\varepsilon}(u,z)-p(u,z)
\right|
\le\varepsilon,
\]
with
\[
\mathcal{L}(\operatorname{Poly}_{p,\varepsilon})
+
\mathcal{M}(\operatorname{Poly}_{p,\varepsilon})
\le
C\log(1/\varepsilon),
\qquad
\mathcal B(\operatorname{Poly}_{p,\varepsilon})
\le \varepsilon^{-c}.
\]

\item For every fixed interval $[a,b]\subset(0,\infty)$, there is a
ReLU network
\[
\operatorname{Inv}_{[a,b],\varepsilon}:[a,b]\to\mathbb{R}
\]
such that
\[
\sup_{r\in[a,b]}
\left|
\operatorname{Inv}_{[a,b],\varepsilon}(r)-r^{-1}
\right|
\le\varepsilon,
\]
and
\[
\mathcal{L}(\operatorname{Inv}_{[a,b],\varepsilon})
+
\mathcal{M}(\operatorname{Inv}_{[a,b],\varepsilon})
\le
C\log^2(1/\varepsilon),
\qquad
\mathcal B(\operatorname{Inv}_{[a,b],\varepsilon})
\le \varepsilon^{-c}.
\]
\end{enumerate}
\end{lemma}


\begin{proof}
The multiplication and monomial constructions follow from the standard
sawtooth approximation of the square function and polarization; see,
for example, \citet{Yarotsky2017} and
\citet[Lemmas~A.2--A.4]{SchmidtHieber2020}.

For part 1, the construction in \citet{SchmidtHieber2020} additionally gives the
zero-preserving property in part~1.

For part~2, use the monomial approximation networks of
\citet[Lemmas~A.3--A.4]{SchmidtHieber2020}.
The affine functions $\ell_\beta(u)$ are represented exactly.
Multiplying their outputs by the monomial approximations
using part~1 and summing the resulting terms gives the
desired network.
Since the dimensions and degree are fixed and all intermediate
quantities are bounded on the fixed domain, choosing the
individual approximation errors as sufficiently small
constant multiples of $\varepsilon$ gives the stated uniform
error with depth and sparsity $O(\log(1/\varepsilon))$.

For part~3, let $g(u):=u^{-1}$ on $[1,b/a]$.
For every $n\in\mathbb N_0$,
\[
    \sup_{u\in[1,b/a]}|g^{(n)}(u)|
    =\sup_{u\in[1,b/a]}\frac{n!}{u^{n+1}}
    \le n!.
\]
Thus, \citet[Lemma~A.6]{ElbrachterEtAl2021} provides a
ReLU approximation $\Phi_\delta$ of $g$ with uniform error
at most $\delta$, fixed width, depth
$O(\log^2(1/\delta))$, and bounded weights.
Taking $\delta=\min\{a,1\}\varepsilon$ and defining
\[
    \operatorname{Inv}_{[a,b],\varepsilon}(x)
    :=a^{-1}\Phi_\delta(x/a)
\]
gives uniform error at most $\varepsilon$.
Fixed width implies that sparsity has the same order as depth,
and the fixed input and output rescalings preserve these orders.
\end{proof}

\paragraph{The manifold proximal objective}

Recall that $M=K_1\subset\mathbb{R}^d$ is the support of the target
distribution, which is an $m$-dimensional embedded
manifold, and that $\phi:=\Psi^*$ is the reverse extended Brenier potential. We write
\[
\bar\phi:=\phi|_M.
\]
We first define the minimizer of the proximal objective
restricted to $M$, and approximate it by a neural network model. 
For $X=(t,x)\in I\times\mathbb{R}^d$, define the manifold proximal
objective
\begin{equation}
\mathcal{J}_X(z) := (1-t)\bar\phi(z) +
\frac{t}{2}\|z\|^2 - \langle x,z\rangle, \qquad z\in M.
\label{mbot:eq:JX}
\end{equation}
This corresponds to the proximal 
objective on $\mathbb R^d$ through 
\begin{equation}
(1-t) \left\{ \phi(z) + \frac{1}{2\lambda_t}
\left\|z-\frac{x}{t}\right\|^2 \right\}
=
(1-t)\phi(z) + \frac{t}{2}\|z\|^2 - \langle x,z\rangle
+ \frac{\|x\|^2}{2t}.
\label{mbot:eq:moreau-relation}
\end{equation}
Thus, $\mathcal J_X$ is essentially the proximal objective restricted on $M$. However, the form \eqref{mbot:eq:JX} is preferred for the following reasons; it remains well defined at $t=0$, where the usual variables $x/t$ and $\lambda_t$ are singular; and for every fixed $z\in M$, the value
$X\mapsto\mathcal J_X(z)$ is affine in $X=(t,x)$.  The last property is
crucial for the finite-anchor ReLU construction below.

For $x_0\in K_0$, set
\[
Z_t(x_0):=(1-t)x_0+tS(x_0),
\]
and introduce the population OT interpolation graph by
\begin{equation}
\Gamma_I^\star := \left\{ \bigl(t,Z_t(x_0)\bigr):
t\in I,\ x_0\in K_0 \right\}.
\label{mbot:eq:exact-graph}
\end{equation}
We define the \emph{manifold proximal map} on $\Gamma_I^\star$ by
\begin{equation}
\Pi_{\mathrm{OT}}(X)
:=
\operatorname*{argmin}_{z\in M}\mathcal{J}_X(z),
\qquad
X\in\Gamma_I^\star.
\label{mbot:eq:PiOT-def}
\end{equation}

\begin{assumption}
\label{mbot:ass:smooth-proximal-geometry}
Let $s=r+\alpha>0$, where $r\in\mathbb{N}_0$ and
$\alpha\in(0,1]$. Assume the following.

\begin{enumerate}
\item $M\subset\mathbb{R}^d$ is a compact embedded
$C^{r+2,\alpha}$ manifold without boundary, of dimension $m$.

\item The restriction $\bar\phi=\phi|_M$ 
is finite and belongs to $C^{r+2,\alpha}(M)$.

\item The Bregman co-coercivity of Assumption~\ref{mbot:ass:graph-restricted-bregman} holds: 
\begin{equation}
D_\Psi(x,z):=\Psi(x)-\Psi(z) - \ip{S(z)}{x-z}
\ge
\frac{1}{2L_S}
\|S(x)-S(z)\|^2,
\qquad x,z\in K_0.
\label{mbot:eq:graph-bregman-again}
\end{equation}

\item The manifold admits uniformly regular local charts:
there exist $\rho_{\mathrm{ch}}>0$ and $C_{\mathrm{ch}}<\infty$
such that, for every $y\in M$, there is a chart $Y_y:B_{\rho_{\mathrm{ch}}}^m(0)\to M$ with $Y_y(0)=y$ 
satisfying
\begin{equation}
C_{\mathrm{ch}}^{-1}I_m
\preceq
DY_y(\xi)^\top DY_y(\xi)
\preceq
C_{\mathrm{ch}}I_m
\label{mbot:eq:chart-metric}
\end{equation}
and
\begin{equation}
\|Y_y\|_{C^{r+2,\alpha}}
+
\|Y_y^{-1}\|_{C^{r+2,\alpha}}
\le C_{\mathrm{ch}}
\label{mbot:eq:chart-regularity}
\end{equation}
on their respective domains.
\end{enumerate}
\end{assumption}

The uniform-chart condition is a standard bounded-geometry
consequence for a compact, sufficiently smooth embedded manifold. It
is stated explicitly to guarantee that the constants in the network
construction must be uniform over the centers of the local charts.

\begin{lemma}
\label{mbot:lem:proximal-characterization}
Suppose Assumption~\ref{mbot:ass:smooth-proximal-geometry} holds.
Let $x_0\in K_0$, $y=S(x_0)$, $t\in I$, and
\[
X=\bigl(t,Z_t(x_0)\bigr)\in\Gamma_I^\star.
\]
Then $y$ is the unique minimizer of $\mathcal{J}_X$ on $M$:
\begin{equation}
y
=
\operatorname*{argmin}_{z\in M}\mathcal{J}_X(z) = \Pi_{\mathrm{OT}}(X).
\label{mbot:eq:Pi-as-minimizer}
\end{equation}
For $t>0$, it also holds that
\begin{equation}
y
=
Q_t(Z_t(x_0))
=
\operatorname{prox}_{\lambda_t\phi}
\left(
\frac{Z_t(x_0)}{t}
\right).
\label{mbot:eq:Pi-equals-prox}
\end{equation}
\end{lemma}

\begin{proof}
Since $x_0\in\partial\phi(y)$, the proximal optimality condition
gives $y=\operatorname{prox}_{\lambda_t\phi}(y+\lambda_t x_0)$.
Because $y+\lambda_t x_0 = Z_t(x_0)/t$, this proves \eqref{mbot:eq:Pi-equals-prox}. The proximal
minimizer on $\mathbb R^d$ is unique and belongs to $M$, so it is also the unique minimizer of $\mathcal J_X$, the restriction to $M$. The case $t=0$ follows directly from $x_0\in\partial\phi(y)$,
which means that $y$ minimizes $z\mapsto\phi(z)-\langle x_0,z\rangle$.
\end{proof}

This lemma implies that, on the exact population OT graph, the manifold proximal map $\Pi_{\text{OT}}$ equals to the denoizer $Q_t$.

\begin{lemma}
\label{mbot:lem:uniform-quadratic-growth}
Let $x_0\in K_0$.  Under Assumption~\ref{mbot:ass:smooth-proximal-geometry}, for every
$X=(t,Z_t(x_0))\in\Gamma_I^\star$ and every $z\in M$,
\begin{equation}
\mathcal{J}_X(z)-\mathcal{J}_X(S(x_0))
\ge
\frac{\mu_\star}{2}\|z-S(x_0)\|^2,
\label{mbot:eq:uniform-QG}
\end{equation}
where $\mu_\star := \min\{1,\frac{1}{L_S}\}>0$.
\end{lemma}
\begin{proof}
Let $y:=S(x_0)$. Since $M=S(K_0)$, choose $z_0\in K_0$ such that $z=S(z_0)$.
Fenchel equality at $(x_0,y)$ and $(z_0,z)$ gives
\begin{equation}
\bar\phi(z)-\bar\phi(y) - \langle x_0,z-y\rangle
= \Psi(x_0)-\Psi(z_0)  -  \langle S(z_0),x_0-z_0\rangle
=
D_\Psi(x_0,z_0).
\label{mbot:eq:dual-bregman}
\end{equation}
Using $x=Z_t(x_0)=(1-t)x_0+ty$, we obtain
\begin{equation}
\begin{aligned}
\mathcal{J}_X(z)-\mathcal{J}_X(y)
&=
(1-t)
\left\{
\bar\phi(z)-\bar\phi(y)
-
\langle x_0,z-y\rangle
\right\}
+
\frac{t}{2}\|z-y\|^2
\\
&=
(1-t)D_\Psi(x_0,z_0)
+
\frac{t}{2}\|z-y\|^2.
\end{aligned}
\label{mbot:eq:J-gap-Bregman}
\end{equation}
By \eqref{mbot:eq:graph-bregman-again}, $D_\Psi(x_0,z_0)
\ge \frac{1}{2L_S}\|y-z\|^2$, and 
consequently,
\[
\mathcal{J}_X(z)-\mathcal{J}_X(y)
\ge
\frac{1}{2}
\left(
t+\frac{1-t}{L_S}
\right)
\|z-y\|^2
\ge
\frac{\mu_\star}{2}\|z-y\|^2.
\]
\end{proof}

Lemma~\ref{mbot:lem:uniform-quadratic-growth} establishes
a uniform quadratic-growth bound for $\mathcal J_X$ on $M$
when $X\in\Gamma_I^\star$.  To discuss the paths given by minibatch OT, we need to consider a slightly larger set than $\Gamma_I^\star$ for parameter $X$.  
In the next lemma, we extend the above bound of quadratic growth to parameter $X^e=(t,Z_t(x_0)+e)$ for sufficiently small $\|e\|$.

\begin{lemma}
\label{mbot:lem:uniform-tube-stability}
Under Assumption~\ref{mbot:ass:smooth-proximal-geometry}, there exist
constants $\delta_0>0$ and $\mu_{\mathrm{tub}}>0$
such that the following holds.  Define
\begin{equation}
\Gamma_{I,\delta_0}^\star
:=
\left\{
\bigl(t,Z_t(x_0)+e\bigr):
 t\in I,\ x_0\in K_0,\ \|e\|\le\delta_0
\right\}.
\label{mbot:eq:stable-tube}
\end{equation}
For
\[
X^e=\bigl(t,Z_t(x_0)+e\bigr)\in
\Gamma_{I,\delta_0}^\star,
\qquad y=S(x_0),
\]
the restriction $\mathcal J_{X^e}|_M$ has a unique minimizer, denoted by
\[
y_e:=\Pi_{\mathrm{OT}}(X^e)
:=\operatorname*{argmin}_{z\in M}\mathcal J_{X^e}(z).
\]
Moreover,
\begin{equation}
\|y_e-y\|
\le
\frac{2}{\mu_\star}\|e\|,
\label{mbot:eq:tube-selector-stability}
\end{equation}
and, for every $z\in M$,
\begin{equation}
\mathcal J_{X^e}(z)-\mathcal J_{X^e}(y_e)
\ge
\frac{\mu_{\mathrm{tub}}}{2}\|z-y_e\|^2.
\label{mbot:eq:tube-QG}
\end{equation}
The constants are uniform over all representations of $X^e$ in
\eqref{mbot:eq:stable-tube}.
\end{lemma}
\begin{proof}
\noindent
{\bf localization of every global minimizer.}
Fix $t\in I$ and $x_0\in K_0$, and write
$X^0:=(t,Z_t(x_0))$, $X^e:=(t,Z_t(x_0)+e)$, and $y:=S(x_0)$.
Since $M$ is compact and $\mathcal J_{X^e}|_M$ is continuous,
a global minimizer exists. Let $y_e$ be any such minimizer;
uniqueness is not assumed at this stage.
The definition \eqref{mbot:eq:JX} gives $\mathcal J_{X^e}(z)
=\mathcal J_{X^0}(z)-\langle e,z\rangle$ for any $z\in M$, and therefore 
\[
\mathcal J_{X^e}(z)-\mathcal J_{X^e}(y)
=\mathcal J_{X^0}(z)-\mathcal J_{X^0}(y)
-\langle e,z-y\rangle.
\]
Lemma \ref{mbot:lem:uniform-quadratic-growth} with $X=X^0$ and $S(x_0)=y$ yields 
\begin{equation}
\mathcal J_{X^e}(z)-\mathcal J_{X^e}(y)
\ge \frac{\mu_\star}{2}\|z-y\|^2-\|e\|\,\|z-y\|,
\qquad z\in M.
\label{eq:JXe_zy}
\end{equation}
Taking $z=y_e$ and using
$\mathcal J_{X^e}(y_e)\le\mathcal J_{X^e}(y)$ gives
\[
\frac{\mu_\star}{2}\|y_e-y\|^2
\le \|e\|\,\|y_e-y\|, 
\]
which proves
\eqref{mbot:eq:tube-selector-stability}.  This estimate holds for every global minimizer, independently of whether the minimizer is unique.

\smallskip
{\bf Stability with parameters.}
To prove the quadratic growth \eqref{mbot:eq:tube-QG}, we use the results of \citet{Lewis2002}, where they discuss strong minimizers, i.e., quadratic growth in a neighborhood.  In particular, to derive the uniform coefficient of the quadratic growth in \eqref{mbot:eq:tube-QG}, we use Theorem 5.5 of \citet{Lewis2002}, which discusses the strong minimizers depending on a parameter. For our purpose, we consider $\mathcal J_X(z)$ as a function of $z$ with parameter $X$.

Define $f_X(z):= \mathcal J_X(z)$ for $z\in M$ and $+\infty$ for $z\notin M$, and parametric mapping 
\[
\tilde f(X,z):=f_X(z),
\]
where $X=(t,x)\in \mathbb R^{d+1}$ is regarded as a parameter for the minimization problem given by $z\mapsto f_X(z)$. Let $\tilde M:=\mathbb R^{d+1}\times M$. 
The function $\tilde f$ is proper and lower semicontinuous.
By Assumption~\ref{mbot:ass:smooth-proximal-geometry}(1)--(2),
$\tilde M$ is a $C^2$ manifold, and $\tilde f$ is a $C^2$ function on $\tilde M$ extended by $+\infty$ outside $\tilde M$.

The quadratic growth \eqref{mbot:eq:uniform-QG} means that
$y=S(x_0)$ is a strong local minimizer of
$\tilde f(X^0,\cdot)|_M$.
Moreover, the transversality condition required for the theorem in \cite{Lewis2002} is obvious under Assumption~\ref{mbot:ass:smooth-proximal-geometry}(1)--(2).  
From \citet[Theorem~5.5]{Lewis2002}, there are neighborhoods
$V$ of $X^0$ and $U$ of $y$ such that, for every $X'\in V$,
the restriction $\mathcal J_{X'}|_{M\cap U}$ has a unique
critical point, which is a strong local minimizer and
depends on $X'$ in $C^1$-smooth.

\smallskip
\noindent{\bf Global minimizer.}
We next show that every global minimizer lies in $U$ after
shrinking $V$.
Choose $a>0$ such that $B_y(a):=\{y'\in \mathbb R^d\mid \|y'-y\|\le a\} \subset U$, and set
$\gamma:=\mu_\star a^2/2>0$.
For $z\in M\setminus U$, \eqref{mbot:eq:uniform-QG} implies
\[
\mathcal J_{X^0}(z)-\mathcal J_{X^0}(y)\ge\gamma.
\]
For $X'=(t',x')$, the definition \eqref{mbot:eq:JX} also gives
\[
\mathcal J_{X'}(z)-\mathcal J_{X^0}(z)
=(t'-t)\Bigl(\frac12\|z\|^2-\bar\phi(z)\Bigr)
-\langle x'-Z_t(x_0),z\rangle.
\]
Compactness of $M$ therefore provides a finite constant $C_J>0$
such that
\[
\sup_{z\in M}|\mathcal J_{X'}(z)-\mathcal J_{X^0}(z)|
\le C_J\|X'-X^0\|.
\]
Shrink $V$ so that $2C_J\|X'-X^0\|<\gamma/2$ on $V$.
For every $z\in M\setminus U$, it then follows that
\[
\mathcal J_{X'}(z)-\mathcal J_{X'}(y)
\ge\gamma-2C_J\|X'-X^0\|>\gamma/2,
\]
which shows that no point of $M\setminus U$ can be a global minimizer.
Hence, the local minimizer in $M\cap U$ is also the global minimizer. 
We may therefore denote this point by $\Pi_{\mathrm{OT}}(X')$, which is 
single-valued  and $C^1$ on $V$.

Apply this argument at every $X^0\in\Gamma_{I}^\star$.
On overlapping parameter neighborhoods, the resulting maps agree,
because both give the unique global minimizer of the same objective.
They therefore define a single $C^1$ map $\Pi_{\mathrm{OT}}$
on an open neighborhood $\mathcal V$ of $\Gamma_I^\star$.

\smallskip
\noindent{\bf Uniform Lipschitz bound.}
We will show the Lipschitz property of $e \mapsto y_e=\Pi_{\mathrm{OT}}(X^e)$.  By the compactness of $\Gamma_I^\star$, 
choose $\rho>0$ such that
\[
\Gamma_{I,2\rho}^\star
=\{(t,Z_t(x_0)+e):t\in I,\ x_0\in K_0,\ \|e\|\le2\rho\}
\subset\mathcal V.
\]
Since $\Pi_{\mathrm{OT}}$ is $C^1$ on $\mathcal V$, the constant
\[
L:=\max\Bigl\{1,
\sup_{X\in\Gamma_{I,2\rho}^\star}
\|D_x\Pi_{\mathrm{OT}}(X)\|_{\mathrm{op}}\Bigr\}
\]
is finite.
For fixed $t,x_0$ and $\|e_1\|,\|e_2\|\le2\rho$, put
$e_\theta:=(1-\theta)e_1+\theta e_2$ for $\theta\in[0,1]$.
The entire segment $(t,Z_t(x_0)+e_\theta)$ remains in this tube.
Expressing $y_{e_1}-y_{e_2}$ in terms of the integral over the segment yields 
\begin{equation}
    \|y_{e_2}-y_{e_1}\|\le L\|e_2-e_1\|,
\qquad \|e_1\|,\|e_2\|\le2\rho,
\label{eq:ye_Lipschitz}
\end{equation}
with the same $L$ for every $X^0\in\Gamma_I^\star$.

We now choose the radius $\delta_0$ of the smaller tube
in the statement of the lemma.  Take $0<\delta_0\le\min\{\rho/2,\mu_\star\rho/4\}$.  Then, \eqref{mbot:eq:tube-selector-stability} gives $\|y_e-y\|\le\rho/2$ whenever $\|e\|\le\delta_0$.

\smallskip
\noindent{\bf Quadratic growth with a common constant.}
We derive the uniform quadratic-growth bound for $X^e\in \Gamma_{I,\delta_0}^\star$ from
the uniform Lipschitz estimate \eqref{eq:ye_Lipschitz} and the optimality
of the perturbed minimizers.

\medskip 
We first consider the case $z\in M$ with $\|z-y\|<\rho$. Fix $\|e\|\le\delta_0$, and introduce the additional perturbation
\[
h:=\frac{z-y_e}{2L}.
\]
Since $\|z-y_e\|\le\|z-y\|+\|y-y_e\|<3\rho/2$, $\|e\|\leq \delta_0\leq \rho/2$, and $L\ge1$, we have $\|e+h\|\le\|e\|+\frac{\|z-y_e\|}{2L} <\frac{5\rho}{4}<2\rho$.
Thus, $X^{e+h}:=(t,Z_t(x_0)+e+h)$ belongs to the larger tube $\Gamma_{I,2\rho}^\star$ so that 
$y_{e+h}:=\Pi_{\mathrm{OT}}(X^{e+h})$ is well defined. \eqref{eq:ye_Lipschitz} gives $\|y_{e+h}-y_e\|\le L\|h\|$.
From the identity
$\mathcal J_{X^{e+h}}=\mathcal J_{X^e}-\langle h,\cdot\rangle$, 
the optimality of $y_{e+h}$ for $\mathcal J_{X^{e+h}}$ and the optimality of $y_e$ for $\mathcal J_{X^e}$ provide
\[
\mathcal J_{X^e}(z)-\langle h,z\rangle
\ge\mathcal J_{X^e}(y_{e+h})-\langle h,y_{e+h}\rangle \ge J_{X^e}(y_e) - \langle h,y_{e+h}\rangle,
\]
from which we obtain
\[
\mathcal J_{X^e}(z)-\mathcal J_{X^e}(y_e)
\ge\langle h,z-y_{e+h}\rangle
=\langle h,z-y_e\rangle-\langle h,y_{e+h}-y_e\rangle.
\]
Cauchy-Schwarz and the Lipschitz bound \eqref{eq:ye_Lipschitz} imply
$\langle h,y_{e+h}-y_e\rangle
\le\|h\|\,\|y_{e+h}-y_e\|
\le L\|h\|^2$. The definition of $h$ therefore yields
\[
\mathcal J_{X^e}(z)-\mathcal J_{X^e}(y_e)
\ge\langle h,z-y_e\rangle-L\|h\|^2
=\frac{1}{4L}\|z-y_e\|^2.
\]
Consequently, with $\mu_{\mathrm{loc}}:=1/(2L)>0$,
\begin{equation}
\mathcal J_{X^e}(z)-\mathcal J_{X^e}(y_e)
\ge\frac{\mu_{\mathrm{loc}}}{2}\|z-y_e\|^2
\quad \text{for}\quad z\in M,\quad \|z-y\|<\rho.
\label{eq:quadratic_growth_near}
\end{equation}
Both $\rho$ and $\mu_{\mathrm{loc}}$ are independent of $t,x_0,e$
within the ranges specified above.

\medskip 
It remains to consider $z\in M$ with $\|z-y\|\ge\rho$.
The choice of $\delta_0$ implies $\|e\|\le\frac{\mu_\star\rho}{4}
\le\frac{\mu_\star}{4}\|z-y\|$.  
Using \eqref{eq:JXe_zy} and
$\mathcal J_{X^e}(y_e)\le\mathcal J_{X^e}(y)$, we obtain
\[
\mathcal J_{X^e}(z)-\mathcal J_{X^e}(y_e)
\ge\mathcal J_{X^e}(z)-\mathcal J_{X^e}(y)
\ge\frac{\mu_\star}{4}\|z-y\|^2.
\]
Since $\|y_e-y\|\le\rho/2$, we have $\|z-y_e\|\le\|z-y\|+\frac{\rho}{2}
\le\frac32\|z-y\|$. Thus
\[
\mathcal J_{X^e}(z)-\mathcal J_{X^e}(y_e)
\ge\frac{\mu_\star}{9}\|z-y_e\|^2,
\qquad \|z-y\|\ge\rho.
\]
Together with \eqref{eq:quadratic_growth_near}, this proves
\eqref{mbot:eq:tube-QG} with
$ \mu_{\mathrm{tub}}
:=\min\{\frac{1}{2L},\frac{2\mu_\star}{9}\}$. 

The choices of $\delta_0$ and $\mu_{\mathrm{tub}}$ depend only on
the fixed objective family and its OT interpolation graph,
not on a particular $t$ or $x_0$.
\end{proof}

\begin{remark}
The arguments to derive the local quadratic growth bound are consistent with the general equivalence between tilt stability and uniform local
quadratic growth established by \citet[Corollary~3.2]{DrusvyatskiyLewis2013}.  We derived the bound more directly in the proof. 
\end{remark}

Observe that neither Lemma~\ref{mbot:lem:uniform-quadratic-growth} nor
Lemma~\ref{mbot:lem:uniform-tube-stability} requires a centered normal fiber
$0\in C_y$, or a smooth global map 
$y\mapsto s(y)\in\partial\phi(y)$. The entire normal OT fiber is
encoded through the input $x$ in the proximal objective
$\mathcal{J}_X$.

\paragraph{Intrinsic-dimensional ReLU approximation}

We now state the main approximation theorem.

\begin{theorem}
\label{mbot:thm:intrinsic-Pi-approximation}
Suppose Assumption~\ref{mbot:ass:smooth-proximal-geometry} holds, with \(s=r+\alpha>0\) as specified there.
There exist constants $C,h_0,R_\Phi>0$, depending only on
$d,m,s,c,K_0,M,L_S,
\|\bar\phi\|_{C^{r+2,\alpha}(M)}$, and $C_{\mathrm{ch}}$,
such that, for every $h\in(0,h_0)$, there is a ReLU network $\Phi_h:I\times\mathbb{R}^d\to\mathbb{R}^d$ satisfying
\begin{equation}
\sup_{X\in\Gamma_{I,\delta_0}^\star}
\left\|
\Phi_h(X)-\Pi_{\mathrm{OT}}(X)
\right\|
\le
C h^s,
\label{mbot:eq:h-approximation}
\end{equation}
and
\begin{equation}
\mathcal{M}(\Phi_h) \le C h^{-(m+1)}\log^2(1/h), \qquad 
\mathcal{L}(\Phi_h) \le C\log^2(1/h),
\qquad
\mathcal{B}(\Phi_h)
\le h^{-C}.
\label{mbot:eq:h-depth-weight}
\end{equation}
The network may additionally be chosen with the global output bound
\begin{equation}
\sup_{(t,x)\in I\times\mathbb R^d}\|\Phi_h(t,x)\|
\le R_\Phi.
\label{mbot:eq:Phi-global-bound}
\end{equation}
Consequently, for every sufficiently large $P$, there exists a ReLU
network $\Phi_P$ with
\begin{equation}
\mathcal{M}(\Phi_P)\le P,
\qquad
\mathcal{L}(\Phi_P)\le C\log^2 P,
\qquad
\mathcal{B}(\Phi_P)\le P^C,
\label{mbot:eq:P-complexity}
\end{equation}
such that
\begin{equation}
\sup_{X\in\Gamma_{I,\delta_0}^\star}
\left\| \Phi_P(X)-\Pi_{\mathrm{OT}}(X) \right\|
\le C \left( \frac{P}{\log^2 P} \right)^{-s/(m+1)}.
\label{mbot:eq:P-intrinsic-rate}
\end{equation}

\end{theorem}

\smallskip
\begin{remark}
Theorem \ref{mbot:thm:intrinsic-Pi-approximation} shows that the manifold proximal map $\Pi_{\text{OT}}$ can be approximated by a ReLU network in such accuracy that  the polynomial exponent with respect to the
number of nonzero parameters is $s/(m+1)$, independent of the
ambient dimension $d$.  The constants and the polynomial weight-envelope
exponent may still depend on $d$.
\end{remark}
\begin{proof}
The proof is constructive.

\paragraph*{Step 1: $h$-net of $I\times M$.} 
Equip $\mathcal{B}:=I\times M$ with the metric
\begin{equation*}
d_{\mathcal{B}}\bigl((t,y),(t',y')\bigr)
:=
|t-t'|+d_M(y,y').
\label{mbot:eq:base-metric}
\end{equation*}
Choose a maximal $h$-separated set $\{b_j=(t_j,y_j)\}_{j=1}^{N_h}\subset \mathcal B$. It is also an $h$-net. Since $I\times M$ is a compact
$(m+1)$-dimensional manifold, the size of the set is bounded by 
\begin{equation}
N_h\le C h^{-(m+1)}.
\label{mbot:eq:covering-number}
\end{equation}

For $X=(t,x)\in\Gamma_{I,\delta_0}^\star$, write
\[
y:=\Pi_{\mathrm{OT}}(X) \,\in M.
\]

\paragraph*{Step 2: energy-based localization.}

For each $j=1,\ldots,N_h$, define the score $\Sigma_j(X)$ to measure the discrepancy between 
$(t,y)$ and $b_j$ by 
\begin{equation}
\Sigma_j(X) := \mathcal{J}_X(y_j) + \lambda h|t-t_j|,
\label{mbot:eq:score}
\end{equation}
where $\lambda>0$ is a fixed constant, and set
\begin{equation}
\underline{\Sigma}_h(X) := \min_{1\le j\le N_h}\Sigma_j(X).
\label{mbot:eq:min-score}
\end{equation}
Notice that $\mathcal{J}_X(y_j)$ is affine in the input $X=(t,x)$ of $\Sigma_j(X)$. Also, $|t-t_j| =
\rho(t-t_j)+\rho(t_j-t)$ and $\min\{a,b\}=a-\rho(a-b)$.
Hence, all the scores and their minimum can be represented exactly by a ReLU network with $O(N_h)$ nonzero parameters and depth $O(\log N_h)$.

We next establish the localization property. Since $y$ is a minimizer
of $\mathcal{J}_X|_M$, its intrinsic gradient vanishes:
\[
\nabla_M\mathcal{J}_X(y)=0.
\]
Smoothness of $\bar\phi$ and compactness of $M$ give a constant $L_J<\infty$ such that, whenever $d_M(y,z)$ is sufficiently small,
\begin{equation}
0
\le
\mathcal{J}_X(z)-\mathcal{J}_X(y)
\le L_Jd_M(y,z)^2,
\label{mbot:eq:local-upper-energy}
\end{equation}
uniformly in $X\in\Gamma_{I,\delta_0}^\star$.

For $(t,y)$, choose $j_0$ such that $d_{\mathcal{B}}\bigl((t,y),b_{j_0}\bigr)\le h$.  Then, it follows from \eqref{mbot:eq:local-upper-energy} that  
$\Sigma_{j_0}(X) \le \mathcal{J}_X(y)+C_0h^2$ 
for a constant $C_0$, and therefore,
\begin{equation}
\underline{\Sigma}_h(X) \le \mathcal{J}_X(y)+C_0h^2.
\label{mbot:eq:min-score-upper}
\end{equation}

Fix $A>0$ and define the raw gates
\begin{equation}
\widetilde{\eta}_j(X)
:=
\left[
1-
\frac{
\Sigma_j(X)-\underline{\Sigma}_h(X)
}{
Ah^2
}
\right]_+.
\label{mbot:eq:raw-gate}
\end{equation}
The score difference is nonnegative, so
$0\le\widetilde{\eta}_j\le1$. The index attaining the minimum score
satisfies $\widetilde{\eta}_j(X) =1$.

If $\widetilde{\eta}_j(X)>0$, then $\Sigma_j(X) < \underline{\Sigma}_h(X)+Ah^2$, which in 
combination with \eqref{mbot:eq:min-score-upper} and
Lemma~\ref{mbot:lem:uniform-tube-stability} yields
\begin{equation*}
\frac{\mu_{\mathrm{tub}}}{2}\|y_j-y\|^2
+
\lambda h|t_j-t|
\le
(A+C_0)h^2.
\label{mbot:eq:active-localization-euclidean}
\end{equation*}
Consequently,
\begin{equation}
\|y_j-y\|\le C h,
\qquad
|t_j-t|\le C h.
\label{mbot:eq:active-localization}
\end{equation}
for a constant $C>0$. For sufficiently small $h$, local equivalence of the Euclidean and geodesic metrics on $M$ gives
\begin{equation}
\widetilde{\eta}_j(X)>0
\quad\Longrightarrow\quad
d_{\mathcal{B}}\bigl((t,y),b_j\bigr)\le C_{\mathrm{loc}}h
\label{mbot:eq:gate-localization}
\end{equation}
for a constant $C_{\mathrm{loc}}$.  

Since the set $\{b_j\}$ is $h$-separated, the packing estimate on the
$(m+1)$-dimensional manifold $I\times M$ implies that there is a
constant $N_{\mathrm{ov}}<\infty$, independent of $h$, such that
\begin{equation}
\#\{j:\widetilde{\eta}_j(X)>0\}
\le N_{\mathrm{ov}}
\label{mbot:eq:bounded-overlap}
\end{equation}
for every $X\in\Gamma_{I,\delta_0}^\star$. We call such $j$ an active index. 

\paragraph*{Step 3: local proximal experts.}

For $X\in \Gamma_{I,\delta_0}^\star$, fix an active index $j$ and use the uniformly regular chart $Y_j:=Y_{y_j}:B_{\rho_{\mathrm{ch}}}^m(0)\to M$ in the assumption \ref{mbot:ass:smooth-proximal-geometry}. 
From \eqref{mbot:eq:gate-localization}, by taking sufficiently small $h$, the true minimizer $y$ belongs to
this chart and can be written as
\begin{equation}
y=Y_j(\xi_j^\star),
\qquad
\|\xi_j^\star\|\le C_r h.
\label{mbot:eq:true-local-coordinate}
\end{equation}

Define
\[
A_j(\xi):=\bar\phi(Y_j(\xi))
\]
and
\begin{equation*}
J_{j,X}(\xi)
:=
\mathcal{J}_X(Y_j(\xi)).
\label{mbot:eq:local-objective}
\end{equation*}
Its gradient is
\begin{equation}
G_j(X,\xi) :=
\nabla_\xi J_{j,X}(\xi) =
(1-t)\nabla A_j(\xi) + DY_j(\xi)^\top
\{tY_j(\xi)-x\}.
\label{mbot:eq:local-gradient}
\end{equation}
In particular, for each fixed $\xi$, the map
$X=(t,x)\mapsto G_j(X,\xi)$ is affine.


We first consider $s>1$.
We use the $C^1$ regularity of $\Pi_{\mathrm{OT}}$ on an open
neighborhood of $\Gamma_{I,\delta_0}^\star$, established in the proof
of Lemma~\ref{mbot:lem:uniform-tube-stability}. In particular, compactness of the tube gives
\[
L_\Pi:=\max\left\{1,
\sup_{X\in\Gamma_{I,\delta_0}^\star}
\|D_x\Pi_{\mathrm{OT}}(X)\|_{\mathrm{op}}\right\}<\infty.
\]
This constant is independent of $h$ and of the anchor index $j$.
For fixed $j$, the expression
\[
\xi_j^\star(X)=Y_j^{-1}(\Pi_{\mathrm{OT}}(X))
\]
defines a $C^1$ map on a neighborhood of any input for which the
minimizer belongs to the chart. Such a neighborhood exists at each
active pair $(j,X)$, by \eqref{mbot:eq:true-local-coordinate}
and continuity of $\Pi_{\mathrm{OT}}$.
Since $\Pi_{\mathrm{OT}}(X)$ minimizes $\mathcal J_X$ on $M$, we have 
\begin{equation}
\label{eq:statinarity_G}
  G_j(X,\xi_j^\star(X))=0.  
\end{equation}
Thus, the local solution and its smooth dependence on $X$ are inherited
from Lemma~\ref{mbot:lem:uniform-tube-stability}.

\smallskip
\paragraph*{Step 3: A uniform Hessian bound at the minimizer.}
Fix an active pair $(j,X)$ and abbreviate
\[
B_j:=DY_j(\xi_j^\star),\qquad
H_j:=D_\xi^2J_{j,X}(\xi_j^\star).
\]
The matrix $H_j$ is symmetric positive semidefinite, since
$\xi_j^\star$ is an interior local minimizer of the $C^2$ function
$J_{j,X}$.
By \eqref{mbot:eq:local-gradient}, differentiation with respect
to $x$, with $\xi$ held fixed, gives $D_xG_j(X,\xi)=-DY_j(\xi)^\top$.
Differentiating the stationarity equation \eqref{eq:statinarity_G} with respect to $x$ therefore yields
\[
H_jD_x\xi_j^\star(X)=B_j^\top.
\]
The chart metric bound \eqref{mbot:eq:chart-metric} makes $B_j$
full column rank. Hence $B_j^\top$ has rank $m$, and the last
identity forces the $m\times m$ matrix $H_j$ to be invertible.
Combined with positive semidefiniteness, this shows that $H_j$
is positive definite. Consequently,
\[
D_x\xi_j^\star(X)=H_j^{-1}B_j^\top.
\]
Differentiating $\Pi_{\mathrm{OT}}(X)=Y_j(\xi_j^\star(X))$ now gives
\[
D_x\Pi_{\mathrm{OT}}(X)=B_jH_j^{-1}B_j^\top.
\]
In particular, this derivative is symmetric positive semidefinite,
and its operator-norm bound implies
$B_jH_j^{-1}B_j^\top\preceq L_\Pi I_d$.
Multiply this inequality on the left by
$(B_j^\top B_j)^{-1}B_j^\top$ and on the right by its transpose.
Using \eqref{mbot:eq:chart-metric}, we obtain
\[
H_j^{-1}\preceq L_\Pi(B_j^\top B_j)^{-1}
\preceq C_{\mathrm{ch}}L_\Pi I_m.
\]
Thus, with $\mu_1:=(2C_{\mathrm{ch}}L_\Pi)^{-1}>0$,
\begin{equation}
D_\xi^2J_{j,X}(\xi_j^\star)\succeq 2\mu_1I_m,
\label{mbot:eq:hessian-at-min}
\end{equation}
uniformly over all inputs in the stable tube and all active indices.

\smallskip
\paragraph*{Step 4: Local approximation of $\Pi_{\text{OT}}$ by a neural network.}  In this part, we approximate $X\mapsto \Pi_{\text{OT}}(X)$ on a local cube. Because $\Pi_{\text{OT}}(X)$ is defined by a minimization on $M$, we approximate it by using the iteration of gradient descent with a neural network approximation of the gradient. 

We first derive a uniform bound for the second derivative of the objective function $J_{j,X}$ and show that the gradient descent is a contracting map. The smoothness of $\bar\phi$, the uniform chart bounds
\eqref{mbot:eq:chart-metric}--\eqref{mbot:eq:chart-regularity},
and compactness of the tube give constants $C_H,L_1<\infty$,
independent of $h,j,X$, such that, for any $\xi$ with $\|\xi\| < \rho_{\mathrm{ch}}/2$, 
\begin{equation}\label{eq:Hessian_bound}
  \|D_\xi^2J_{j,X}(\xi)-D_\xi^2J_{j,X}(\zeta)\|_{\mathrm{op}}
\le C_H\|\xi-\zeta\|^\alpha\quad\text{and}\quad \|D_\xi^2J_{j,X}(\xi)\|_{\mathrm{op}}\le L_1.
\end{equation}
Enlarge $L_1$ if necessary so that $L_1>2\mu_1$.
For $C_r$ in \eqref{mbot:eq:true-local-coordinate}, choose a fixed $R_r$ such that $R_r>2C_r+1$, and decrease $h_0\in(0,1)$ if necessary 
so that
\[
(\sqrt m R_r+C_r)h_0<\rho_{\mathrm{ch}}/2,
\qquad
C_H\bigl((\sqrt m R_r+C_r)h_0\bigr)^\alpha\le\mu_1.
\]
Then, for any $h<h_0$ and $\xi$ in the local cube $[-R_r h, R_r h]^m$, we have $\|\xi\|\leq \sqrt{m}\|\xi\|_\infty \le \sqrt{m} R_r h$ and therefore 
\eqref{mbot:eq:true-local-coordinate} gives
\[
\|\xi-\xi_j^\star\|\le \|\xi\| + \|\xi_j^\star\| \leq (\sqrt m R_r+C_r)h.
\]
By the choice of the constants and \eqref{eq:Hessian_bound}, we have $\|D_\xi^2J_{j,X}(\xi)-D_\xi^2J_{j,X}(\xi_j^\star)\|_{\mathrm{op}} \leq \mu_1$.
Combining this with \eqref{mbot:eq:hessian-at-min} yields, for $\|\xi\|_\infty\le R_r h$ and $h< h_0$, 
\begin{equation}
\mu_1I_m\preceq D_\xi^2J_{j,X}(\xi)\preceq L_1 I_m.
\label{mbot:eq:local-strong-convexity}
\end{equation}

Choose
\begin{equation}
\gamma:=\frac{2}{L_1+\mu_1},\qquad
q_{\mathrm{ctr}}:=\frac{L_1-\mu_1}{L_1+\mu_1}\in(0,1).
\label{mbot:eq:gradient-step}
\end{equation}
For $\xi,\zeta$ in the local cube $[-R_r h,R_r h]^m$, define the averaged Hessian
\[
\overline H_{j,X}(\xi,\zeta)
:=\int_0^1 D_\xi^2J_{j,X}(\zeta+\theta(\xi-\zeta))\,d\theta.
\]
The segment lies in the cube, so
$\mu_1I_m\preceq\overline H_{j,X}(\xi,\zeta)\preceq L_1I_m$.
Define the map of the gradient descent of $J_{j,X}$ by 
\[
T_{j,X}(\xi):=\xi-\gamma G_j(X,\xi).
\]
The fundamental theorem of calculus gives
\[
T_{j,X}(\xi)-T_{j,X}(\zeta)
=\bigl(I_m-\gamma\overline H_{j,X}(\xi,\zeta)\bigr)(\xi-\zeta).
\]
The choice of $\gamma$ implies
$\|I_m-\gamma\overline H_{j,X}(\xi,\zeta)\|_{\mathrm{op}}
\le q_{\mathrm{ctr}}$; hence
\begin{equation}
\|T_{j,X}(\xi)-T_{j,X}(\zeta)\|
\le q_{\mathrm{ctr}}\|\xi-\zeta\|,
\label{mbot:eq:gradient-contraction}
\end{equation}
which means that $T_{j,X}$ satisfies a uniform contraction estimate on the local cube.
All constants in these estimates are independent of $h$, $j$,
and the active input $X$.



Next, we approximate $\xi_j^*=Y_j^{-1}(\Pi_{\text{OT}}(X))$ using a neural network that approximately realizes the gradient descent optimization. Because $s=r+\alpha$ and the map
$\xi\mapsto G_j(X,\xi)$ has a uniformly bounded 
$C^{r,\alpha}$-norm, its Taylor expansion of order $r$ at the
origin,
\begin{equation}
T_{j,X}^{(r)}(\xi)
:=
\sum_{|\beta|\le r}
\frac{D_\xi^\beta G_j(X,0)}{\beta!}\xi^\beta,
\label{mbot:eq:gradient-taylor}
\end{equation}
satisfies
\begin{equation}
\sup_{\|\xi\|_\infty\le R_r h}
\left\|
G_j(X,\xi)-T_{j,X}^{(r)}(\xi)
\right\|
\le
C h^s.
\label{mbot:eq:gradient-taylor-error}
\end{equation}
Each coefficient
$D_\xi^\beta G_j(X,0)$ is affine in $X$.
By Lemma~\ref{mbot:lem:relu-arithmetic}, there is a ReLU network
$\widehat G_{j,h}$ such that it achieves approximation of order $h^s$, i.e.,
\begin{equation}
\sup_{\substack{X\in\Gamma_{I,\delta_0}^\star\\
\|\xi\|_\infty\le R_r h}}
\left\|
\widehat G_{j,h}(X,\xi)-G_j(X,\xi)
\right\|
\le
\delta_h,
\qquad
\delta_h:=c_Gh^s,
\label{mbot:eq:gradient-network-error}
\end{equation}
where $c_G>0$ is chosen below, and
\begin{equation}
\mathcal{L}(\widehat G_{j,h})
+
\mathcal{M}(\widehat G_{j,h})
\le
C\log(1/h).
\label{mbot:eq:gradient-network-size}
\end{equation}

Starting from $\widehat\xi_j^{(0)}=0$, define the iteration of the approximate gradient descent by 
\begin{equation}
\widehat\xi_j^{(\ell+1)} = \widehat\xi_j^{(\ell)}
- \gamma \widehat G_{j,h} \bigl( X,\widehat\xi_j^{(\ell)} \bigr).
\label{mbot:eq:unrolled-gradient}
\end{equation}
Let $e_\ell :=\| \widehat\xi_j^{(\ell)}-\xi_j^\star \|$.
Note that $\xi_j^\star$ is a fixed point of $T_{j,X}$.  As long as the iterates remain in the local cube,
\eqref{mbot:eq:gradient-contraction} and
\eqref{mbot:eq:gradient-network-error} imply
\begin{equation*}
e_{\ell+1} \le q_{\mathrm{ctr}}e_\ell+\gamma\delta_h.
\label{mbot:eq:iteration-error}
\end{equation*}
Since $e_0=\|\xi_j^\star\|\le C_r h$, induction arguments give
\begin{equation}
e_\ell
\le
q_{\mathrm{ctr}}^\ell C_r h
+
\frac{\gamma\delta_h}{1-q_{\mathrm{ctr}}}.
\label{mbot:eq:iteration-solution}
\end{equation}
By choosing $R_r$ sufficiently large and then $c_G$ in $\delta_h=c_Gh^s$ sufficiently small, we can guarantee that all iterates remain in the local cube.

Let
\begin{equation}
K_h
:=
\left\lceil
\frac{(s-1)\log(1/h)+C_K}
{|\log q_{\mathrm{ctr}}|}
\right\rceil
\label{mbot:eq:number-iterations}
\end{equation}
with $C_K$ sufficiently large. Then, $q_{\mathrm{ctr}}^{K_h}h\le Ch^s$,
and therefore
\begin{equation}
\left\|
\widehat\xi_j^{(K_h)}-\xi_j^\star
\right\|
\le Ch^s.
\label{mbot:eq:coordinate-error}
\end{equation}

By a similar Taylor expansion of the chart $Y_j$ on the cube,
Lemma~\ref{mbot:lem:relu-arithmetic} guarantees the existence of a ReLU network 
$\widehat Y_{j,h}$ that satisfies
\begin{equation}
\sup_{\|\xi\|_\infty\le R_r h}
\left\|
\widehat Y_{j,h}(\xi)-Y_j(\xi)
\right\|
\le Ch^s.
\label{mbot:eq:chart-network-error}
\end{equation}
Define the local expert
\begin{equation}
E_{j,h}(X)
:=
\widehat Y_{j,h}
\bigl(
\widehat\xi_j^{(K_h)}(X)
\bigr).
\label{mbot:eq:local-expert}
\end{equation}
Equations \eqref{mbot:eq:coordinate-error} and
\eqref{mbot:eq:chart-network-error} yield
\begin{equation}
\widetilde{\eta}_j(X)>0
\quad\Longrightarrow\quad
\|E_{j,h}(X)-y\|\le Ch^s.
\label{mbot:eq:expert-error}
\end{equation}

The number of iterations is $K_h=O(\log(1/h))$, and each approximate
gradient evaluation uses $O(\log(1/h))$ nonzero parameters and layers.
Consequently,
\begin{equation}
\mathcal{L}(E_{j,h})
+
\mathcal{M}(E_{j,h})
\le
C\log^2(1/h).
\label{mbot:eq:expert-size}
\end{equation}

When $0<s\le1$, no iterative correction is required. We can simply set $E_{j,h}(X):=y_j$. 
For an active index, \eqref{mbot:eq:active-localization} gives
\[
\|E_{j,h}(X)-y\|\le Ch\le Ch^s.
\]
because $h\in(0,1)$.

Finally, for later use, each expert may be clipped coordinatewise to a fixed box $[-B,B]^d$ that contains $M$. Such clipping is represented exactly using ReLU and
ensures a uniform output bound without increasing the approximation
error on active patches.

\smallskip 
\paragraph*{Step 5: Partition of unity and assembly.}
With $\tilde \eta_j$ in \eqref{mbot:eq:raw-gate}, define 
\begin{equation}
R_h(X):=\sum_{j=1}^{N_h}\widetilde{\eta}_j(X).
\label{mbot:eq:normalizer}
\end{equation}
Since at least one index attains the minimum score, we have $R_h(X)\ge1$. 
By \eqref{mbot:eq:bounded-overlap}, we also have $1\le R_h(X)\le N_{\mathrm{ov}}$ on $\Gamma_{I,\delta_0}^\star$.
We construct a partition of unity $\{\eta_j\}_{j=1}^{N_h}$ by 
\begin{equation}
\eta_j(X) := \frac{\widetilde{\eta}_j(X)}{R_h(X)},
\label{mbot:eq:exact-normalized-gates}
\end{equation}
which satisfies 
\[
\eta_j(X)\ge0, \qquad \sum_{j=1}^{N_h}\eta_j(X)=1.
\]
Moreover, if $\eta_j(X)>0$, then the local expert $E_{j,h}$ satisfies
\eqref{mbot:eq:expert-error}. Hence, the exact mixture
\begin{equation}
\overline\Phi_h(X)
:=
\sum_{j=1}^{N_h}\eta_j(X)E_{j,h}(X)
\label{mbot:eq:exact-mixture}
\end{equation}
obeys
\begin{equation}
\|\overline\Phi_h(X)-y\| \le
\sum_{j=1}^{N_h} \eta_j(X)\|E_{j,h}(X)-y\| \le Ch^s.
\label{mbot:eq:exact-mixture-error}
\end{equation}

The normalization and products are realized by ReLU
networks; by Lemma~\ref{mbot:lem:relu-arithmetic}, approximate
$r\mapsto r^{-1}$ on $[1,N_{\mathrm{ov}}]$ to an accuracy of $c h^s$,
and use zero-preserving multiplication networks to form the weights
and products $\eta_jE_{j,h}$. The zero-preserving property is
important: if $\widetilde{\eta}_j(X)=0$, the corresponding product is
exactly zero. Thus, approximation errors accumulate only over the
at most $N_{\mathrm{ov}}$ active indices, rather than over all
$N_h$ indices.

We choose the arithmetic accuracy to be sufficiently small, which gives a ReLU network $\Phi_h$ to approximate $\overline\Phi_h(X)$.  More explicitly, we choose the reciprocal and multiplication subnetworks
with uniform approximation error at most
$c_{\mathrm{arith}}h^s$.
Since the clipped expert outputs are uniformly bounded, the arithmetic
error in each active term $\eta_j(X)E_{j,h}(X)$ is at most
$C_{\mathrm{arith}}h^s$, where $C_{\mathrm{arith}}>0$
is independent of $h$, $j$, and $X$.
By zero preservation and the overlap bound, the total assembly error
is therefore at most
$C_{\mathrm{arith}}N_{\mathrm{ov}}\varepsilon_h$.
Taking
$0<c_{\mathrm{arith}} \le
\min\{
\frac14,\,
\frac{1}{C_{\mathrm{arith}}N_{\mathrm{ov}}}
\}$ 
gives a ReLU network $\Phi_h$ such that 
\begin{equation}
\sup_{X\in\Gamma_{I,\delta_0}^\star}
\|\Phi_h(X)-\overline\Phi_h(X)\|
\le Ch^s.
\label{mbot:eq:assembly-error}
\end{equation}
Together with \eqref{mbot:eq:exact-mixture-error}, this proves
\eqref{mbot:eq:h-approximation}.  Finally, apply a coordinatewise ReLU clipping
map to the assembled output, with a fixed box containing a unit
neighborhood of $M$.  For sufficiently small $h$ this clipping is inactive
on $\Gamma_{I,\delta_0}^\star$, while it gives the global bound
\eqref{mbot:eq:Phi-global-bound}.  This final operation has constant depth and
size.

The $N_h$ experts are implemented in parallel. Each requires
$O(\log^2(1/h))$ nonzero parameters. The score network, minimum tree,
raw gates, summation, reciprocal, and multiplication networks require
at most
\[
C N_h\log^2(1/h)
\]
additional nonzero parameters. By \eqref{mbot:eq:covering-number},
\[
\mathcal{M}(\Phi_h)
\le
C h^{-(m+1)}\log^2(1/h),
\]
which proves the bound of $\mathcal M$ in \eqref{mbot:eq:P-complexity}. The parallel construction and the local iterations give
\[
\mathcal{L}(\Phi_h)\le C\log^2(1/h).
\]
The factors $h^{-2}$ in the gates and the accuracy-dependent arithmetic
networks give a polynomial parameter envelope, proving
\eqref{mbot:eq:h-depth-weight}.

Finally, for sufficiently large $P$, choose
\begin{equation}
h_P := c_P \left( \frac{P}{\log^2P} \right)^{-1/(m+1)}
\label{mbot:eq:h-choice}
\end{equation}
with $c_P>0$ sufficiently small. Then $\mathcal{M}(\Phi_{h_P})\le P$, and
\[
h_P^s \le C \left( \frac{P}{\log^2P} \right)^{-s/(m+1)}.
\]
This proves \eqref{mbot:eq:P-intrinsic-rate}.
\end{proof}

\bigskip 

\paragraph{Approximation of the population OT vector field}

On the exact interpolation graph,
\begin{equation}
u^\star(t,x)
=
\frac{\Pi_{\mathrm{OT}}(t,x)-x}{1-t}.
\label{mbot:eq:field-from-Pi}
\end{equation}
Since $t\in[0,c]$ with $c<1$, the factor $(1-t)^{-1}$ is uniformly
bounded.

\begin{corollary}
\label{mbot:cor:intrinsic-field-approximation}
Under the assumptions of
Theorem~\ref{mbot:thm:intrinsic-Pi-approximation}, for every sufficiently
large $P$ there is a ReLU network
\[
f_P:I\times\mathbb{R}^d\to\mathbb{R}^d
\]
with
\[
\mathcal{M}(f_P)\le P,
\qquad
\mathcal{L}(f_P)\le C\log^2P,
\qquad
\mathcal{B}(f_P)\le P^C,
\]
such that
\begin{equation}
\sup_{(t,x)\in\Gamma_I^\star}
\|f_P(t,x)-u^\star(t,x)\|
\le
C_c
\left(
\frac{P}{\log^2P}
\right)^{-s/(m+1)}.
\label{mbot:eq:intrinsic-field-sup-rate}
\end{equation}
Consequently,
\begin{equation}
\|f_P-u^\star\|_{\star,I}
\le
C_c
\left(
\frac{P}{\log^2P}
\right)^{-s/(m+1)}.
\label{mbot:eq:intrinsic-field-L2-rate}
\end{equation}
\end{corollary}

\begin{proof}
Let $\Phi_P$ be given by
Theorem~\ref{mbot:thm:intrinsic-Pi-approximation}. A one-dimensional
ReLU network can approximate
\[
t\mapsto\frac{1}{1-t},
\qquad t\in[0,c],
\]
to the same accuracy using Lemma~\ref{mbot:lem:relu-arithmetic}. Combining
this reciprocal network with multiplication networks and the affine
map $(t,x)\mapsto\Phi_P(t,x)-x$ yields $f_P$.

Since $(1-t)^{-1}\le(1-c)^{-1}$, the error in
\eqref{mbot:eq:intrinsic-field-sup-rate} follows from
\eqref{mbot:eq:P-intrinsic-rate}. The additional arithmetic networks have
only polylogarithmic complexity and therefore do not change the
parameter exponent.

Finally, $\mu_t^\star$ is supported on
$\{Z_t(x_0):x_0\in K_0\}$. Hence
\[
\begin{aligned}
\|f_P-u^\star\|_{\star,I}^2
&=
\int_I
\int
\|f_P(t,x)-u^\star(t,x)\|^2
\,\mu_t^\star(dx)\,dt
\\
&\le
|I|
\sup_{(t,x)\in\Gamma_I^\star}
\|f_P(t,x)-u^\star(t,x)\|^2.
\end{aligned}
\]
\end{proof}


\begin{corollary}
\label{mbot:cor:intrinsic-batch-approximation}
Let $a_P(B)$ be defined in Theorem \ref{mbot:thm:batch-relu-oracle-zero}. 
Under Assumption~\ref{mbot:ass:smooth-proximal-geometry}, for all sufficiently
large $P$,
\begin{equation}
a_P(B) \le C_c \left( \frac{P}{\log^2P} \right)^{-s/(m+1)}
+ C_c\varepsilon_B. 
\label{mbot:eq:batch-intrinsic-epsilon}
\end{equation}
Consequently, Proposition~\ref{mbot:prop:automatic-detachment} gives
\begin{equation}
a_P(B) \le C_c \left( \frac{P}{\log^2P} \right)^{-s/(m+1)}
+ C_{c,L_S}\Delta_B^{1/2}.
\label{mbot:eq:batch-intrinsic-approx}
\end{equation}
\end{corollary}

\begin{proof}
Let
\[
r_P
:=
\left(
\frac{P}{\log^2P}
\right)^{-s/(m+1)},
\]
and let $\Phi_P$ be the ReLU network that approximates the endpoint in 
Theorem~\ref{mbot:thm:intrinsic-Pi-approximation}.  Draw
$(X_0,X_1)\sim\pi_B$ and set
\begin{equation}
X_t^B:=(1-t)X_0+tX_1,
\qquad
X_t^\star:=(1-t)X_0+tS(X_0),
\qquad
D:=\|X_1-S(X_0)\|.
\label{mbot:eq:batch-and-exact-points-intrinsic}
\end{equation}
Then
\begin{equation}
X_t^B-X_t^\star=t\{X_1-S(X_0)\}.
\label{mbot:eq:batch-ambient-perturbation}
\end{equation}
Recall also that $\delta_0$ is a constant to define $\Gamma_{I,\delta_0}^\star$.  

Suppose first that $\|X_t^B-X_t^\star\|=tD\le\delta_0$.  Lemma~\ref{mbot:lem:uniform-tube-stability}
and Theorem~\ref{mbot:thm:intrinsic-Pi-approximation} imply
\begin{equation}
\begin{aligned}
\|\Phi_P(t,X_t^B)-S(X_0)\|
&\le
\|\Phi_P(t,X_t^B)-\Pi_{\mathrm{OT}}(t,X_t^B)\|
+ \|\Pi_{\mathrm{OT}}(t,X_t^B)-S(X_0)\|
\\
&\le Cr_P+CtD.
\end{aligned}
\label{mbot:eq:selector-off-graph-bound}
\end{equation}
For the second term, \eqref{mbot:eq:tube-selector-stability} is applied from $\|X_t^B-X_t^\star\|\leq \delta_0$.
For $t>0$, Proposition~\ref{mbot:prop:exact-field-zero-regularity} and the exact identity
$Q_t(X_t^\star)=S(X_0)$ give
\begin{equation}
\|Q_t(X_t^B)-S(X_0)\|
\le
\frac1t\|X_t^B-X_t^\star\|
\le D.
\label{mbot:eq:prox-common-source-bound}
\end{equation}
At $t=0$ the corresponding identity is understood through the
on-support definition $u^\star(0,X_0)=S(X_0)-X_0$.

Second, Suppose $tD>\delta_0$.  The global clipping
\eqref{mbot:eq:Phi-global-bound} shows $\|\Phi_P(t,X_t^B)\|\leq R_\Phi$. Since \eqref{mbot:eq:prox-common-source-bound} always holds, it follows from the compactness of $K_1=M$ that $\|Q_t(X_t^B)\|=\|Q_t(X_t^B)-S(X_0)\|+\|S(X_0)\|$ is bounded.  Therefore, we have a constant $C>0$ such that $\|\Phi_P(t,X_t^B)-Q_t(X_t^B)\|\le C$.  Since
$t\le c$ and $tD>\delta_0$ implies $1<D/c\delta_0$, we can replace $C$ with $\frac{C}{c\delta_0}D$. 

Combining the two cases yields, for almost every $t\in I$,
\begin{equation}
\|\Phi_P(t,X_t^B)-Q_t(X_t^B)\|
\le
Cr_P+CD.
\label{mbot:eq:endpoint-batch-pointwise}
\end{equation}

Construct $f_P$ using $(\Phi_P - x)/(1-t)$ with arithmetic
accuracy $O(r_P)$, and clip its output coordinatewise to a fixed
box containing the range of $u^\star$ on $\mathcal X_I$.  Since
$(1-t)^{-1}\le(1-c)^{-1}$, the arithmetic error and
\eqref{mbot:eq:endpoint-batch-pointwise} give
\begin{equation}
\|f_P(t,X_t^B)-u^\star(t,X_t^B)\|
\le
C_c(r_P+D)
\label{mbot:eq:field-batch-pointwise}
\end{equation}
for almost every $t$.  Integration with respect to
$dt\,\pi_B(dX_0,dX_1)$ and Minkowski's inequality give
\[
\|f_P-u^\star\|_{B,I}
\le
C_c r_P+C_c\varepsilon_B.
\]
This proves \eqref{mbot:eq:batch-intrinsic-epsilon}; the bound
\eqref{mbot:eq:batch-intrinsic-approx} follows from
$\varepsilon_B^2\le2L_S\Delta_B$.
\end{proof}


\begin{remark}[Regularity assumptions for the learned field]
The preceding approximation error bounds do not require Lipschitz
constants of the comparison networks to be bounded uniformly in the
network size.
Moreover, these approximation results do not provide uniform Lipschitz
or one-sided Lipschitz (OSL) control of the learned vector field.
The regularity conditions needed for the subsequent prediction-error
comparison and flow-stability analysis are imposed separately.
\end{remark}

\begin{remark}[No second derivative of the ReLU network is used]
The Hessian bounds used in the local construction concern the smooth
manifold objective $j_{j,X}=J_X\circ Y_j$, not the approximating ReLU
networks.
The neural approximation argument requires only function-value
approximation of the smooth maps involved in the construction;
it does not require differentiating their ReLU realizations.
Thus the piecewise-affine regularity of ReLU is compatible with the
proof, and no second derivatives of the ReLU networks are used.
\end{remark}

\begin{remark}[Role of the OT fibers]
The construction never explicitly approximates the sets
\[
C_y
=
\bigl(\partial\phi(y)\cap K_0\bigr)
-
\nabla_M(\phi|_M)(y)
\subset N_yM.
\]
In particular, it does not assume that $0\in C_y$, nor that the
subgradient multifunction admits a smooth global section. The OT
fiber enters only through the fact that, for an exact OT point
\[
x=(1-t)x_0+ty,
\qquad
x_0\in\partial\phi(y),
\]
the point $y$ is the uniformly stable minimizer of
$\mathcal{J}_{(t,x)}$.

This is the reason that the normal dimension $d-m$ does not enter the
covering number. The only mesh is placed on
\[
I\times M,
\]
whose dimension is $m+1$. The ambient input $x\in\mathbb{R}^d$ enters
the local stationarity equation
\begin{equation}
(1-t)\nabla A_j(\xi)
+
DY_j(\xi)^\top
\{tY_j(\xi)-x\}
=
0.
\label{mbot:eq:local-stationarity-role-of-fibers}
\end{equation}
affinely and is processed by the same local expert for all points in
the corresponding normal OT fiber.
\end{remark}


\subsection{Prediction-error comparison between expected-batch and population OT interpolations}
\label{mbot:sec:population-law-transfer}

The regression bound in Section~\ref{mbot:sec:neural-regression} measures the
prediction error under the expected-batch interpolation distribution
$\nu_{B,t}$.  The flow analysis in
Section~\ref{mbot:sec:strong-flow-convergence}, however, requires the error along
population OT paths, whose distribution is $\mu_t^\star$.  We compare these
two errors by coupling the paths through the same source point.  The
comparison inequalities are independent of the approximation
theorem.  We then use the intrinsic-dimensional approximation from
Section~\ref{mbot:sec:intrinsic-approximation} to obtain the neural rate in
Corollary~\ref{mbot:cor:intrinsic-field-exact-law}.

The quantitative comparison below uses spatial Lipschitz control of the
learned field, whereas the subsequent ODE stability estimate uses
one-sided Lipschitz control.  Neither the regression bound nor the
intrinsic approximation result provides uniform control of these
learned-field quantities as $P$ increases; the regularity assumptions are
therefore stated separately.

Throughout this section, we retain the compact-support and forward-map
conditions in Assumptions~\ref{mbot:eq:compact-endpoint-supports} and
\ref{mbot:ass:graph-restricted-forward-regularity}, together with the Bregman
co-coercivity condition on the source support in
Assumption~\ref{mbot:ass:graph-restricted-bregman}.  Let $I=[0,c]$ as before, where $c<1$ is
fixed, and set
\[
\mathcal X_t
:=\{(1-t)x_0+tx_1:x_0\in K_0,\ x_1\in K_1\}.
\]
Recall the definitions in \eqref{mbot:eq:DeltaB} and \eqref{mbot:eq:dB}:
\[
\Delta_B
=\frac12\left\{
\mathbb E_B W_2^2(\widehat P_{0,B},\widehat P_{1,B})
-W_2^2(P_0,P_1)\right\},
\qquad
d_B=W_2(\barpi,\pistar).
\]
Also define
\begin{equation}
D_{B,I}^2:=\int_I W_2^2(\nu_{B,t},\mu_t^\star)\dd t.
\label{mbot:eq:DBI}
\end{equation}

\begin{lemma}
\label{mbot:lem:path-law-discrepancy-zero}
\begin{equation}
D_{B,I}
\le\min\left\{\sqrt c\,d_B,
\sqrt{\frac{2L_Sc^3}{3}}\,\Delta_B^{1/2}\right\}.
\label{mbot:eq:DBI-bound}
\end{equation}
\end{lemma}

\begin{proof}
The map $F_t(x_0,x_1)=(1-t)x_0+tx_1$ has Lipschitz constant
$\sqrt{(1-t)^2+t^2}$ for the Euclidean metric.  Because $\nu_{B,t}=F_{t,\#} \barpi$ and $\mu_t^\star = F_{t,\#}\pistar$, we have 
\[
W_2(\nu_{B,t},\mu_t^\star) \le \sqrt{(1-t)^2+t^2}\,d_B.
\]
Next, consider $(X_0,X_1)\sim\barpi$ and couple
$F_t(X_0,X_1)$ with $F_t(X_0,S(X_0))$.  The latter has distribution
$\mu_t^\star$ because the first marginal of $\barpi$ is $P_0$.
Their difference is $t\{X_1-S(X_0)\}$, so the squared transportation cost
of this coupling is $t^2\varepsilon_B^2$.  Then, we have 
\[
W_2(\nu_{B,t},\mu_t^\star) \le t\varepsilon_B
\le t\sqrt{2L_S\Delta_B}.
\]
The integral over $[0,c]$ gives the assertion.
\end{proof}

\begin{lemma}
\label{mbot:lem:common-source-transfer}
Let $(X_0,X_1)\sim\barpi$, and define
\[
X_t^B:=(1-t)X_0+tX_1,
\qquad
X_t^\star:=(1-t)X_0+tS(X_0).
\]
For every measurable map $h:I\times\mathbb R^d\to\mathbb R^q$,
\begin{equation}
\|h\|_{\star,I}
\le\|h\|_{B,I}
+\left(\int_I\mathbb E_{\barpi}
\bigl[\|h(t,X_t^\star)-h(t,X_t^B)\|^2\bigr]\dd t\right)^{1/2}.
\label{mbot:eq:common-source-transfer-basic}
\end{equation}
Suppose, in addition, that a measurable function
$L_h:I\to[0,\infty)$ satisfies
\begin{equation}
\|h(t,X_t^\star)-h(t,X_t^B)\|
\le L_h(t)\|X_t^\star-X_t^B\|
\qquad\barpi\text{-a.s.}
\label{mbot:eq:common-source-local-lipschitz}
\end{equation}
for almost every $t\in I$.  Then, 
\begin{equation}
\|h\|_{\star,I}
\le\|h\|_{B,I}
+\sqrt{2L_S\Delta_B}
\left(\int_I t^2L_h(t)^2\dd t\right)^{1/2}.
\label{mbot:eq:common-source-transfer-cost}
\end{equation}
\end{lemma}

\begin{proof}
From $X_t^B\sim\nu_{B,t}$ and $X_t^\star\sim\mu_t^\star$ with $(X_0,X_1)\sim\barpi$, the application of the triangular inequality in $L^2(I\times\barpi)$ to
\[
h(t,X_t^\star)
=h(t,X_t^B)+\{h(t,X_t^\star)-h(t,X_t^B)\}
\]
proves \eqref{mbot:eq:common-source-transfer-basic}.  Next,
the relation $X_t^\star-X_t^B=t\{S(X_0)-X_1\}$ and
\eqref{mbot:eq:common-source-local-lipschitz} give
\begin{align*}
&\left(\int_I\mathbb E_{\barpi}
\|h(t,X_t^\star)-h(t,X_t^B)\|^2\dd t\right)^{1/2}
\\
&\qquad\le
\left(\int_I t^2L_h(t)^2
\mathbb E_{\barpi}\|X_1-S(X_0)\|^2\dd t\right)^{1/2}
=\varepsilon_B\left(\int_I t^2L_h(t)^2\dd t\right)^{1/2}.
\end{align*}
This proves 
\eqref{mbot:eq:common-source-transfer-cost} from
\eqref{mbot:eq:epsilon-from-Delta}.
\end{proof}

\begin{theorem}
\label{mbot:thm:exact-law-field-bound}
Suppose Assumption~\ref{mbot:ass:compact-neural-regression} holds, and let
$\widehat u_{B,L,P}\in\mathcal F_P$ satisfy the approximate empirical
minimization condition \eqref{mbot:eq:approx-erm-zero}, with $L\ge2$.
Set $h_{B,L,P}:=\widehat u_{B,L,P}-u^\star$ and define the common-source
path-comparison term
\begin{equation}
\mathfrak T_{B,L,P}
:=\left(\int_I\mathbb E_{\barpi}
\bigl[\|h_{B,L,P}(t,X_t^\star)-h_{B,L,P}(t,X_t^B)\|^2\bigr]
\dd t\right)^{1/2}.
\label{mbot:eq:common-source-transfer-increment}
\end{equation}
Here the evaluation pair $(X_0,X_1)\sim\barpi$ is independent of the
training randomness, and the expectation is taken with the learned
network held fixed.  For every $\eta\in(0,1)$, with probability at least
$1-\eta$ over the training batches,
\begin{equation}
\begin{aligned}
\|\widehat u_{B,L,P}-u^\star\|_{\star,I}
\le{}&C a_P(B)+C\sqrt{\kappa_{S,I}\Delta_B}
\\
&+C K_I\sqrt{\frac{
P\log^2(P+1)\log(C_0A_PPL)+\log(2/\eta)}{L}}
\\
&+C\sqrt{\varepsilon_{\mathrm{opt}}}+\mathfrak T_{B,L,P}.
\end{aligned}
\label{mbot:eq:exact-law-field-bound}
\end{equation}
The constants are independent of $B,L,P,\eta$.
\end{theorem}

\begin{proof}
For each training realization, apply
Lemma~\ref{mbot:lem:common-source-transfer} with
$h=\widehat u_{B,L,P}-u^\star$ to obtain
\[
\|\widehat u_{B,L,P}-u^\star\|_{\star,I}
\le\|\widehat u_{B,L,P}-u^\star\|_{B,I}
+\mathfrak T_{B,L,P}.
\]
Theorem~\ref{mbot:thm:batch-relu-oracle-zero} bounds the first term of the R.H.S.~on the
same training event of probability at least $1-\eta$.
\end{proof}

\begin{theorem}
\label{mbot:thm:exact-law-field-bound-lipschitz}
Suppose Assumption~\ref{mbot:ass:compact-neural-regression} holds.
Assume further that there is a deterministic measurable function
$L_P:I\to[0,\infty)$ such that, for almost every $t\in I$,
\begin{equation}
\|\widehat u_{B,L,P}(t,x)-\widehat u_{B,L,P}(t,y)\|
\le L_P(t)\|x-y\|,
\qquad x,y\in\mathcal X_t
\label{mbot:eq:localized-learned-lipschitz}
\end{equation}
holds for every realization of the training procedure.  Define
\begin{equation}
\Lambda_P^{\mathrm{tr}}(I):=\left(\int_0^c t^2\{L_P(t)+\ell_t\}^2\dd t\right)^{1/2}, \qquad 
\ell_t:=\max\left\{\frac1t,\frac1{1-t}\right\} \qquad (t>0).
\label{mbot:eq:localized-transfer-budget}
\end{equation}
If $\Lambda_P^{\mathrm{tr}}(I)<\infty$, then, for every $\eta\in(0,1)$,
with probability at least $1-\eta$,
\begin{equation}
\begin{aligned}
\|\widehat u_{B,L,P}-u^\star\|_{\star,I}
\le{}&C a_P(B)+C\sqrt{\kappa_{S,I}\Delta_B}
+C K_I\sqrt{\frac{
P\log^2(P+1)\log(C_0A_PPL)+\log(2/\eta)}{L}}
\\
&+C\sqrt{\varepsilon_{\mathrm{opt}}}
+\sqrt{2L_S\Delta_B}\,\Lambda_P^{\mathrm{tr}}(I).
\end{aligned}
\label{mbot:eq:localized-exact-law-bound}
\end{equation}
\end{theorem}

\begin{proof}
For $h=\widehat u_{B,L,P}-u^\star$, condition
\eqref{mbot:eq:localized-learned-lipschitz} and
Proposition~\ref{mbot:prop:exact-field-zero-regularity} imply
\[
\|h(t,x)-h(t,y)\|
\le\{L_P(t)+\ell_t\}\|x-y\|,
\qquad x,y\in\mathcal X_t,
\]
for almost every $t>0$.  Both common-source interpolation points belong
to $\mathcal X_t$.  Therefore Lemma~\ref{mbot:lem:common-source-transfer}
gives
\[
\mathfrak T_{B,L,P}
\le\varepsilon_B\Lambda_P^{\mathrm{tr}}(I)
\le\sqrt{2L_S\Delta_B}\,\Lambda_P^{\mathrm{tr}}(I).
\]
Substitute this estimate into \eqref{mbot:eq:exact-law-field-bound}.
The value of $\ell_t$ at $t=0$ is irrelevant to the integral.  In
particular, $t\ell_t$ remains bounded on $(0,c]$, so no positive initial-time
cutoff is needed.
\end{proof}

\begin{remark}\label{mbot:rem:Lipschitz_condition}
Condition \eqref{mbot:eq:localized-learned-lipschitz} concerns the learned
neural network.
Bounds on weights or spectral normalization can provide such Lipschitz control.
Restricting the network class by a Lipschitz, spectral, or OSL constraint
cannot increase its covering entropy.  If training is instead restricted to a
smaller Lipschitz-constrained class, the oracle inequality can use the
approximation error of that smaller class.  The intrinsic approximation
bound may not be carried over to such a restriction, however.  To the best of our knowledge, including the Lipschitz condition to the approximation class remain open.  Although \cite{Oko2023} and \cite{Fukumizu2024-ml} use a partition of the time interval to avoid the effect of Lipschitz constants in bounding the flow errors, the technique requires separate training of many neural networks depending on the time domains, which is not realistic in practical learning. 
\end{remark}

\begin{corollary}
\label{mbot:cor:intrinsic-field-exact-law}
\label{mbot:cor:smooth-field-exact-law}
Suppose that Assumptions~\ref{mbot:ass:compact-neural-regression}, 
\ref{mbot:ass:smooth-proximal-geometry}, and the Lipschitz condition \eqref{mbot:eq:localized-learned-lipschitz} hold. Then, 
for every sufficiently large $P$, every $L\ge2$, and every
$\eta\in(0,1)$, with probability at least $1-\eta$,
\begin{equation}
\begin{aligned}
\|\widehat u_{B,L,P}-u^\star\|_{\star,I}
\le{}&C\Delta_B^{1/2}
+C\left(\frac{P}{\log^2(P+1)}\right)^{-s/(m+1)}
\\
&+C K_I\sqrt{\frac{
P\log^2(P+1)\log(C_0A_PPL)+\log(2/\eta)}{L}}
\\
&+C\sqrt{\varepsilon_{\mathrm{opt}}}
+\sqrt{2L_S\Delta_B}\,\Lambda_P^{\mathrm{tr}}(I).
\end{aligned}
\label{mbot:eq:intrinsic-field-exact-law-unbalanced}
\end{equation}
The constant $C$ may depend on the fixed geometric and model parameters,
including $c,L_S,d,m,s$, but not on $B,L,P,\eta$.

Moreover, by the choice $P_L^{\mathrm{int}}
\asymp (\log(eL))^2
\left(\frac{L}{(\log(eL))^5}\right)^{(m+1)/(2s+m+1)}$, 
for all sufficiently large $L$, with probability at least $1-\eta$,
\begin{equation}
\begin{aligned}
\|\widehat u_{B,L,P_L^{\mathrm{int}}}-u^\star\|_{\star,I}
& \le C\Delta_B^{1/2}
+C\left(\frac{\log^5(eL)}{L}\right)^{s/(2s+m+1)}   \\
& \qquad +C\sqrt{\frac{\log(2/\eta)}{L}}
+C\sqrt{\varepsilon_{\mathrm{opt}}}
+\sqrt{2L_S\Delta_B}\,
\Lambda_{P_L^{\mathrm{int}}}^{\mathrm{tr}}(I).
\end{aligned}
\label{mbot:eq:smooth-field-exact-law-bound}
\end{equation}
In particular, the neural approximation-estimation contribution is
$\widetilde O(L^{-s/(2s+m+1)})$.
\end{corollary}

\begin{proof}
Corollary~\ref{mbot:cor:intrinsic-batch-approximation} and 
\eqref{mbot:eq:localized-exact-law-bound} prove 
\eqref{mbot:eq:intrinsic-field-exact-law-unbalanced}.

Let $H_L:=\log(eL)$.  For the parameter choice, put $k:=m+1$ and
$N_L:=(L/H_L^5)^{k/(2s+k)}$.  Then
$P_L^{\mathrm{int}}\asymp H_L^2N_L$ and
$\log(P_L^{\mathrm{int}}+1)\asymp H_L$.  Consequently,
\[
\left(\frac{P_L^{\mathrm{int}}}
{\log^2(P_L^{\mathrm{int}}+1)}\right)^{-s/k}
\asymp N_L^{-s/k}
=\left(\frac{H_L^5}{L}\right)^{s/(2s+k)}.
\]
The polynomial weight envelope and the fact that
$P_L^{\mathrm{int}}$ is polynomial in $L$ up to logarithmic factors give
\[
\log(C_0A_{P_L^{\mathrm{int}}}P_L^{\mathrm{int}}L)\le C H_L.
\]
Thus the non-confidence part of the root estimation error is at most
\[
C\left(\frac{P_L^{\mathrm{int}}H_L^3}{L}\right)^{1/2}
\asymp C\left(\frac{N_LH_L^5}{L}\right)^{1/2}
=C\left(\frac{H_L^5}{L}\right)^{s/(2s+k)}.
\]
Splitting off the confidence term proves
\eqref{mbot:eq:smooth-field-exact-law-bound}.
\end{proof}

The exponent $m+1$ comes from the structured approximation in
Section~\ref{mbot:sec:intrinsic-approximation}.  The logarithmic factor accounts for the
$\log^2 P$ approximation overhead in \eqref{mbot:eq:batch-intrinsic-approx}.
The power $5$ is the sufficient logarithmic factor supplied by the present
approximation and entropy bounds; no optimality of that factor is claimed.
To absorb the last term of
\eqref{mbot:eq:smooth-field-exact-law-bound} into $C\Delta_B^{1/2}$, one
still needs a uniform bound on
$\Lambda_{P_L^{\mathrm{int}}}^{\mathrm{tr}}(I)$, which is relevant to the Lipschitz condition.

\subsection{Flow convergence rate}
\label{mbot:sec:strong-flow-convergence}

Let $\mathcal F_{\mathrm{tr}}$ be the sigma-field generated by the random batches for training, sampled time variables, and any algorithmic randomness.
Independently of $\mathcal F_{\mathrm{tr}}$, draw
$(X_0,X_1)\sim\pistar$ and write $\E_\star$ for expectation only over this
evaluation pair.  Define
\[
       X_t^\star=(1-t)X_0+tX_1,
\]
and let the ODE for generation 
\begin{equation}
 \frac{d\widehat X_t}{dt}=\widehat u_{B,L,P}(t,\widehat X_t),
 \qquad
 \widehat X_0=X_0^\star=X_0.
 \label{mbot:eq:learned-flow-zero}
\end{equation}

\subsubsection{From vector field bound to flow bound}

\begin{assumption}[One-sided Lipschitz (OSL) condition for the learned field]
\label{mbot:ass:learned-osl-zero}
The map $t\mapsto\rho_{\widehat u_{B,L,P}}(t)$ in \eqref{mbot:eq:OSL-definition} is measurable and
\begin{equation}
 R^{\mathrm{OSL}}_{B,L,P}(I)
 :=\int_I
 (\rho_{\widehat u_{B,L,P}}(t))_+\dd t<\infty.
 \label{mbot:eq:R-osl-zero}
\end{equation}
\end{assumption}

Note that, if the global spatial Lipschitz condition on
$\mathbb R^d$ 
\begin{equation}
 \sup_{t\in I}
 \operatorname{Lip}_x\widehat u_{B,L,P}(t,\cdot)\le L_P 
 \label{mbot:eq:learned-lip-zero}
\end{equation}
holds, then the OSL holds:
\begin{equation}
       R^{\mathrm{OSL}}_{B,L,P}(I)\le cL_P.
       \label{mbot:eq:L-sufficient-OSL-zero}
\end{equation}
Thus, when Sections~\ref{mbot:sec:neural-regression}--\ref{mbot:sec:strong-flow-convergence} are combined, the OSL condition is not an additional
regularity requirement if the ordinary Lipschitz bound is imposed pathwise.  It is stated separately because it is strictly weaker and can yield a sharper
stability factor.
In contrast, \eqref{mbot:eq:localized-learned-lipschitz} only controls the
field on $\mathcal X_t$ and does not by itself imply the global OSL
condition above.

\begin{theorem}
\label{mbot:thm:flow-stability-zero}
Under Assumption~\ref{mbot:ass:learned-osl-zero}, conditionally on
$\mathcal F_{\mathrm{tr}}$,
\begin{equation}
 \left(\E_\star\sup_{t\in I}
 \|\widehat X_t-X_t^\star\|^2\right)^{1/2}
 \le
 e^{R^{\mathrm{OSL}}_{B,L,P}(I)}\sqrt c\,
 \|\widehat u_{B,L,P}-u^\star\|_{\star,I}.
 \label{mbot:eq:strong-flow-bound-zero}
\end{equation}
Let $\widehat\mu_c :=
\operatorname{Law}_\star (\widehat X_c\mid\mathcal F_{\mathrm{tr}})$
denote the conditional distribution of the generated state
$\widehat X_c$ given the training sigma-field
$\mathcal F_{\mathrm{tr}}$.  Then
\begin{equation}
       W_2(\widehat\mu_c,\mu_c^\star)
       \le\left(\E_\star\|\widehat X_c-X_c^\star\|^2\right)^{1/2}
       \label{mbot:eq:pushforward-flow-zero}
\end{equation}
and is bounded by the right-hand side of
\eqref{mbot:eq:strong-flow-bound-zero}.
\end{theorem}

\begin{proof}
Let $ e(t)=\|\widehat X_t-X_t^\star\|$.  
Using the OSL coefficient of the learned field,
\[
       D^+e(t)\le\rho_{\widehat u_{B,L,P}}(t)e(t)+\|\widehat u_{B,L,P}(t,X_t^\star)-u^\star(t,X_t^\star)\|,
       \qquad e(0)=0.
\]
Here \(D^+\) denotes the upper right Dini derivative.
Variation of constants and \eqref{mbot:eq:R-osl-zero} give
\[
       e(t)\le e^{R^{\mathrm{OSL}}_{B,L,P}(I)}
       \int_0^t\|\widehat u_{B,L,P}(t,X_t^\star)-u^\star(t,X_t^\star)\|\dd s.
\]
Taking the supremum and applying Cauchy--Schwarz yields
\[
       \sup_{t\in I}e(t)^2
       \le e^{2R^{\mathrm{OSL}}_{B,L,P}(I)}c
       \int_I\|\widehat u_{B,L,P}(t,X_t^\star)-u^\star(t,X_t^\star)\|^2\dd s.
\]
Taking $\E_\star$ completes the proof.
\end{proof}

\begin{corollary}[Finite-sample end-to-end bound on $I$]
\label{mbot:cor:main-end-to-end-zero}
Suppose the assumptions of
Theorems~\ref{mbot:thm:exact-law-field-bound-lipschitz} and
\ref{mbot:thm:flow-stability-zero} hold and
\begin{equation}
       R^{\mathrm{OSL}}_{B,L,P}(I)\le R
       \label{mbot:eq:uniform-osl-budget-zero}
\end{equation}
for a deterministic $R<\infty$.  Then, with probability at least $1-\eta$,
\begin{align}
 &\left(\E_\star\sup_{t\in I}
 \|\widehat X_t-X_t^\star\|^2\right)^{1/2}
 \notag\\
 &\quad\le C_{c,R}\Bigg[
 a_P(B)+\sqrt{\kappa_{S,I}\Delta_B}
 +K_I\sqrt{\frac{P\log^2(P+1)\log(C_0A_PPL)+\log(2/\eta)}{L}}
 +\sqrt{\varepsilon_{\mathrm{opt}}}
 \notag\\
 &\hspace{39mm}
 +\sqrt{2L_S\Delta_B}\,
 \Lambda_P^{\mathrm{tr}}(I)
 \Bigg],
 \label{mbot:eq:main-generic-rate-zero}
\end{align}
where $C_{c,R}=Ce^R\sqrt c$.
\end{corollary}


\subsubsection{Quantitative batch-size rates}
\label{mbot:sec:batch-size-rate}



In this subsection, we discuss the convergence rate of empirical OT of batch size $B$ to the population one, that is, the rate of $\Delta_B$.

Define
\begin{equation}
 \Psi_m(B):=
 \begin{cases}
 B^{-1/2},&m<4,\\
 B^{-1/2}\log(B+1),&m=4,\\
 B^{-2/m},&m>4.
 \end{cases}
 \label{mbot:eq:Psi-m-zero}
\end{equation}
When the compact-support condition of Assumption~\ref{mbot:ass:compact-neural-regression} holds, lower-complexity
adaptation results imply, under the corresponding metric-entropy hypotheses,
that when one marginal is supported on an $m$-dimensional regular compact
manifold,
\begin{equation}
 \E\left|W_2^2(\widehat P_{0,B},\widehat P_{1,B})
       -W_2^2(P_0,P_1)\right|
 \le C\Psi_m(B).
 \label{mbot:eq:known-cost-rate-zero}
\end{equation}
See \cite{HundrieserStaudtMunk2024}.  Hence
\begin{equation}
       \Delta_B\le C\Psi_m(B).
       \label{mbot:eq:Delta-Psi-zero}
\end{equation}
Under Assumption~\ref{mbot:ass:graph-restricted-forward-regularity}, the batch contribution
to the root field and strong-flow errors on $I$ is therefore
\begin{equation}
 O(\Psi_m(B)^{1/2})
 =O\!\left(
 \begin{cases}
 B^{-1/4},&m<4,\\
 B^{-1/4}\sqrt{\log(B+1)},&m=4,\\
 B^{-1/m},&m>4.
 \end{cases}\right).
 \label{mbot:eq:regular-manifold-batch-rate}
\end{equation}

\begin{remark}
Extensions of empirical OT-cost rates to unbounded distributions require
explicit moment and tail assumptions and may contain additional tail terms;
see \cite{StaudtHundrieser2024}.  
\end{remark}

\begin{remark}
The rate above is a corollary within the standing globally regular
forward-map regime (Assumption \ref{mbot:ass:graph-restricted-forward-regularity}).  Singular transports such as Gaussian-to-circle (Remark \ref{mbot:rem:singular-forward-map-excluded}) and
Gaussian-to-discrete OT do not satisfy this assumption
and are not asserted to obey the present initial-time field or flow rate.
\end{remark}

\subsubsection{Main result: Explicit intrinsic-dimensional rates of flow}
\label{mbot:sec:explicit-end-to-end-rates}

We now combine the structured approximation in
Section~\ref{mbot:sec:intrinsic-approximation}, the prediction-error comparison
in Corollary~\ref{mbot:cor:intrinsic-field-exact-law}, and the flow stability
estimate in Theorem~\ref{mbot:thm:flow-stability-zero}.  In this section,
$m=\dim M$ and $I=[0,c]$ with $c<1$ fixed.  The resulting neural rate
depends on the intrinsic space--time dimension $m+1$.  
As in Section~\ref{mbot:sec:strong-flow-convergence}, $\E_\star$ denotes
expectation over the independent population evaluation pair, with the
training realization held fixed.

\begin{corollary}[Main result]
\label{mbot:cor:intrinsic-fast-zero}
\label{mbot:cor:generic-fast-zero}
Suppose that Assumptions~\ref{mbot:ass:compact-neural-regression}, 
\ref{mbot:ass:smooth-proximal-geometry}, \ref{mbot:ass:learned-osl-zero}, and the Lipschitz condition \eqref{mbot:eq:localized-learned-lipschitz} hold. 
Choose the integer parameter budget $P=P_L^{\mathrm{int}}$ as 
\begin{equation}
    P_L^{\mathrm{int}}
\asymp \log^2(eL)
\left(\frac{L}{\log^5(eL)}\right)^{(m+1)/(2s+m+1)}.
\label{mbot:eq:optimal-intrinsic-parameter}
\end{equation}
Suppose further that, uniformly in the asymptotic regime and for every
training realization,
\begin{equation}
\Lambda_{P_L^{\mathrm{int}}}^{\mathrm{tr}}(I)\le\Lambda,
\qquad
R^{\mathrm{OSL}}_{B,L,P_L^{\mathrm{int}}}(I)\le R
\label{mbot:eq:uniform-dynamical-budgets-zero}
\end{equation}
for deterministic constants $\Lambda,R<\infty$.
Then, for all sufficiently large $L$ and every $\eta\in(0,1)$, with
probability at least $1-\eta$ over the training batches,
\begin{equation}
\left(\E_\star\sup_{t\in I}
\|\widehat X_t-X_t^\star\|^2\right)^{1/2}
\le C\Bigg[
\Psi_m(B)^{1/2}
+\left(\frac{\log^5(eL)}{L}\right)^{s/(2s+m+1)}
+\sqrt{\frac{\log(2/\eta)}{L}}
+\sqrt{\varepsilon_{\mathrm{opt}}}
\Bigg],
\label{mbot:eq:generic-fast-rate-zero}
\end{equation}
where $\Psi_m$ is defined in \eqref{mbot:eq:Psi-m-zero}.
The constant $C$ may depend on the fixed geometric parameter and network architecture,
including $d,m,s,c,L_S$, and on $\Lambda,R$, but not on $B,L,\eta$.
The Wasserstein error $W_2(\widehat\mu_c,\mu_c^\star)$ satisfies the
same upper bound.  Thus the neural approximation--estimation contribution
to the root strong-flow error is
$\widetilde O(L^{-s/(2s+m+1)})$.
\end{corollary}

\begin{proof}
The approximation $a_P(B)$, already incorporated into
Corollary~\ref{mbot:cor:intrinsic-field-exact-law}, is bounded as 
\[
a_P(B)
\le C\left(\frac{P}{\log^2(P+1)}\right)^{-s/(m+1)}
+C\Delta_B^{1/2}.
\]
For $P=P_L^{\mathrm{int}}$, on a training event of 
probability at least $1-\eta$, combining the field estimate
\eqref{mbot:eq:smooth-field-exact-law-bound}, Theorem~\ref{mbot:thm:flow-stability-zero}, and the assumption \eqref{mbot:eq:uniform-dynamical-budgets-zero} proves \eqref{mbot:eq:generic-fast-rate-zero}. The Wasserstein conclusion follows from \eqref{mbot:eq:pushforward-flow-zero}.
\end{proof}


In Corollary \ref{mbot:cor:intrinsic-fast-zero}, the intrinsic manifold dimension $m$ enters through two separate estimates:
$\Psi_m(B)$ controls the empirical-OT cost bias, while
$s/(2s+m+1)$ comes from structured neural approximation and estimation. Manifold support is not sufficient for the latter; the smooth manifold proximal geometry by OT is essential in the approximation arguments in Section \ref{mbot:sec:intrinsic-approximation}.


\begin{remark}[Explicit dependence on the learned-field budgets]
\label{mbot:rem:intrinsic-dynamical-budgets-section9}
Without the uniform bounds in
\eqref{mbot:eq:uniform-dynamical-budgets-zero}, but retaining the learned-field
Lipschitz condition and the ODE solvability and finite-OSL assumptions,
combine \eqref{mbot:eq:intrinsic-field-exact-law-unbalanced} with
\eqref{mbot:eq:strong-flow-bound-zero}.  For every sufficiently large $P$, this
gives, on the same training event of probability at least $1-\eta$,
\begin{equation}
\begin{aligned}
&\left(\E_\star\sup_{t\in I}
\|\widehat X_t-X_t^\star\|^2\right)^{1/2}
\le C e^{R^{\mathrm{OSL}}_{B,L,P}(I)}\sqrt c
\Bigg[
\Psi_m(B)^{1/2}
+\left(\frac{P}{\log^2(P+1)}\right)^{-s/(m+1)}
\\
&\qquad\qquad 
+K_I\sqrt{\frac{
P\log^2(P+1)\log(C_0A_PPL)+\log(2/\eta)}{L}}
+\sqrt{\varepsilon_{\mathrm{opt}}}
+\sqrt{2L_S\Delta_B}\,\Lambda_P^{\mathrm{tr}}(I)
\Bigg].
\end{aligned}
\label{mbot:eq:intrinsic-flow-explicit-budgets-section9}
\end{equation}
Thus growth of $\Lambda_P^{\mathrm{tr}}(I)$ affects the common-source
path-comparison correction, and growth of
$R^{\mathrm{OSL}}_{B,L,P}(I)$ affects the entire flow bound exponentially.
Moreover, imposing a smaller Lipschitz-constrained regression class
requires verifying the approximation bound within that class, as discussed
in Remark~\ref{mbot:rem:Lipschitz_condition}.

All rates in this section concern a fixed $c<1$ and comparison with
$\mu_c^\star$, not a quantitative terminal-time guarantee at $t=1$.
No uniformity of the constants as $c\uparrow1$ is asserted; the terminal
extension requires the additional care discussed in
Section~\ref{mbot:sec:terminal-extension} because of the divergence of the velocity $\uStar(x_t,t)$.
\end{remark}

\subsubsection{Conditional extension to the terminal time}
\label{mbot:sec:terminal-extension}


The spatial Lipschitz bound in Proposition~\ref{mbot:prop:exact-field-zero-regularity} grows like
$(1-t)^{-1}$ as $t\uparrow1$.
Accordingly, the preceding estimates are stated on $I=[0,c]$
for a fixed $c<1$, without asserting uniformity of their
constants as $c\uparrow1$.
Nevertheless, the one-sided Lipschitz stability argument extends
to the whole interval $[0,1]$ under the additional integrability
and terminal-limit assumptions stated below.

\begin{theorem}
\label{mbot:thm:terminal-direct-zero}
Assume
\begin{equation}
       \int_0^1(\rho_{\widehat u}(t))_+\dd t\le R,
       \label{mbot:eq:terminal-osl-zero}
\end{equation}
that the learned and exact trajectories admit limits at $t=1$, and that
\[
       \int_0^1\E_\star
       \|\widehat u(t,X_t^\star)-u^\star(t,X_t^\star)\|^2\dd t<\infty.
\]
Then
\begin{equation}
 \left(\E_\star\sup_{0\le t\le1}
 \|\widehat X_t-X_t^\star\|^2\right)^{1/2}
 \le e^R
 \left(\int_0^1\E_\star
 \|\widehat u(t,X_t^\star)-u^\star(t,X_t^\star)\|^2\dd t\right)^{1/2}.
 \label{mbot:eq:terminal-direct-flow-zero}
\end{equation}
Thus a quantitative rate of the full-interval vector field is inherited by the endpoint map with the same exponent.
\end{theorem}

\begin{proof}
Apply Theorem~\ref{mbot:thm:flow-stability-zero} on $I=[0,c]$ and let
$c\uparrow1$, using monotone convergence and the assumed terminal limits.
\end{proof}

\subsection{Summary of the quantitative conclusions}
\label{mbot:sec:quantitative-summary}

We summarize the intrinsic-dimensional conclusion obtained by combining
Section~\ref{mbot:sec:intrinsic-approximation},
Corollary~\ref{mbot:cor:intrinsic-field-exact-law}, and
Theorem~\ref{mbot:thm:flow-stability-zero}.  In addition to the standing
compact-support, forward-map, and source-support Bregman co-coercivity
conditions, this conclusion requires the smooth manifold proximal geometry
of Assumption~\ref{mbot:ass:smooth-proximal-geometry}.  
We also assume the learned-field Lipschitz condition
\eqref{mbot:eq:localized-learned-lipschitz} and the ODE solvability and OSL
conditions of Theorem~\ref{mbot:thm:flow-stability-zero}.  Fix $I=[0,c]$ with
$c<1$, and write
\[
\mathcal E_{B,L,P}^{\mathrm{flow}}(I)
:=\left(\mathbb E_\star\sup_{t\in I}
\|\widehat X_t-X_t^\star\|^2\right)^{1/2}.
\]
The expectation is over the independent population evaluation pair,
conditionally on the training randomness, as in
Section~\ref{mbot:sec:strong-flow-convergence}.

\subparagraph{Intrinsic approximation and finite-sample error.}
The essential approximation input is
\eqref{mbot:eq:batch-intrinsic-approx}:
\[
a_P(B)
\le C\left(\frac{P}{\log^2(P+1)}\right)^{-s/(m+1)}
+C\Delta_B^{1/2},
\qquad m=\dim M.
\]
This bound is under the expected-batch interpolation distribution; its
$\Delta_B^{1/2}$ term accounts for the perturbation away from population
OT paths.  It does not require a global Lipschitz bound on the comparison
network.

Combining \eqref{mbot:eq:intrinsic-field-exact-law-unbalanced} with
\eqref{mbot:eq:strong-flow-bound-zero} gives, on the same training event of
probability at least $1-\eta$,
\begin{equation}
\begin{aligned}
\mathcal E_{B,L,P}^{\mathrm{flow}}(I)
\le{}&C e^{R^{\mathrm{OSL}}_{B,L,P}(I)}\sqrt c
\Bigg[
\Delta_B^{1/2}
+\left(\frac{P}{\log^2(P+1)}\right)^{-s/(m+1)}
\\
&\qquad
+K_I\sqrt{\frac{
P\log^2(P+1)\log(C_0A_PPL)+\log(2/\eta)}{L}}
\\
&\qquad
+\sqrt{\varepsilon_{\mathrm{opt}}}
+\sqrt{2L_S\Delta_B}\,\Lambda_P^{\mathrm{tr}}(I)
\Bigg].
\end{aligned}
\label{mbot:eq:intrinsic-summary-explicit-budgets}
\end{equation}
In particular, if the deterministic bounds
\begin{equation}
\Lambda_P^{\mathrm{tr}}(I)\le\Lambda,
\qquad R^{\mathrm{OSL}}_{B,L,P}(I)\le R
\label{mbot:eq:intrinsic-summary-uniform-budgets}
\end{equation}
hold uniformly in the asymptotic regime for every training realization,
then the root strong-flow error
has the form
\begin{equation}
\begin{aligned}
\mathcal E_{B,L,P}^{\mathrm{flow}}(I)\ \lesssim{}&
\underbrace{\Delta_B^{1/2}}_{\text{expected-batch OT error}}
+\underbrace{\left(P/\log^2(P+1)\right)^{-s/(m+1)}}_
{\text{intrinsic ReLU approximation}}
\\
&+\underbrace{K_I\sqrt{\frac{
P\log^2(P+1)\log(C_0A_PPL)+\log(2/\eta)}{L}}}_
{\text{estimation from }L\text{ independent training batches}}
\\
&+\underbrace{\sqrt{\varepsilon_{\mathrm{opt}}}}_{\text{optimization}}.
\end{aligned}
\label{mbot:eq:final-summary-zero}
\end{equation}
The hidden constants may depend on the fixed geometry, $d,m,s,c,L_S$,
the sieve constants, $\Lambda$, and $R$, but not on $B,L,P,\eta$.
Thus the algebraic approximation exponent is governed by the intrinsic
space--time dimension $m+1$, although constants may depend on the ambient
dimension $d$.  The effective sample count in this argument is $L$,
not $LB$, because independence is used between training batches.

\subparagraph{Optimized neural and batch-size rates.}
For the intrinsic parameter choice
\eqref{mbot:eq:optimal-intrinsic-parameter}, namely
\[
P_L^{\mathrm{int}}
\asymp\log^2(eL)
\left(\frac{L}{\log^5(eL)}\right)^{(m+1)/(2s+m+1)},
\]
conditions \eqref{mbot:eq:intrinsic-summary-uniform-budgets} yield
\begin{equation}
\mathcal E_{B,L,P_L^{\mathrm{int}}}^{\mathrm{flow}}(I)
\le C\Bigg[\Delta_B^{1/2}
+\left(\frac{\log^5(eL)}{L}\right)^{s/(2s+m+1)}
+\sqrt{\frac{\log(2/\eta)}{L}}
+\sqrt{\varepsilon_{\mathrm{opt}}}\Bigg]
\label{mbot:eq:final-intrinsic-optimized-rate}
\end{equation}
with probability at least $1-\eta$.  Hence the neural contribution is
$\widetilde O(L^{-s/(2s+m+1)})$.  Under the empirical-cost hypotheses in
Section~\ref{mbot:sec:batch-size-rate}, the additional bound
\eqref{mbot:eq:Delta-Psi-zero} gives
\begin{equation}
\mathcal E_{B,L,P_L^{\mathrm{int}}}^{\mathrm{flow}}(I)
\le C\Bigg[
\Psi_m(B)^{1/2}
+\left(\frac{\log^5(eL)}{L}\right)^{s/(2s+m+1)}
+\sqrt{\frac{\log(2/\eta)}{L}}
+\sqrt{\varepsilon_{\mathrm{opt}}}\Bigg],
\label{mbot:eq:final-intrinsic-manifold-rate}
\end{equation}
where $\Psi_m$ is defined in \eqref{mbot:eq:Psi-m-zero}.  Thus intrinsic
dimension enters both the empirical-OT contribution and the structured
neural approximation contribution, through distinct arguments.
The Wasserstein error $W_2(\widehat\mu_c,\mu_c^\star)$ obeys the same
upper bounds by \eqref{mbot:eq:pushforward-flow-zero}.

\subparagraph{Scope of the conclusion.}
The result above uses Assumption~\ref{mbot:ass:smooth-proximal-geometry};
manifold support alone does not justify replacing $d$ by $m$ in a generic
smooth-function approximation theorem.  Moreover, neither the
approximation theorem nor the polynomial weight envelope supplies the
uniform learned-field budgets in
\eqref{mbot:eq:intrinsic-summary-uniform-budgets}.  If those budgets grow,
their dependence must remain explicit as in
\eqref{mbot:eq:intrinsic-summary-explicit-budgets}.

A quantitative endpoint rate at $t=1$ requires additional
full-interval error estimates and the conditions of
Theorem~\ref{mbot:thm:terminal-direct-zero}, which are not established by the preceding
fixed-$c$ analysis.


\par
\endgroup

\section{Additional experiments and detailed experimental settings}

\subsection{Toy examples}
\label{sec:experiments_add_toy}

We include an additional toy example of ``circle", which consists of the uniform distribution on $\mathbb S^1$. The experimental setting is the same as  ``Two Moons" in Sec.~\ref{sec:experiment}.  We also replicate the example of ``Two Moons" shown there.  Figure \ref{fig:circle_moon} shows the results.  In a similar way to the result for ``Two Moons" in Sec.~\ref{sec:experiment}, the trained vector field around the target manifold are normal to the manifold. The two eigenvalues of $D_x v_t(x_t)$ for different $t$ are clearly separated with $\lambda_1 = 0$ and $\lambda_2 = -\gamma = -1.0$ for $t\to 1$.  The results on the circle date agrees better with Theorem~\ref{thm:terminal-lyap} than ``Two Moons", since the target distribution is a manifold rigorously.  

\begin{figure}[ttt]
    \centering
    (a) Circle\\ 
    \includegraphics[width=0.70\linewidth]
    {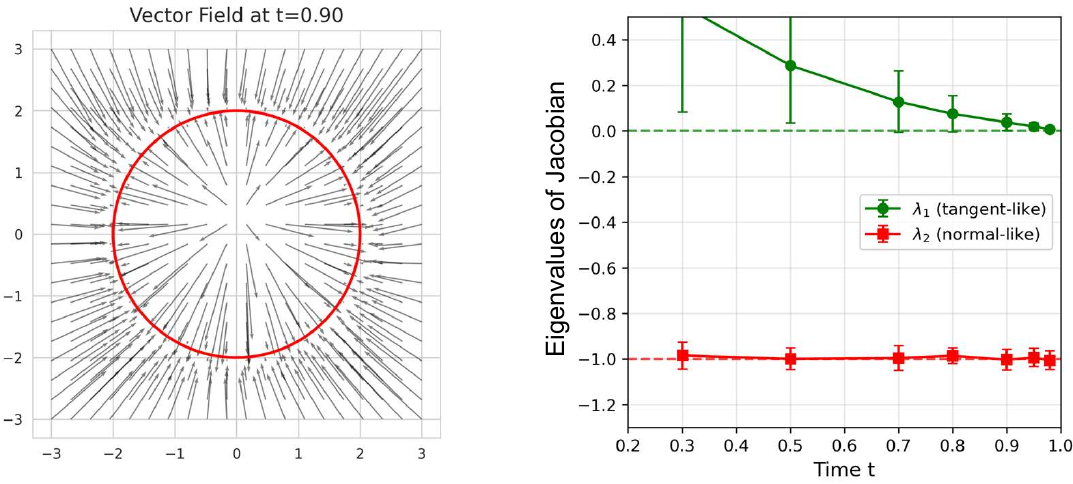}\\
    (b) Two Moons\\
    \includegraphics[width=0.70\linewidth]{fig/Lyapunov_twomoon_new2_cropped.pdf}
    \caption{Left: OT-CFM vector field at time $t=0.9$ obtained by a Neural Network, which appears to be an attracting force on the manifold.  Right: Mean and STD of the eigenvalues of the Jacobian $(1-t)Dv_t$ over different trajectories. The second eigenvalue corresponds to a terminal Lyapounov exponent close to -1, as predicted by the analysis in Sec.~\ref{sec:lyap-prox-epi}. 
    }\label{fig:circle_moon}
\end{figure}

\subsection{Experimental results of Minibatch OT}
\label{sec:exp_minibatch}

We trained OT-CFM on the circle dataset (Sec.~\ref{sec:experiments_add_toy}) with minibatch sizes $B \in \{16, 64, 128, 256, 512, 1024\}$ (3 seeds each, 30k iterations) and computed the rescaled eigenvalues $(1-t)\lambda_i$ of the Jacobian 
$D_x v_t$ along the time trajectory $t \in [0.3, 1.0]$.

Fig.~\ref{fig:minibatchOT} clearly shows that larger batch sizes produce eigenvalues that follow more closely the theoretical values of Theorem \ref{thm:terminal-lyap} for the population OT. Interestingly, even with a small batch size, the eigenvalues do not differ significantly from the theoretical ones. In the left panel, the engenvalues converge to $\hat{\lambda}_1 = 0$ as $t \to 1$ for all batch sizes, with larger batches converging more smoothly. In the right panel, larger batches (e.g., $B=512, 1024$) closely track the theoretical value $-\gamma = -1.0$ across $t \in [0.3, 0.9]$, while smaller batches (e.g., $B=16$) deviate more, underestimating the normal contraction rate. This provides direct empirical support for the convergence of the minibatch OT-CFM and confirms that our theoretical analysis gives a reasonable approximation for minibatch OT.

\begin{figure}
    \centering
    \includegraphics[width=0.8\linewidth]{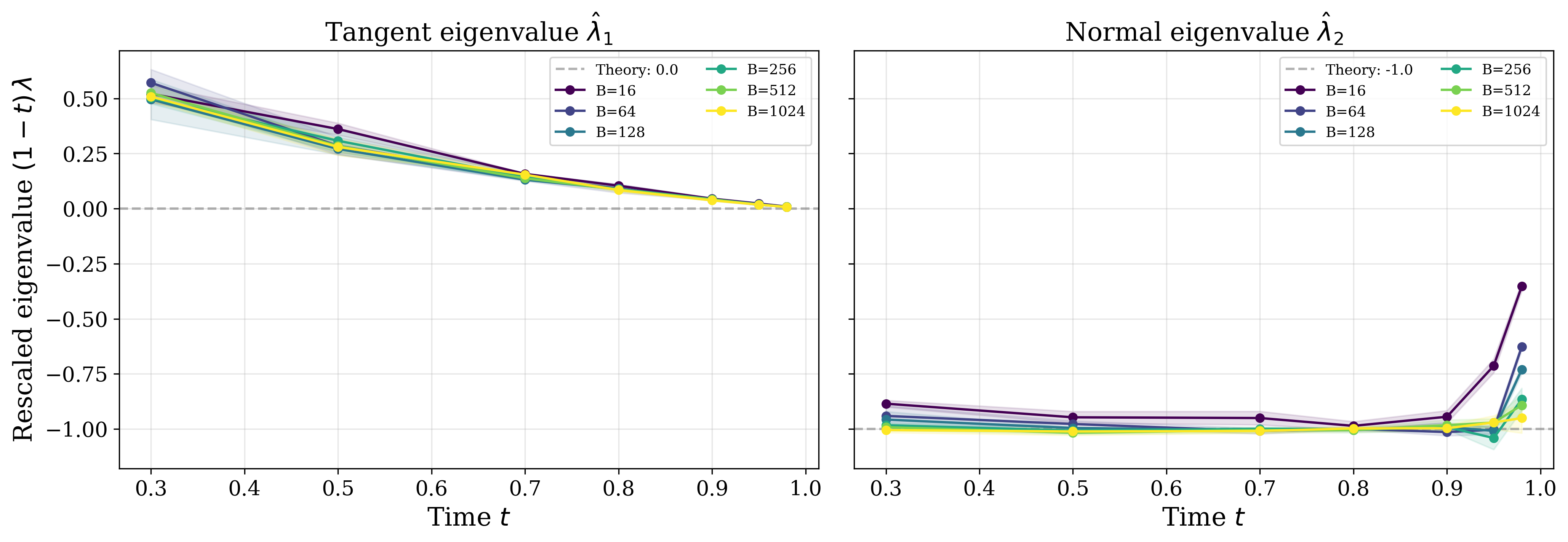}
    \caption{Eigenvalues of Jacobian of the vector field trained with minibatch OT-CFM.}
    \label{fig:minibatchOT}
\end{figure}

\subsection{Additional example: stability of data manifold}
\label{sec:MNIST47}

This subsection includes another example with the digit ``7" in addition to the MNIST perturbation experiment in Sec.\ref{sec:experiment}.  The results for digit ``7" (Figure \ref{fig:mnist_47}) demonstrate properties similar to those of case ``4".  The dimension of the larger eigenvalues is around 100 again.

\begin{figure}
    \centering
    \includegraphics[width=0.8\linewidth]{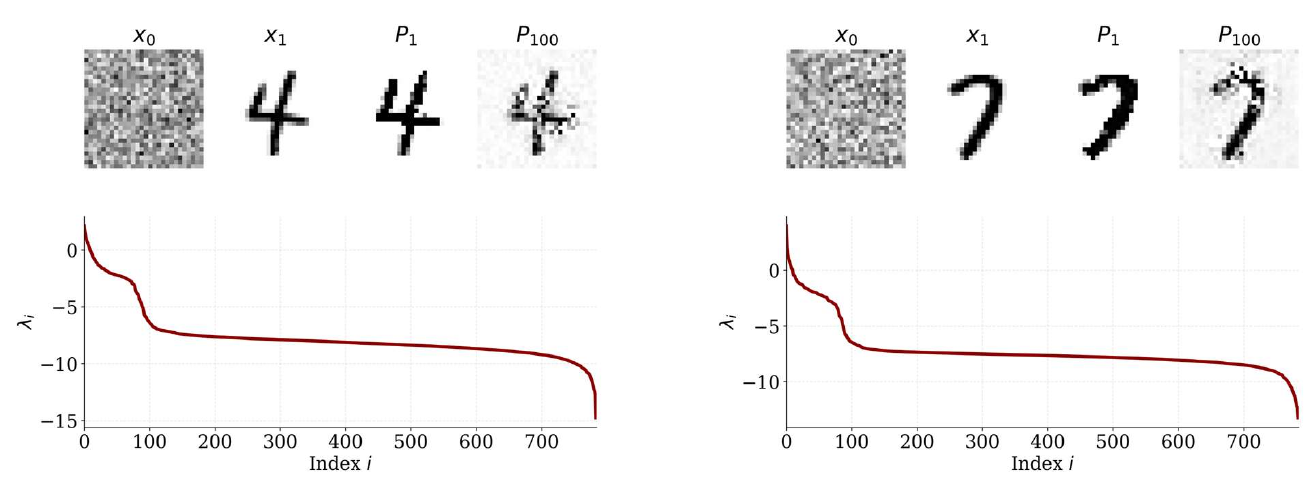}
    \caption{Perturbed MNIST images with eigenvectors}
    \label{fig:mnist_47}
\end{figure}

\subsection{Robustness of the data manifold to learning errors}   
\label{sec:robustness_MNIST}
The implication of our results in Sec.\ref{sec:lyap-prox-epi} is that any small perturbation of the initial or intermediate point will be reduced to movement tangent to the data manifold at the terminal time. We demonstrate this theoretical argument with the MNIST dataset in Figure \ref{fig:perturb_MNIST}, where eight random perturbations $x_0^{(k)} = x_0 + \epsilon \delta_k$ ($\epsilon=2.0$, $\delta_k$ unit-norm random direction in $\mathbb{R}^{784}$) are applied at $t=0$ and evolved to $t=0.99$ via the learned ODE. All perturbed outputs remain valid MNIST digits, and the displacement maps $\Delta x_1 = x_1' - x_1$ exhibit structured, edge-localized patterns corresponding to tangential deformations, confirming that the normal components of the initial perturbation are contracted by the flow. Quantitatively, PCA on the displacement matrix (Figure \ref{fig:rank_MNIST}) shows that the effective rank of $\Delta x_1$ is only 24 (at 95\% cumulative variance) out of 784 ambient dimensions when disturbed at $t=0$. On the other hand, when the perturbation is applied at $t=0.75,$ the effective rank is 181 because less contraction time remains, directly reflecting the exponential normal contraction predicted by Theorem \ref{thm:terminal-lyap}.  

\begin{figure}
    \centering
    \includegraphics[width=\linewidth]{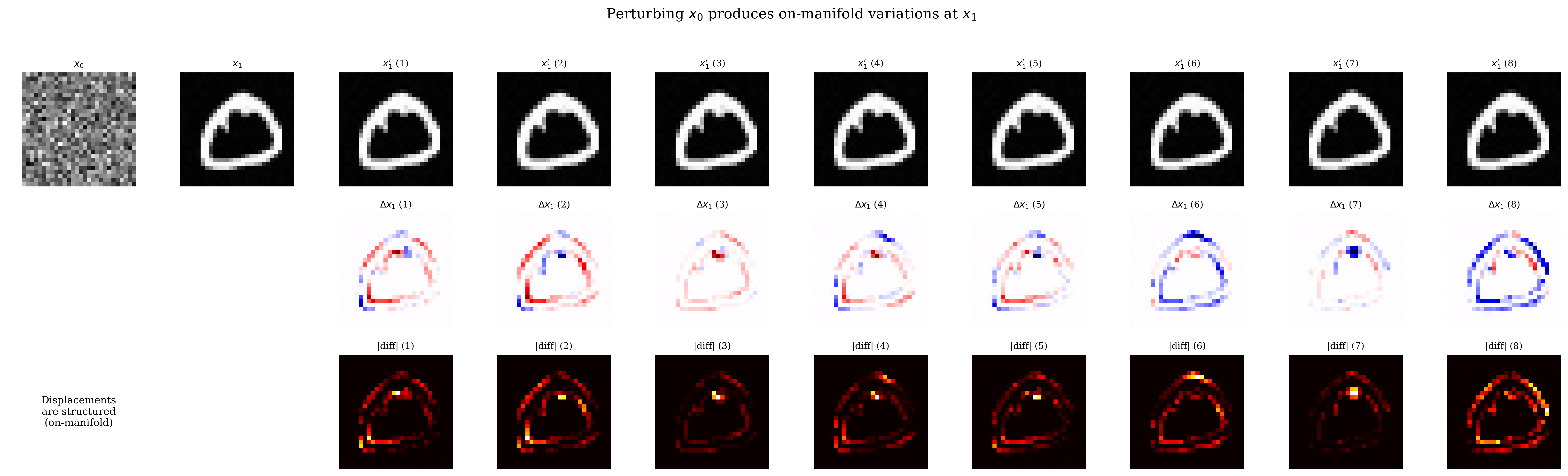}
    \caption{Generated images with perturbed initial points.  The top rows show the generated images, and the bottom rows shows the displacement $\Delta x_1 = x_1' - x_1$.}
    \label{fig:perturb_MNIST}
\end{figure}

\begin{figure}
    \centering
    \includegraphics[width=\linewidth]{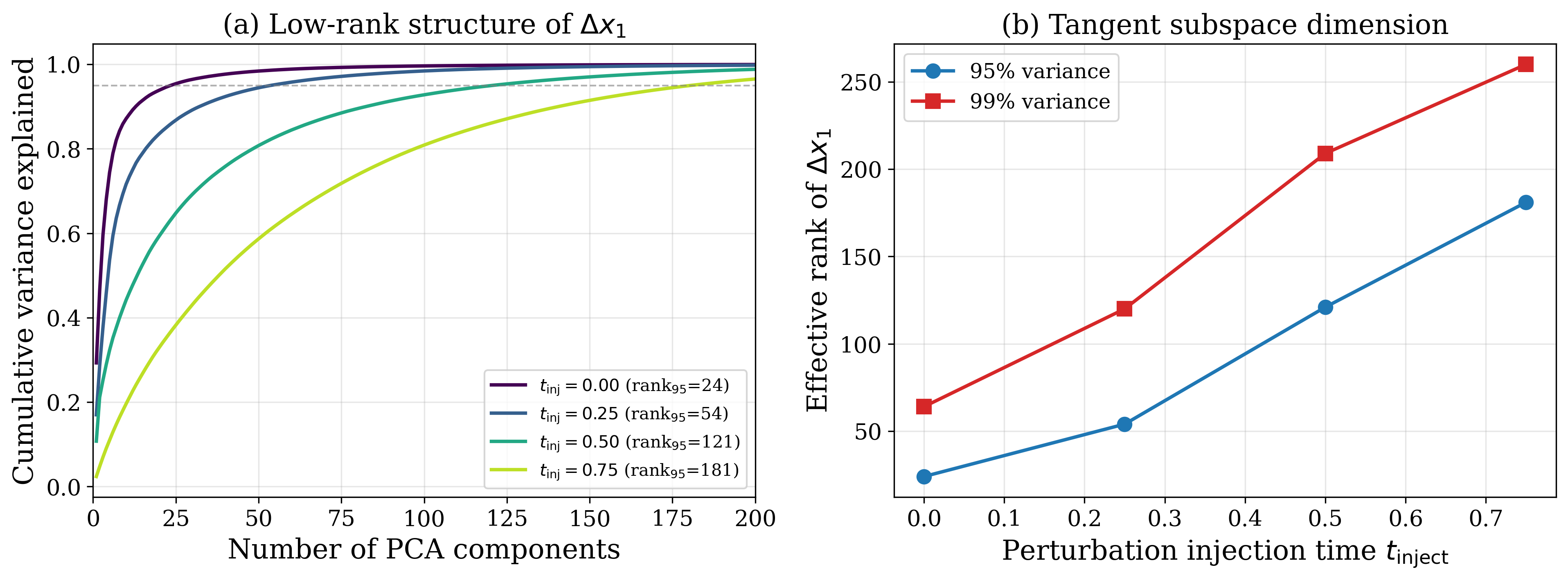}
    \caption{Results of the PCA for the displacement maps.}
    \label{fig:rank_MNIST}
\end{figure}

\subsection{CIFAR-10 perturbation experiments}
\label{sec:additional-cifar-perturbations}

This subsection provides further details and additional examples for the CIFAR-10 experiment summarized in Sec.~\ref{sec:experiment}.  Each CIFAR-10 image is represented
as a vector in $\mathbb{R}^{3\times 32\times 32}\simeq\mathbb{R}^{3072}$, and its pixel values are scaled to $[-1,1]$.  The source distribution is
$P_0=\mathcal{N}(0,I_{3072})$, and we use the affine interpolation
$x_t=(1-t)x_0+t x_1$.
At every training iteration, batches of size $B=1024$ are drawn from the
source and target distributions.  Their coupling is approximated by the
Sinkhorn algorithm with entropic regularization parameter $0.01$, $200$
Sinkhorn iterations, and sampling from the resulting transport plan.

The vector field is represented by a time-conditioned U-Net with approximately $57.8$ million parameters.  The base channel width is $128$, and residual blocks and self-attention layers are used at the $16\times16$, $8\times8$, and $4\times4$ resolutions.  We use dropout $0.1$ and a sinusoidal time embedding applied to $1000t$.  The model is trained for $2000$ epochs using
AdamW, with the learning rate decreased from $2\times10^{-4}$ to
$10^{-6}$ by cosine annealing.  Training uses BF16 arithmetic and an
exponential moving average with decay $0.999$.  The generated samples below use the EMA parameters and a fixed-step Heun solver up to
$t_{\mathrm{end}}=0.999$.

\paragraph*{Endpoint sensitivity with respect to the initial point.}
Let
\[
    \widehat\Phi(x_0)
    :=
    \widehat\Phi_{t_{\mathrm{end}},0}(x_0)
\]
denote the endpoint map induced by the learned vector field and the numerical
ODE solver.  To directly examine the influence of perturbations to $x_0$, we
consider its Jacobian
\[
    J_{\Phi}(x_0)
    :=
    D_{x_0}\widehat\Phi(x_0)
    \in\mathbb{R}^{3072\times3072}.
\]
Rather than explicitly forming this matrix, we use Jacobian--vector and
vector--Jacobian products and approximate the leading singular values and right singular vectors of the endpoint-map Jacobian using randomized SVD:
\[
    J_{\Phi}(x_0)
    \simeq
    U\operatorname{diag}(\sigma_1,\ldots,\sigma_r)V^\top,
    \qquad
    \sigma_1\geq\cdots\geq\sigma_r.
\]
The columns $v_k$ of $V$ are directions in the initial Gaussian space, while
$\sigma_k u_k$ describes the corresponding first-order displacement at the
endpoint:
\[
    \widehat\Phi(x_0+\varepsilon v_k)
    -
    \widehat\Phi(x_0)
    =
    \varepsilon\sigma_k u_k+o(\varepsilon).
\]
Thus, the directions shown below are right singular directions of the
endpoint flow map.

For both random and singular-direction perturbations, we use a
root-mean square perturbation of $0.01$ per coordinate, that is, the Euclidean
magnitude of the perturbation is $\varepsilon  =     0.01\sqrt{3072} \simeq 0.55$ 

\paragraph*{Independent random perturbations.}
We first draw independent isotropic directions
\[
    r_j\sim\mathcal{N}(0,I_{3072}),
    \qquad
    \bar r_j=\frac{r_j}{\|r_j\|},
\]
and generate
\[
    x_1^{(j)}
    =
    \widehat\Phi(x_0+\varepsilon\bar r_j).
\]
The left panel of Figure~\ref{fig:cifar-perturb} shows generated images $x_1^{(j)}$ for two different $x_0$, for which $\widehat\Phi(x_0)$ are airplane and automobile.  The generated images remain almost unchanged
under the independent random perturbations.

This behavior has a natural high-dimensional interpretation.  If the local
tangent space has dimension $m$ and $r$ is uniform on the unit sphere in
$\mathbb{R}^d$, then
\[
    \mathbb{E}\|\Pi_T r\|^2=\frac{m}{d},
    \qquad
    \mathbb{E}\|\Pi_N r\|^2=\frac{d-m}{d}.
\]
For CIFAR-10, the ambient dimensionality is $d=3072$, but the dimensionality of the image space is considered to be much smaller; various works have given different estimates, which are usually less than 100 (see, e.g., \cite{pope2021intrinsic}).
Consequently, most of the energy of a generic random direction tends to lie in the normal component.  The almost  invariance in
The result in Fig.~\ref{fig:cifar-perturb} is therefore consistent with the terminal contraction of normal perturbations established in Theorem~\ref{thm:terminal-lyap}.

\begin{figure}[t]
    \centering
    \includegraphics[width=0.48\linewidth]{fig/airplane_random2.png}
    \includegraphics[width=0.48\linewidth]{fig/airplane_eigen2.png}

    \vspace{1em}

    \includegraphics[width=0.48\linewidth]{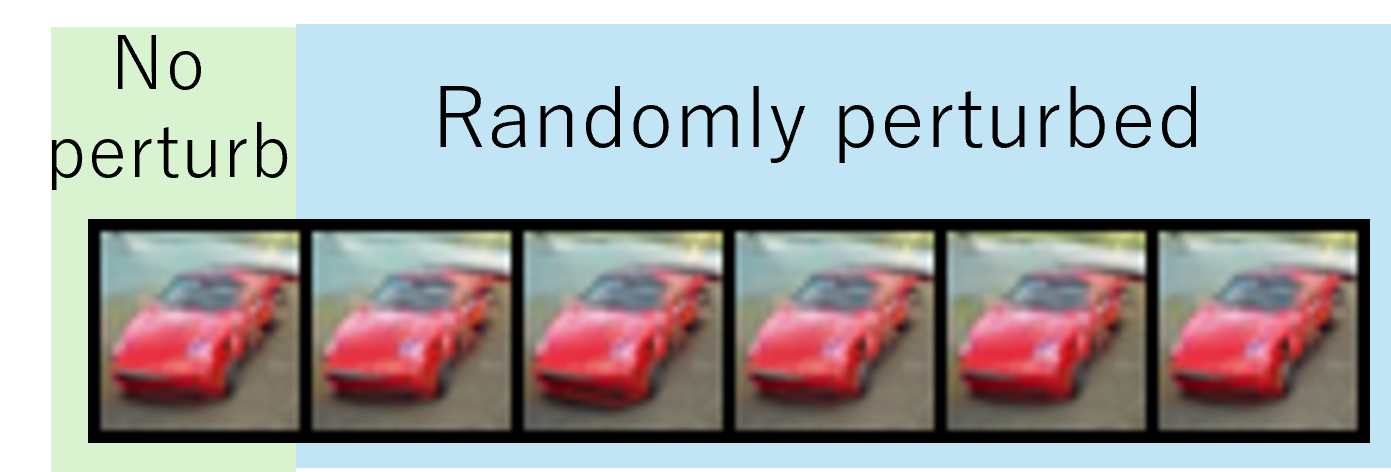} 
    \includegraphics[width=0.48\linewidth]{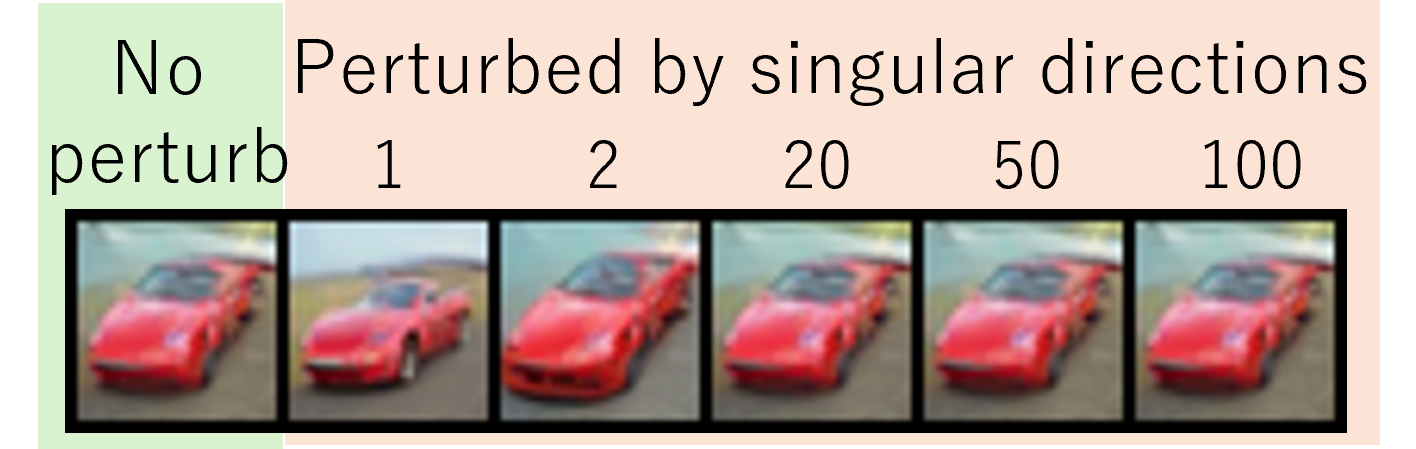} 
    \caption{
        CIFAR-10 generation under perturbations of random directions and sigular vectors at the initial point. In each panel, the leftmost image is generated from the unperturbed initial point $x_0$, and the remaining images are generated from perturbed initial ones. 
        Left: random perturbations. The generated images
        remain nearly identical to the reference image, consistently with contraction of the predominantly normal components. Right: Perturbations by the $k$-th right singular vector $v_k$ of $D_{x_0}\widehat\Phi(x_0)$ in descending singular-value
        order.  The number above each image indicates $k$.  Leading directions produce coherent changes of pose, shape, viewpoint, or background while the samples remain visually close to the CIFAR-10 image manifold.  Higher-rank directions have visibly weaker effects.
    }
    \label{fig:cifar-perturb}
\end{figure}

\paragraph*{Perturbations along dominant flow-map directions.}
We next perturb $x_0$ along selected right singular vectors:
\[
    x_1^{(k)}
    =
    \widehat\Phi(x_0+\varepsilon v_k).
\]
In the right panels of Fig.~\ref{fig:cifar-perturb}, the generated images with $v_k$ are shown, where $v_k$ corresponds to the leading $k$-th singular value $\sigma_k$.  In contrast
to isotropic random perturbations, the leading singular directions ($k=1,2,3$) produce visible but structured variations.  For the airplane, the first few
directions modify the shape of the aircraft while
preserving an airplane-like image.  For the automobile, the leading
directions alter the front or side geometry while retaining
a coherent automobile image.  The changes become much smaller for directions with larger $k$.

These observations are consistent with the tangent--normal decomposition in Section \ref{sec:lyap-prox-epi}. 
Directions that survive the terminal dynamics should appear among the
dominant singular directions of $D_{x_0}\widehat\Phi$, whereas strongly
contracted normal directions correspond to small singular values.  Thus, the
leading $v_k$ provides empirical tangent-like directions: they induce
variations within the image manifold instead of producing
unstructured pixel noise.



\end{document}